\documentclass{article}

\title{
	Finding Optimal Diverse Feature Sets\texorpdfstring{\\}{ }with Alternative Feature Selection
}
\author{
	Jakob Bach~\orcidlink{0000-0003-0301-2798}\\
	\small Independent researcher\footnote{Most of the research for this article was carried out while the author was affiliated with the Karlsruhe Institute of Technology (KIT), Karlsruhe, Germany.}\\
	\small \href{mailto:jakob.bach.ka@gmail.com}{jakob.bach.ka@gmail.com}
}
\date{} 

\usepackage[style=numeric, backend=bibtex]{biblatex}
\usepackage[ruled,linesnumbered,vlined]{algorithm2e} 
\usepackage{amsmath} 
\usepackage{amssymb} 
\usepackage{amsthm} 
\usepackage{booktabs} 
\usepackage{enumitem} 
\usepackage{graphicx} 
\usepackage{multirow} 
\usepackage{orcidlink} 
\usepackage{subcaption} 
\usepackage{hyperref} 

\newtheorem{proposition}{Proposition}
\theoremstyle{definition}
\newtheorem{definition}{Definition}
\newtheorem{example}{Example}

\newcommand{\stirling}[2]{\genfrac\{\}{0pt}{}{#1}{#2}} 

\begin{document}

\maketitle

\begin{abstract}
Feature selection is popular for obtaining small, interpretable, yet highly accurate prediction models.
Conventional feature-selection methods typically yield one feature set only, which might not suffice in some scenarios.
For example, users might be interested in finding alternative feature sets with similar prediction quality, offering different explanations of the data.
In this article, we introduce alternative feature selection and formalize it as an optimization problem.
In particular, we define alternatives via constraints and enable users to control the number and dissimilarity of alternatives.
We consider sequential as well as simultaneous search for alternatives.
Next, we discuss how to integrate conventional feature-selection methods as objectives.
In particular, we describe solver-based search methods to tackle the optimization problem.
Further, we analyze the complexity of this optimization problem and prove $\mathcal{NP}$-hardness.
Additionally, we show that a constant-factor approximation exists under certain conditions and propose corresponding heuristic search methods.
Finally, we evaluate alternative feature selection in comprehensive experiments with 30 binary-classification datasets.
We observe that alternative feature sets may indeed have high prediction quality, and we analyze factors influencing this outcome.
\end{abstract}

\section{Introduction}
\label{sec:afs:introduction}

\paragraph{Motivation}

Feature-selection methods are ubiquitous for a variety of reasons.
By reducing dataset dimensionality, they lower the computational cost and memory requirements of prediction models.
Next, prediction models may generalize better without irrelevant and spurious predictors.
While some model types can implicitly select relevant features, others cannot.
Finally, prediction models may become simpler~\cite{li2017feature}, improving interpretability.

Most conventional feature-selection methods only return one feature set~\cite{borboudakis2021extending}.
These methods optimize a criterion of feature-set quality, e.g., prediction performance.
However, besides the optimal feature set, there might be other, differently composed feature sets with similar quality.
Such alternative feature sets are interesting for users, e.g., to obtain several diverse explanations.
Alternative explanations can provide additional insights into predictions, enable users to develop and test different hypotheses, appeal to different kinds of users, and foster trust in the predictions~\cite{kim2021multi, wang2019designing}.

For example, in a dataset describing physical experiments, feature selection may help to discover relationships between physical quantities.
In particular, highly predictive feature sets indicate which input quantities are strongly related to the output quantity.
Domain experts may use these feature sets to formulate hypotheses on physical laws.
However, if multiple alternative sets of similar quality exist, further analyses and experiments may be necessary to reveal the true underlying physical mechanism.
Only knowing one predictive feature set and using it as the only explanation is misleading in such a situation.

\paragraph{Problem statement}

This article addresses the problem of alternative feature selection, which we informally define as follows:
Find multiple, sufficiently different feature sets that optimize feature-set quality.
We provide formal definitions in Section~\ref{sec:afs:approach:constraints}.
This problem entails an interesting trade-off:
Depending on how many alternatives are desired and how different the alternatives should be, one may have to compromise on quality.
In particular, a higher number of alternatives or a stronger dissimilarity requirement may necessitate selecting more low-quality features in the alternatives.

Two points are essential for alternative feature selection, which we both address in this article.
First, one needs to formalize and quantify what an alternative feature set is.
In particular, users should be able to control the number and dissimilarity of alternatives and hence the quality trade-off.
Second, one needs an approach to find alternative feature sets efficiently.
Ideally, the approach should be general, i.e., cover a broad range of conventional feature-selection methods, given the variety of the latter~\cite{chandrashekar2014survey, li2017feature}.

\paragraph{Related work}

While finding alternative solutions has already been addressed extensively in the field of clustering~\cite{bailey2014alternative}, there is a lack of such approaches for feature selection.
Only a few feature-selection methods target at obtaining multiple, diverse feature sets~\cite{borboudakis2021extending}.
In particular, techniques for ensemble feature selection~\cite{saeys2008robust, seijo2017ensemble} and statistically equivalent feature subsets~\cite{lagani2017feature} produce multiple feature sets but not optimal alternatives.
These approaches do not guarantee the diversity of the feature sets, nor do they let users control diversity.
In fields related to feature selection, the goal of obtaining multiple, diverse solutions has been studied as well, e.g., for subspace clustering~\cite{hu2018subspace, mueller2009relevant}, subgroup discovery~\cite{leeuwen2012diverse}, subspace search~\cite{trittenbach2019dimension}, or explainable-AI techniques~\cite{artelt2022even, kim2016examples, mothilal2020explaining, russell2019efficient} like counterfactuals.
These approaches are not directly applicable or easily adaptable to feature selection, and most of them provide limited or no user control over alternatives, as we will elaborate in Section~\ref{sec:afs:related-work}.

\paragraph{Contributions}

Our contribution is five-fold.

First, we formalize alternative feature selection as an optimization problem.
In particular, we define alternatives via constraints on feature sets.
This approach is orthogonal to the feature-selection method so that users can choose the latter according to their needs.
This approach also allows integrating other constraints on feature sets, e.g., to capture domain knowledge \cite{bach2022empirical, groves2015toward}.
Finally, this approach lets users control the search for alternatives with two parameters, i.e., the number of alternatives and a dissimilarity threshold.
For multiple alternatives, we consider sequential as well as simultaneous search.

Second, we discuss how to solve this optimization problem.
To that end, we describe how to integrate different categories of conventional feature-selection methods in the objective function of the optimization problem.
In particular, we outline solver-based search methods for white-box and black-box optimization.

Third, we analyze the time complexity of the optimization problem.
We show $\mathcal{NP}$-hardness, even for a simple notion of feature-set quality, i.e., univariate feature qualities, as used in filter feature selection.

Fourth, we propose heuristic search methods for univariate feature qualities.
We show that, under certain conditions, the optimization problem resides in the complexity class $\mathcal{APX}$, i.e., a constant-factor approximation exists.

Fifth, we evaluate alternative feature selection with comprehensive experiments.
In particular, we use 30 binary-classification datasets from the Penn Machine Learning Benchmarks (PMLB)~\cite{olson2017pmlb, romano2021pmlb} and five feature-selection methods.
We focus our evaluation on the feature-set quality of the alternatives relative to our search methods for alternatives and user parameters.
Additionally, we evaluate runtime.
We publish all code and data online (cf.~Section~\ref{sec:afs:experimental-design:implementation}).

\paragraph{Experimental results}

We observe that several factors influence the quality of alternatives, i.e., the dataset, feature-selection method, metric for feature-set quality, search method, and user parameters for searching alternatives.
As expected, feature-set quality tends to decrease with an increasing number of alternatives and an increasing dissimilarity threshold for alternatives.
Thus, these parameters allow users to control the trade-off between dissimilarity and quality of alternatives.
Also, no valid alternative may exist if the parameter values are too strict.
Runtime-wise, a solver-based sequential search for multiple alternatives was significantly faster than a simultaneous one while yielding a similar quality.
Additionally, our heuristic search methods for univariate feature qualities achieved a high quality within negligible runtime.
Finally, we observe that the prediction performance of feature sets may only weakly correlate with the quality assigned by feature-selection methods.
In particular, seemingly bad alternatives regarding the latter might still be good regarding the former.

\paragraph{Outline}

Section~\ref{sec:afs:fundamentals} introduces notation and fundamentals.
Section~\ref{sec:afs:approach} describes and analyzes alternative feature selection.
Section~\ref{sec:afs:related-work} reviews related work.
Section~\ref{sec:afs:experimental-design} outlines our experimental design, while Section~\ref{sec:afs:evaluation} presents the experimental results.
Section~\ref{sec:afs:conclusion} concludes.
Appendix~\ref{sec:afs:appendix} contains supplementary materials.

\paragraph{Related versions}

The dissertation~\cite{bach2025leveraging} contains a shortened version of this article, which retains all central experimental and theoretical results.
In particular, the dissertation evaluates the same run of the experimental pipeline.
The search method \emph{Greedy Depth Search} (cf.~Appendix~\ref{sec:afs:appendix:greedy-depth}) appears exclusively in this arXiv version but is not evaluated anyway.

There is also a journal version of this article~\cite{bach2024alternative}.
However, it is outdated in some places since we derived it from the first arXiv version.
For example, the journal version lacks the heuristic search methods for alternatives (cf.~Section~\ref{sec:afs:approach:univariate-heuristics}), uses an older version of the \emph{Greedy Wrapper} approach (cf.~Algorithm~\ref{al:afs:greedy-wrapper}), and partly differs in notation and formalization, e.g., definitions.

\section{Fundamentals}
\label{sec:afs:fundamentals}

In this section, we introduce basic notation (cf.~Section~\ref{sec:afs:fundamentals:notation}) and review different methods to measure the quality of feature sets (cf.~Section~\ref{sec:afs:fundamentals:quality}).

\subsection{Notation}
\label{sec:afs:fundamentals:notation}

$X \in \mathbb{R}^{m \times n}$ stands for a dataset in the form of a matrix.
Each row is a data object, and each column is a feature.
$\tilde{F} = \{f_1, \dots, f_n\}$ is the corresponding set of feature names.
We assume that categorical features have already been made numeric, e.g., via one-hot encoding.
$X_{\cdot{}j} \in \mathbb{R}^m$ denotes the vector representation of the $j$-th feature.
$y \in Y^m$ represents the prediction target with domain $Y$, e.g., $Y=\{0,1\}$ for binary classification or $Y=\mathbb{R}$ for regression.

In feature selection, one makes a binary decision $s_j \in \{0,1\}$ for each feature, i.e., either selects it or not.
The vector $s \in \{0,1\}^n$ combines all these selection decisions and yields the selected feature set $F_s = \{f_j \mid s_j=1\} \subseteq \tilde{F}$.
To simplify notation, we drop the subscript~$s$ in definitions where we do not explicitly refer to the value of~$s$ but only the set~$F$.
The function $Q(s,X,y)$ returns the quality of such a feature set.
Without loss of generality, we assume that this function should be maximized.

\subsection{Measuring Feature (Set) Quality}
\label{sec:afs:fundamentals:quality}

There are different ways to evaluate feature-set quality $Q(s,X,y)$.
We only give a short overview here; see~\cite{chandrashekar2014survey, li2017feature, njoku2023wrapper} for comprehensive studies and surveys of feature selection.
Also, note that we assume a supervised feature-selection scenario, i.e., feature-set quality depending on a prediction target~$y$.
In principle, our definitions of alternatives also apply to an unsupervised scenario.
Since the prediction target only appears in the function~$Q(s,X,y)$, one could replace~$Q(s,X,y)$ with $Q(s,X)$, i.e., an unsupervised notion of quality.

A conventional categorization of feature-selection methods distinguishes between filter, wrapper, and embedded methods~\cite{guyon2003introduction}.

\paragraph{Filter methods}

Filter methods evaluate feature sets without training a prediction model.
Univariate filters assess each feature independently, e.g., using the absolute Pearson correlation or the mutual information between a feature and the prediction target.
Multivariate filters additionally consider interactions between features.
Such methods often combine a measure of feature relevance with a measure of feature redundancy.
Examples include CFS~\cite{hall1999correlation, hall2000correlation}, FCBF~\cite{yu2003feature}, and mRMR~\cite{peng2005feature}.
Some filter methods also consider feature intercooperation, i.e., the joint relevance of two or more features~\cite{sosa2024feature}.

\paragraph{Wrapper methods}

Wrapper methods~\cite{kohavi1997wrappers} evaluate feature sets by training prediction models with them and measuring prediction quality.
They employ a generic search strategy to iterate over candidate feature sets, e.g., genetic algorithms.
Feature-set quality is a black-box function in this search.

\paragraph{Embedded methods}

Embedded methods train prediction models with built-in feature selection, e.g., decision trees~\cite{breiman1984classification} or random forests~\cite{breiman2001random}.
Thus, the criterion for feature-set quality is model-specific.
For example, tree-based models often use information gain or the Gini index to select features during training.

\paragraph{Post-hoc feature-importance methods}

Apart from conventional feature selection, there are various methods that assess feature importance after training a model.
These methods range from local explanation methods like LIME~\cite{ribeiro2016should} or SHAP~\cite{lundberg2017unified} to global importance methods like permutation importance~\cite{breiman2001random} or SAGE~\cite{covert2020understanding}.
In particular, assessing feature importance plays a crucial role in the field of machine-learning interpretability~\cite{carvalho2019machine, molnar2020interpretable}.

\section{Alternative Feature Selection}
\label{sec:afs:approach}

In this section, we present the problem of and approaches for alternative feature selection.
First, we define the overall structure of the optimization problem, i.e., objective and constraints (cf.~Section~\ref{sec:afs:approach:problem}).
Second, we formalize the notion of alternatives via constraints (cf.~Section~\ref{sec:afs:approach:constraints}).
Third, we discuss objective functions corresponding to different feature-set quality measures from Section~\ref{sec:afs:fundamentals:quality} and describe how to solve the resulting optimization problem (cf.~Section~\ref{sec:afs:approach:objectives}).
Fourth, we analyze the time complexity of the optimization problem (cf.~Section~\ref{sec:afs:approach:complexity}).
Fifth, we propose and analyze heuristic search methods for the optimization problem with univariate feature qualities (cf.~Section~\ref{sec:afs:approach:univariate-heuristics}).

\subsection{Optimization Problem}
\label{sec:afs:approach:problem}

Alternative feature selection has two goals.
First, the quality of an alternative feature set should be high.
Second, an alternative feature set should differ from one or more other feature set(s).
There are several ways to combine these two goals in an optimization problem:

First, one can consider both goals as objectives, obtaining an unconstrained multi-objective problem.
Second, one can treat feature-set quality as objective and enforce alternatives with constraints.
Third, one can consider being alternative as objective and constrain feature-set quality, e.g., with a lower bound.
Fourth, one can define constraints for both, feature-set quality and being alternative, searching for feasible solutions instead of optimizing.

We stick to the second formulation, i.e., optimizing feature-set quality subject to being alternative.
This formulation has the advantage of keeping the original objective function of feature selection.
Thus, users do not need to specify a range or a threshold on feature-set quality but can control how alternative the feature sets must be instead.
We obtain the following optimization problem for a single alternative feature set~$F_s$:
\begin{equation}
	\begin{aligned}
		\max_s &\quad Q(s,X,y) \\
		\text{subject to:} &\quad F_s~\text{being alternative}
	\end{aligned}
	\label{eq:afs:afs-general}
\end{equation}
In the following, we discuss different objective functions $Q(s,X,y)$ and suitable constraints for \emph{being alternative}.
Additionally, many feature-selection methods also limit the feature-set size $|F_s|$ to a user-defined value~$k \in \mathbb{N}$, which adds a further, simple constraint to the optimization problem.

\subsection{Constraints -- Defining Alternatives}
\label{sec:afs:approach:constraints}

In this section, we formalize alternative feature sets.
First, we discuss the base case where an individual feature set is an alternative to another one (cf.~Section~\ref{sec:afs:approach:constraints:single}).
Second, we extend this notion to multiple alternatives, considering sequential and simultaneous search as two different search problems (cf.~Section~\ref{sec:afs:approach:constraints:multiple}).

Our notion of alternatives is independent of the feature-selection method.
We provide two parameters, i.e., a dissimilarity threshold~$\tau$ and the number of alternatives~$a$, allowing users to control the search for alternatives.

\subsubsection{Single Alternative}
\label{sec:afs:approach:constraints:single}

We consider a feature set an alternative to another feature set if it differs sufficiently.
Mathematically, we express this notion with a set-dissimilarity measure~\cite{choi2010survey, egghe2009new}.
These measures typically assess how strongly two sets overlap and relate this to their sizes.
E.g., a well-known set-dissimilarity measure is the Jaccard distance, which is defined as follows for the feature sets $F'$ and $F''$:
\begin{equation}
	d_{\text{Jacc}}(F',F'') = 1 - \frac{|F' \cap F''|}{|F' \cup F''|} = 1 - \frac{|F' \cap F''|}{|F'| + |F''| - |F' \cap F''|}
	\label{eq:afs:jaccard}
\end{equation}
In this article, we use a dissimilarity measure based on the Dice coefficient:
\begin{equation}
	d_{\text{Dice}}(F',F'') = 1 - \frac{2 \cdot |F' \cap F''|}{|F'| + |F''|}
	\label{eq:afs:dice}
\end{equation}
Generally, we only have mild assumptions on the set-dissimilarity measure~$d(\cdot)$.
Our subsequent definitions, examples, and propositions assume symmetry, i.e., $d(F',F'')=d(F'',F')$, normalization $d(\cdot) \in [0,1]$, and that $d(\cdot) = 1$ implies an empty intersection of the two sets.
In particular, $d(\cdot)$~does not need to be a metric but can also be a semi-metric~\cite{wilson1931semi} like~$d_{\text{Dice}}(\cdot)$.
In contrast to metrics, semi-metrics may violate the triangle inequality.

We leverage the set-dissimilarity measure for the following definition:
\begin{definition}[Single alternative]
	Given a symmetric set-dissimilarity measure~$d(\cdot) \in [0, 1]$ with $d(\cdot) = 1$ implying no set overlap, and a dissimilarity threshold~$\tau \in [0, 1]$, a feature set $F'$ is an alternative to a feature set~$F''$ (and vice versa) if $d(F',F'') \geq \tau$.
	\label{def:afs:single-alternative}
\end{definition}
The threshold~$\tau$ controls how alternative the feature sets must be and depends on the dataset as well as user preferences.
In particular, requiring strong dissimilarity may cause a significant drop in feature-set quality.
Some datasets may contain many features of similar utility, thereby enabling many alternatives of similar quality, while predictions on other datasets may depend on a few key features.
Only users can decide which drop in feature-set quality is acceptable as a trade-off for obtaining alternatives.
Thus, we leave $\tau$ as a user parameter.
In case the set-dissimilarity measure $d(\cdot)$ is normalized to $[0,1]$, like the Dice dissimilarity (cf.~Equation~\ref{eq:afs:dice}) or Jaccard distance (cf.~Equation~\ref{eq:afs:jaccard}), the interpretation of $\tau$ is user-friendly:
Setting $\tau=0$ allows identical alternatives, while $\tau=1$ implies zero overlap.

If the choice of $\tau$ is unclear a priori, users can try out different values and compare the resulting feature-set quality.
One systematic approach is a binary search:
Start with the mid-range value of $\tau=0$, i.e., 0.5 for $\tau \in [0,1]$.
If the quality of the resulting alternative is too low, decrease $\tau$ to 0.25, i.e., allow more similarity.
If the quality of the resulting alternative is acceptably high, increase $\tau$ to 0.75, i.e., check a more dissimilar feature set.
Continue this procedure till an alternative with an acceptable quality-dissimilarity trade-off is found.

When implementing Definition~\ref{def:afs:single-alternative}, the following proposition gives way to using a broad range of solvers to tackle the related optimization problem:
\begin{proposition}[Linearity of constraints for alternatives]
	Using the Dice dissimilarity (cf.~Equation~\ref{eq:afs:dice}), alternative feature sets (cf.~Definition~\ref{def:afs:single-alternative}) can be expressed with 0-1 integer linear constraints.
	\label{prop:afs:linear-constraints}
\end{proposition}
\begin{proof}
We re-arrange terms in the Dice dissimilarity (cf.~Equation~\ref{eq:afs:dice}) to eliminate the quotient of feature-set sizes:
\begin{equation}
	\begin{aligned}
		d_{\text{Dice}}(F',F'') &= & 1 - \frac{2 \cdot |F' \cap F''|}{|F'| + |F''|} &\geq \tau \\
		&\Leftrightarrow & |F' \cap F''| &\leq \frac{1 - \tau}{2} \cdot (|F'| + |F''|)
	\end{aligned}
	\label{eq:afs:dice-rearranged}
\end{equation}
Next, we express the set sizes in terms of the feature-selection vector $s$:
\begin{equation}
	\begin{aligned}
		|F_s| =& \sum_{j=1}^n s_j \\
		|F_{s'} \cap F_{s''}| =& \sum_{j=1}^n s'_j \cdot s''_j
	\end{aligned}
	\label{eq:afs:feature-set-size}
\end{equation}
Finally, we replace each product $s'_j \cdot s''_j$ with an auxiliary variable~$t_j$, bound by additional constraints, to linearize it~\cite{mosek2022modeling}:
\begin{equation}
	\begin{aligned}
		t_j \leq& s'_j \\
		t_j \leq& s''_j \\
		1 + t_j \geq& s'_j + s''_j \\
		t_j \in& \{0,1\}
	\end{aligned}
	\label{eq:afs:product-linear}
\end{equation}
Combining Equations~\ref{eq:afs:dice-rearranged},~\ref{eq:afs:feature-set-size}, and~\ref{eq:afs:product-linear}, we obtain a set of constraints that only involve linear expressions of binary decision variables.
In particular, there are only sum expressions and multiplications with constants but no products between variables.
If one feature set is known, i.e., either $s'$ or $s''$ is fixed, Equation~\ref{eq:afs:feature-set-size} only multiplies variables with constants and is already linear without Equation~\ref{eq:afs:product-linear}.
\end{proof}
Given a suitable objective function, which we discuss later, linear constraints allow using a broad range of solvers.
As an alternative formulation, one could also encode such constraints into propositional logic (\textsc{SAT})~\cite{ulrich2022selecting}.

If the set sizes $|F'|$ and $|F''|$ are constant, e.g., user-defined, Equation~\ref{eq:afs:dice-rearranged} implies that the threshold~$\tau$ has a linear relationship to the maximum number of overlapping features~$|F' \cap F''|$.
This correspondence eases the interpretation of~$\tau$ and makes us use the Dice dissimilarity in the following.
In contrast, the Jaccard distance exhibits a non-linear relationship between $\tau$ and the overlap size, which follows from re-arranging Equation~\ref{eq:afs:jaccard} in combination with Definition~\ref{def:afs:single-alternative}:
\begin{equation}
	\begin{aligned}
		d_{\text{Jacc}}(F',F'') &= & 1 - \frac{|F' \cap F''|}{|F'| + |F''| - |F' \cap F''|} &\geq \tau \\
		&\Leftrightarrow & |F' \cap F''| &\leq \frac{1 - \tau}{2 - \tau} \cdot (|F'| + |F''|)
	\end{aligned}
	\label{eq:afs:jaccard-rearranged}
\end{equation}
Further, if $|F'| = |F''|$, as in our experiments, the Dice dissimilarity (cf.~Equation~\ref{eq:afs:dice-rearranged}) becomes identical to several other set-dissimilarity measures~\cite{egghe2009new}.
The parameter~$\tau$ then directly expresses which fraction of features in one set needs to differ from the other set and vice versa, which further eases interpretability:
\begin{equation}
	d_{\text{Dice}}(F',F'') \geq \tau \Leftrightarrow |F' \cap F''| \leq (1 - \tau) \cdot |F'| = (1 - \tau) \cdot |F''|
	\label{eq:afs:dice-rearranged-equal-size}
\end{equation}
Thus, if users are uncertain how to choose $\tau$ and $|F'|$ is reasonably small, they can try out all values of $\tau \in \{l / |F'|\}$ with $l \in \{1, \dots, |F'|\}$.
In particular, these $|F'|$~unique values of $\tau$ suffice to produce all distinct solutions that one could obtain with an arbitrary $\tau \in (0,1]$.

\subsubsection{Multiple Alternatives}
\label{sec:afs:approach:constraints:multiple}

If users desire multiple alternative feature sets rather than only one, we can determine these alternatives sequentially or simultaneously.
The number of alternatives~$a \in \mathbb{N}_0$ is a parameter to be set by the user.
The overall number of feature sets is $a + 1$ since we deem one feature set the `original' one.
Table~\ref{tab:afs:seq-sim-comparison} compares the sizes of the optimization problems for these two search problems.

\begin{table}[t]
	\centering
	\caption{Size of the optimization problem, i.e., number of variables and constraints, for $a$~alternatives ($a + 1$~feature sets overall) and $n$ features.}
	\renewcommand*{\arraystretch}{1.3}
	\begin{tabular}{lccc}
		\toprule
		& \multicolumn{2}{c}{Sequential search} & \multirow{2}{*}{Simult. search} \\
		\cmidrule(lr){2-3}
		& $l$-th Alternative & Summed & \\
		\midrule
		Decision variables~$s$ & $n$ & $ (a+1) \cdot n$ & $(a+1) \cdot n$ \\
		Linearization variables~$t$ & $0$ & $0$ & $\frac{a \cdot (a+1) \cdot n}{2}$ \\
		Alternative constraints & $l$ & $\frac{a \cdot (a+1)}{2}$ & $\frac{a \cdot (a+1)}{2}$ \\
		Linearization constraints & $0$ & $0$ & $\frac{3 \cdot a \cdot (a+1) \cdot n}{2}$ \\
		\bottomrule
	\end{tabular}
	\label{tab:afs:seq-sim-comparison}
\end{table}

\paragraph{Sequential-search problem}

In the sequential-search problem, users obtain several alternatives iteratively, with one feature set per iteration.
We constrain this new set to be an alternative to all previously found ones, which are given in the set~$\mathbb{F}$:
\begin{definition}[Sequential alternative]
	A feature set~$F''$ is an alternative to a set of feature sets~$\mathbb{F}$ (and vice versa) if $F''$ is a single alternative (cf.~Definition~\ref{def:afs:single-alternative}) to each $F' \in \mathbb{F}$.
	\label{def:afs:sequential-alternative}
\end{definition}
One could also think of less strict constraints, e.g., requiring only the average dissimilarity to all previously found feature sets to pass a threshold~$\tau$.
However, definitions like the latter may allow some feature sets to overlap heavily or even be identical if other feature sets are very dissimilar.
Thus, we require pairwise dissimilarity in Definition~\ref{def:afs:sequential-alternative}.
Combining Equation~\ref{eq:afs:afs-general} with Definition~\ref{def:afs:sequential-alternative}, we obtain the following optimization problem for each iteration of the search:
\begin{equation}
	\begin{aligned}
		\max_s &\quad Q(s,X,y) \\
		\text{subject to:} &\quad \forall F' \in \mathbb{F}:~d(F_s,F') \geq \tau
	\end{aligned}
	\label{eq:afs:afs-sequential}
\end{equation}
The full textual problem definition corresponding to Equation~\ref{eq:afs:afs-sequential} is the following:
\begin{definition}[Sequential-search problem for one alternative feature set]
	Given
	\begin{itemize}[noitemsep]
		\item a dataset~$X \in \mathbb{R}^{m \times n}$ with prediction target~$y \in Y^m$,
		\item a set~$\mathbb{F}$ of existing feature sets for~$X$,
		\item a symmetric set-dissimilarity measure~$d(\cdot) \in [0,1]$ with $d(\cdot) = 1 \rightarrow$ no set overlap,
		\item and a dissimilarity threshold~$\tau \in [0,1]$,
	\end{itemize}
	sequential search for one alternative feature set is the problem of making feature-selection decisions~$s \in \{0,1\}^n$ that maximize a given notion of feature-set quality~$Q(s,X,y)$ while making the corresponding feature set~$F_s$ a sequential alternative to~$\mathbb{F}$ (cf.~Definition~\ref{def:afs:sequential-alternative}).
	\label{def:afs:alternative-feature-selection-sequential}
\end{definition}
The objective function remains the same as for a single alternative ($|\mathbb{F}| = 1$), i.e., we only optimize the quality of one feature set at once.
In particular, with $\mathbb{F} = \emptyset$ in the first iteration, we optimize for the `original' feature set, which is the same as in conventional feature selection without constraints for alternatives.
Thus, the number of variables in the optimization problem is independent of the number of alternatives~$a$.
Instead, we solve the optimization problem repeatedly; each alternative only adds one constraint to the problem.
As we always compare only one variable feature set to existing, constant feature sets, we also do not need to introduce auxiliary variables as in Equation~\ref{eq:afs:product-linear}.
Thus, we expect the runtime of exact, e.g., solver-based, sequential search to scale well with the number of alternatives.
Further runtime gains may arise if the solver keeps a state between iterations and can warm-start.

However, as the search space becomes narrower over iterations, feature-set quality can deteriorate with each further alternative.
In particular, multiple alternatives from the same sequential search might differ significantly in their quality.
As a remedy, users can decide after each iteration if the feature-set quality is already unacceptably low or if another alternative should be found.
In particular, users do not need to define the number of alternatives~$a$ a priori.

\paragraph{Simultaneous-search problem}

In the simultaneous-search problem, users obtain multiple alternatives at once, so they need to decide on the number of alternatives~$a$ beforehand.
We use pairwise dissimilarity constraints for alternatives again:
\begin{definition}[Simultaneous alternatives]
	A set of feature sets~$\mathbb{F}$ contains simultaneous alternatives if each feature set~$F' \in \mathbb{F}$ is a single alternative (cf.~Definition~\ref{def:afs:single-alternative}) to each other feature set~$F'' \in \mathbb{F}$ with $F' \neq F''$.
	\label{def:afs:simultaneous-alternative}
\end{definition}
Combining Equation~\ref{eq:afs:afs-general} with Definition~\ref{def:afs:simultaneous-alternative}, we obtain the following optimization problem for $a+1$ feature sets:
\begin{equation}
	\begin{aligned}
		\max_{s^{(0)}, \dots, s^{(a)}} &\quad \operatorname*{agg}_{l \in \{0, \dots, a\}} Q(s^{(l)},X,y) \\
		\text{subject to:} &\quad \forall l_1, l_2 \in \{0, \dots, a\},~l_1 \neq l_2:~d(F_{s^{(l_1)}},F_{s^{(l_2)}}) \geq \tau
	\end{aligned}
	\label{eq:afs:afs-simultaneous}
\end{equation}
The full textual problem definition corresponding to Equation~\ref{eq:afs:afs-simultaneous} is the following:
\begin{definition}[Simultaneous-search problem for alternative feature sets]
	Given
	\begin{itemize}[noitemsep]
		\item a dataset~$X \in \mathbb{R}^{m \times n}$ with prediction target~$y \in Y^m$,
		\item the number of alternatives~$a \in \mathbb{N}_0$,
		\item an aggregation operator $\text{agg}(\cdot): \mathbb{R}^{a+1} \to \mathbb{R}$ for feature-set qualities,
		\item a symmetric set-dissimilarity measure~$d(\cdot) \in [0,1]$ with $d(\cdot) = 1 \rightarrow$ no set overlap,
		\item and a dissimilarity threshold~$\tau \in [0,1]$,
	\end{itemize}
	simultaneous search for alternative feature sets is the problem of making feature-selection decisions~$s^{(l)} \in \{0,1\}^n$ for $l \in \{0, \dots, a\}$ that maximize a given notion of feature-set quality~$Q(s,X,y)$ aggregated over the alternatives with $\operatorname*{agg}_{l \in \{0, \dots, a\}} Q(s^{(l)},X,y)$ while making the corresponding feature sets $\mathbb{F} = \{F_{s^{(0)}},$ $\dots, F_{s^{(a)}}\}$ simultaneous alternatives (cf.~Definition~\ref{def:afs:simultaneous-alternative}).
	\label{def:afs:alternative-feature-selection-simultaneous}
\end{definition}
In contrast to the sequential case (cf.~Definition~\ref{def:afs:alternative-feature-selection-sequential}), the problem requires $a+1$ instead one decision vector~$s$, and a modified objective function.
The operator~$\text{agg}(\cdot)$ defines how to aggregate the feature-set qualities of the alternatives.
In our experiments, we consider the sum as well as the minimum to instantiate~$\text{agg}(\cdot)$, which we refer to as \emph{sum-aggregation} and \emph{min-aggregation}.
The latter explicitly fosters balanced feature-set qualities.
Appendix~\ref{sec:afs:appendix:simultaneous-objective-aggregation} discusses these two aggregation operators and additional ideas for balancing qualities in detail.

Runtime-wise, we expect exact simultaneous search to scale worse with the number of alternatives than exact sequential search, as it tackles one large optimization problem instead of multiple smaller ones.
In particular, the number of decision variables increases linearly with the number of alternatives~$a$.
Also, for each feature and each pair of alternatives, we need to introduce an auxiliary variable if we want to obtain linear constraints (cf.~Equation~\ref{eq:afs:product-linear} and Table~\ref{tab:afs:seq-sim-comparison}).

In contrast to the greedy definition of sequential search, simultaneous search optimizes alternatives globally.
Thus, the simultaneous procedure should yield the same or higher average feature-set quality for the same number of alternatives.
Also, the quality can be more evenly distributed over the alternatives, as opposed to the dropping quality over the course of the sequential procedure.
However, increasing the number of alternatives still has a negative effect on the average feature-set quality.
Further, as opposed to the sequential procedure, there are no intermediate steps where users could interrupt the search.

\subsection{Objective Functions -- Finding Alternatives}
\label{sec:afs:approach:objectives}

In this section, we discuss how to find alternative feature sets.
In particular, we describe how to solve the optimization problem from Section~\ref{sec:afs:approach:problem} for the different categories of feature-set quality measures from Section~\ref{sec:afs:fundamentals:quality}.
We distinguish between white-box optimization (cf.~Section~\ref{sec:afs:approach:objectives:white-box}), black-box optimization (cf.~Section~\ref{sec:afs:approach:objectives:black-box}), and embedding alternatives (cf.~Section~\ref{sec:afs:approach:objectives:embedding}).

\subsubsection{White-Box Optimization}
\label{sec:afs:approach:objectives:white-box}

If the feature-set quality function~$Q(s,X,y)$ is sufficiently simple, one can tackle alternative feature selection with a suitable solver for white-box optimization problems.
We already showed that our notion of alternative feature sets results in 0-1 integer linear constraints (cf.~Proposition~\ref{prop:afs:linear-constraints}).
We now discuss several feature-selection methods with objectives that admit formulating a 0-1 integer linear problem.
Appendix~\ref{sec:afs:appendix:multivariate-filter-objectives}~describes feature-selection methods we did not include in our experiments.

\paragraph{Univariate filter feature selection}

For univariate filter feature selection, the objective function is linear by default.
In particular, these methods decompose the quality of a feature set into the qualities of the individual features:
\begin{equation}
	\max_s \quad Q_{\text{uni}}(s,X,y) = \sum_{j=1}^{n} q(X_{\cdot{}j},y) \cdot s_j = \sum_{j=1}^{n} q_j \cdot s_j
	\label{eq:afs:univariate-filter}
\end{equation}
Here, $q(\cdot)$ typically is a bivariate dependency measure, e.g., mutual information~\cite{kraskov2004estimating} or the absolute value of Pearson correlation, to quantify the relationship between one feature and the prediction target.

For this objective, Appendix~\ref{sec:afs:appendix:univariate-complete-optimization-problem} specifies the complete optimization problem, including the constraints for alternatives.
Appendix~\ref{sec:afs:appendix:univariate-pre-selection} describes how to potentially speed up optimization by leveraging the monotonicity of the objective.
Section~\ref{sec:afs:approach:univariate-heuristics} proposes heuristic search methods for this objective.

Instead of an integer problem, one could formulate a weighted partial maximum satisfiability (\textsc{MaxSAT}) problem~\cite{bacchus2021maximum, li2021maxsat}, i.e., a weighted \textsc{Max One} problem~\cite{khanna1997complete}.
In particular, Equation~\ref{eq:afs:univariate-filter} is a sum of weighted binary variables, and the constraints for alternatives can be turned into SAT formulas with a cardinality encoding~\cite{sinz2005towards} for the sum expressions.

\paragraph{Post-hoc feature importance}

Technically, one can also insert values of post-hoc feature-importance scores into Equation~\ref{eq:afs:univariate-filter}.
For example, one can pre-compute permutation importance~\cite{breiman2001random} or SAGE scores~\cite{covert2020understanding} for each feature and use them as univariate feature qualities~$q(X_{\cdot{}j},y)$.
However, such post-hoc importance scores typically evaluate the quality of each feature in the presence of other features. 
For example, a feature may only be important in subsets where another feature is present, due to feature interaction, but unimportant otherwise, and a post-hoc importance method like SHAP~\cite{lundberg2017unified} may reflect both these aspects.
In contrast, Equation~\ref{eq:afs:univariate-filter} implicitly assumes feature independence and cannot adapt importance scores depending on whether other features are selected.
Thus, treating pre-computed post-hoc importance scores as univariate feature qualities in the optimization objective can serve as a heuristic but may not faithfully represent the feature qualities in a particular feature subset~\cite{fryer2021shapley}.

\paragraph{FCBF}

The Fast Correlation-Based Filter (FCBF)~\cite{yu2003feature} bases on the notion of predominance:
Each selected feature's correlation with the prediction target must exceed a user-defined threshold as well as the correlation of each other selected feature with the given one.
While the original FCBF uses a heuristic search to find predominant features, we propose a formulation as a constrained optimization problem to enable a white-box optimization for alternatives:
\begin{equation}
	\begin{aligned}
		\max_s &\quad Q_{\text{FCBF}}(s,X,y) = \sum_{j=1}^{n} q(X_{\cdot{}j},y) \cdot s_j \\
		\text{subject to:} &\quad \forall j_1, j_2 \in \{1, \dots, n\},~j_1 \neq j_2,~(*): s_{j_1} + s_{j_2} \leq 1 \\
		\text{with } (*) \text{:} &\quad q(X_{\cdot{}j_1},y) \leq q(X_{\cdot{}j_2}, X_{\cdot{}j_1}) \\
	\end{aligned}
	\label{eq:afs:fcbf}
\end{equation}
We drop the original FCBF's threshold on feature-target correlation and maximize the latter instead, as in the univariate-filter case.
This change could produce large feature sets that contain many low-quality features.
As a countermeasure, one can constrain the feature-set sizes, as we do in our experiments.
Additionally, one could also filter out the features with low target correlation before optimization.
Further, we keep FCBF's constraints on feature-feature correlation.
In particular, we prevent the simultaneous selection of two features if the correlation between them is at least as high as one of the features' correlation to the target.
Since the condition~$q(X_{\cdot{}j_1},y) \leq q(X_{\cdot{}j_2}, X_{\cdot{}j_1})$ in Equation~\ref{eq:afs:fcbf} does not depend on the decision variables~$s$, one can check whether it holds before formulating the optimization problem and add the corresponding linear constraint $s_{j_1} + s_{j_2} \leq 1$ only for feature pairs where it is needed.

\paragraph{mRMR}

Minimal Redundancy Maximal Relevance (mRMR)~\cite{peng2005feature} combines two criteria, i.e., feature relevance and feature redundancy.
Relevance corresponds to the dependency between features and prediction target, which should be maximized, as for univariate filters.
Redundancy, in turn, corresponds to the dependency between features, which should be minimized.
Both terms are averaged over the selected features.
Using a bivariate dependency measure~$q(\cdot)$, the objective is maximizing the following difference between relevance and redundancy:
\begin{equation}
	\begin{aligned}
		\max_s \quad Q_{\text{mRMR}}(s,X,y) &= \frac{\sum_{j=1}^{n} q(X_{\cdot{}j},y) \cdot s_j}{\sum_{j=1}^{n} s_j} \\
		&- \frac{\sum_{j_1=1}^{n} \sum_{j_2=1}^{n} q(X_{\cdot{}j_1}, X_{\cdot{}j_2}) \cdot s_{j_1} \cdot s_{j_2}}{(\sum_{j=1}^{n} s_j)^2}
	\end{aligned}
	\label{eq:afs:mrmr}
\end{equation}
If one knows the feature-set size $\sum_{j=1}^{n} s_j$ to be a constant~$k$, the denominators of both fractions are constant, so the objective leads to a quadratic-programming problem~\cite{nguyen2014effective, rodriguez2010quadratic}.
If one additionally replaces each product term $s_{j_1} \cdot s_{j_2}$ according to Equation~\ref{eq:afs:product-linear}, the problem becomes linear.
However, there is a more efficient linearization~\cite{nguyen2009optimizing, nguyen2010towards}, which we use in our experiments:
\begin{equation}
	\begin{aligned}
		\max_s &\quad & Q_{\text{mRMR}}(s,X,y) &= \frac{\sum_{j=1}^{n} q(X_{\cdot{}j},y) \cdot s_j}{k} - \frac{\sum_{j=1}^{n} z_j}{k \cdot (k-1)} \\
		\text{subject to:} &\quad \forall j_1: & A_{j_1} &:= \sum_{j_2 \neq j_1} q(X_{\cdot{}j_1}, X_{\cdot{}j_2}) \cdot s_{j_2} \\
		&\quad \forall j: & z_j &\geq M \cdot (s_j - 1) + A_j \\
		&\quad \forall j: & z_j &\in \mathbb{R}_{\geq 0} \\
		\text{with indices:} &\quad & j, j_1, j_2 &\in \{1, \dots, n\}
	\end{aligned}
	\label{eq:afs:mrmr-linear}
\end{equation}
Here, $A_{j_1}$ is the sum of all redundancy terms related to the feature with index~$j_1$, i.e., the summed dependency value between this feature and all other selected features.
Thus, one can use one real-valued auxiliary variable $z_j$ for each feature instead of one new binary variable for each pair of features.
$A_j$ does not need a separate variable but can be directly inserted in the subsequent constraint, so we wrote `$:=$' instead of `$=$'.
Since redundancy should be minimized, $z_j$ assumes the value of $A_j$ with equality if the feature with index~$j$ is selected~($s_j=1$) and is zero otherwise ($s_j=0$).
To this end, $M$ is a large positive value that deactivates the constraint $z_j \geq A_j$ if $s_j=0$.

Since Equation~\ref{eq:afs:mrmr-linear} assumes the feature-set size~$k \in \mathbb{N}$ to be user-defined before optimization, it requires fewer auxiliary variables and constraints than the more general formulation in~\cite{nguyen2009optimizing, nguyen2010towards}.
Additionally, in accordance with~\cite{nguyen2014effective}, we assign a value of zero to the self-redundancy terms $q(X_{\cdot{}j},X_{\cdot{}j})$, effectively excluding them from the objective function.
Thus, the redundancy term uses $k \cdot (k-1)$ instead of $k^2$ for averaging.

\subsubsection{Black-Box Optimization}
\label{sec:afs:approach:objectives:black-box}

If feature-set quality does not have an expression suitable for white-box optimization, one has to treat it as a black-box function when searching for alternatives.
This situation applies to wrapper feature-selection methods, which use prediction models to assess feature-set quality.
One can optimize such black-box functions with search heuristics that systematically iterate over candidate feature sets.
However, search heuristics often assume an unconstrained search space and may propose candidate feature sets that are not alternative enough.
We see four ways to address this issue:

\paragraph{Enumerating feature sets}

Instead of using a search heuristic, one may enumerate all feature sets that are alternative enough.
E.g., one can iterate over all feature sets and sort out those violating the constraints or use a solver to enumerate all valid alternatives directly.
Both approaches are usually very inefficient, as there can be a vast number of alternatives.

\paragraph{Sampling feature sets}

Instead of considering all possible alternatives, one can also sample a limited number.
E.g., one could sample from all feature sets but remove samples that are not alternative enough.
However, if the number of valid alternatives is small, this approach might need many samples.
One could also sample with the help of a solver.
However, uniform sampling from a constrained space is a computationally hard problem, possibly harder than determining if a valid solution exists or not~\cite{ermon2012uniform}.

\paragraph{Multi-objective optimization}

If one phrases alternative feature selection as a multi-objective problem (cf.~Section~\ref{sec:afs:approach:problem}), there are no hard constraints anymore, and one could apply a standard multi-objective black-box search procedure.
However, as explained in Section~\ref{sec:afs:approach:problem}, we decided to pursue a single-objective formulation with constraints.

\paragraph{Adapting search}

One can adapt an existing search heuristic to consider the constraints for alternatives.
One idea is to prevent the search from producing feature sets that violate the constraints or at least make the latter less likely, e.g., with a penalty in the objective function.
Another idea is to `repair' feature sets in the search that violate constraints, e.g., replacing them with the most similar valid feature sets.
Such solver-assisted search approaches are common in search methods for software feature models~\cite{guo2018preserve, henard2015combining, white2010automated}, and our following method for wrapper feature selection falls into this category as well.
Finally, one could also apply solver-based repair to sampled feature sets.

\begin{algorithm}[t]
	\DontPrintSemicolon
	\KwIn{Dataset $X \in \mathbb{R}^{m \times n}$, \newline
		Prediction target $y \in Y^m$, \newline
		Quality function $Q(S,X,y)$ for sets of feature sets, \newline
		Set~$C$ of constraints for alternatives, \newline
		Maximum number of iterations $\mathit{max\_iters} \in \mathbb{N}$}
	\KwOut{Set of feature-selection decision vectors $S = \{s^{(0)}, \dots, s^{(a)}\}$}
	\BlankLine
	$S^{\text{opt}} \leftarrow \text{solve}(C)$ \tcp*{Initial alternatives} \label{al:afs:greedy-wrapper:line:init}
	$\mathit{iters} \leftarrow 1$ \tcp*{Number of iterations = solver calls}
	\lIf(\tcp*[f]{No valid alternatives exist}){$S^{\text{opt}} = \emptyset$}{\Return{$\emptyset$}} \label{al:afs:greedy-wrapper:line:no-valid}
	$j_1 \leftarrow 1$ \tcp*{Indices of features to be swapped} \label{al:afs:greedy-wrapper:line:swap-init-start}
	$j_2 \leftarrow j_1 + 1$\; \label{al:afs:greedy-wrapper:line:swap-init-end}
	\While{$\mathit{iters} < \mathit{max\_iters}$ \textbf{and} $j_1 < n$}{ \label{al:afs:greedy-wrapper:line:stop}
		$S \leftarrow $ solve(Equation~\ref{eq:afs:greedy-wrapper-problem-linear}) \tcp*{Try swap} \label{al:afs:greedy-wrapper:line:swap}
		$\mathit{iters} \leftarrow \mathit{iters} + 1$\;
		\If(\tcp*[f]{Swap if improved}){$S \neq \emptyset$ \textbf{and} $Q(S,X,y) > Q(S^{\text{opt}},X,y)$}{ \label{al:afs:greedy-wrapper:line:improved-condition}
			$S^{\text{opt}} \leftarrow S$\; \label{al:afs:greedy-wrapper:line:improved-start}
			$j_1 \leftarrow 1$ \tcp*{Reset swap-feature indices}
			$j_2 \leftarrow j_1 + 1$\; \label{al:afs:greedy-wrapper:line:improved-end}
		}
		\ElseIf(\tcp*[f]{Try next swap; advance one index}){$j_2 < n$}{ \label{al:afs:greedy-wrapper:line:next-start}
			$j_2 \leftarrow j_2 + 1$\;
		}
		\Else(\tcp*[f]{Try next swap; advance both indices}){
			$j_1 \leftarrow j_1 + 1$\;
			$j_2 \leftarrow j_1 + 1$\; \label{al:afs:greedy-wrapper:line:next-end}
		}
	}
	\Return{$S^{\text{opt}}$}
	\caption{\emph{Greedy Wrapper} for alternative feature selection.}
	\label{al:afs:greedy-wrapper}
\end{algorithm}

\paragraph{Greedy Wrapper}

For wrapper feature selection in our experiments, we propose a novel hill-climbing procedure, displayed in Algorithm~\ref{al:afs:greedy-wrapper}.
Unlike standard hill climbing for feature selection~\cite{kohavi1997wrappers}, our procedure observes constraints.
First, the algorithm uses a solver to find one solution that is alternative enough for the set of constraints~$C$ (Line~\ref{al:afs:greedy-wrapper:line:init}) and stores it as the currently best solution~$S^{\text{opt}}$.
Thus, the algorithm has a valid starting point and can always return a solution unless there are no valid solutions at all (Line~\ref{al:afs:greedy-wrapper:line:no-valid}).
Note that the solution is not only one feature-selection decision vector~$s$ but a set~$S$ of them, to enable simultaneous search.
For sequential search, $|S| = 1$ and $a=0$.
Also, we adapt the notion of feature-set quality~$Q(S,X,y)$ in this algorithm to support a set of feature sets, encompassing the aggregation operator~$\text{agg}(\cdot)$ for simultaneous search (cf.~Definition~\ref{def:afs:alternative-feature-selection-simultaneous}).

Next, the algorithm tries `swapping' two features, i.e., selecting them if they were deselected or deselecting them if they were selected (Line~\ref{al:afs:greedy-wrapper:line:swap}).
The corresponding swap indices~$j_1$ and $j_2$ start at~$1$ and~$2$, respectively (Lines~\ref{al:afs:greedy-wrapper:line:swap-init-start}--\ref{al:afs:greedy-wrapper:line:swap-init-end}).
For simultaneous search, we swap the affected features in each alternative.
As the swap may violate constraints, the algorithm calls the solver to find the solution~$S$ that is closest to the currently best one $S^{\text{opt}}$ while satisfying the swap constraints and the constraints for alternatives~$C$.
To this end, we measure the similarity between feature-selection decisions with the Hamming distance~\cite{choi2010survey}, i.e., how many values of decision variables differ between~$S$ and~$S^{\text{opt}}$.
Overall, we define the maximization problem for Line~\ref{al:afs:greedy-wrapper:line:swap} of Algorithm~\ref{al:afs:greedy-wrapper} as follows:

\begin{equation}
	\begin{aligned}
		\max_{s^{(0)}, \dots, s^{(a)}} &\quad & \text{sim}(S, S^{\text{opt}}) &= \sum_{l=0}^{a} \sum_{j=1}^{n} \left( s^{(l)}_j \leftrightarrow s^{(\text{opt, } l)}_{j} \right) \\
		\text{subject to:} &\quad & C & \\
		&\quad \forall l \in \{0, \dots, a\}: & s^{(l)}_{j_1} &\leftrightarrow \neg s^{(\text{opt, } l)}_{j_1} \\
		&\quad \forall l \in \{0, \dots, a\}: & s^{(l)}_{j_2} &\leftrightarrow \neg s^{(\text{opt, } l)}_{j_2} \\
	\end{aligned}
	\label{eq:afs:greedy-wrapper-problem}
\end{equation}
The values of $s^{(\text{opt, } l)}$, $j_1$, and $j_2$ in this problem are fixed based on Algorithm~\ref{al:afs:greedy-wrapper}, while $s^{(l)}$ remains variable.
Thus, we obtain a 0-1 integer linear program:

\begin{equation}
	\begin{aligned}
		\max_{s^{(0)}, \dots, s^{(a)}} &&\quad \text{sim}(S, S^{\text{opt}}) &\quad = \sum_{l=0}^{a} \Big( \sum\limits_{\substack{j \in \{1, \dots, n\} \\ s^{(\text{opt, } l)}_{j} = 1}} s^{(l)}_j + \sum\limits_{\substack{j \in \{1, \dots, n\} \\ s^{(\text{opt, } l)}_{j} = 0}} \left( 1- s^{(l)}_j \right) \Big) \\
		\text{subject to:} &&\quad &\quad C \\
		&&\quad \forall l \in \{0, \dots, a\}: &\quad s^{(l)}_{j_1} = 1 - s^{(\text{opt, } l)}_{j_1} \\
		&&\quad \forall l \in \{0, \dots, a\}: &\quad s^{(l)}_{j_2} = 1 - s^{(\text{opt, } l)}_{j_2} \\
	\end{aligned}
	\label{eq:afs:greedy-wrapper-problem-linear}
\end{equation}

If a solution~$S$ for Equation~\ref{eq:afs:greedy-wrapper-problem-linear} exists and its quality~$Q(S,X,y)$ improves upon the currently best solution~$S^{\text{opt}}$, the algorithm proceeds with the new solution, attempting again to swap the first and second features (Lines~\ref{al:afs:greedy-wrapper:line:improved-start}--\ref{al:afs:greedy-wrapper:line:improved-end}).
Otherwise, it tries to swap another pair of features (Lines~\ref{al:afs:greedy-wrapper:line:next-start}--\ref{al:afs:greedy-wrapper:line:next-end}).
Specifically, we assess only one solution per swap instead of exhaustively enumerating and evaluating all valid solutions involving the swap.

The algorithm terminates if it reaches a local optimum, i.e., no swap leads to an improvement, or a fixed number of iterations~$\mathit{max\_iters}$ (Line~\ref{al:afs:greedy-wrapper:line:stop}).
We define the iteration count as the number of solver calls, i.e., attempts to generate valid alternatives.
This number also bounds the number of prediction models trained.
However, we only train a model for valid solutions (Line~\ref{al:afs:greedy-wrapper:line:improved-condition}), and not all solver invocations may yield one.

\subsubsection{Embedding Alternatives}
\label{sec:afs:approach:objectives:embedding}

If feature selection is embedded into a prediction model, there is no general approach for finding alternative feature sets.
Instead, one would need to embed the search for alternatives into model training as well.
Thus, we leave the formulation of specific approaches open for future work.
E.g., one could adapt the training of decision trees to not split on a feature if the resulting feature set of the tree was too similar to a given feature set.
As another example, there are various formal encodings of prediction models, e.g., as \textsc{SAT} formulas~\cite{narodytska2018learning, schidler2021sat, yu2021learning}, where `training' already uses a solver.
In such representations, one may directly add constraints for alternatives.

\subsection{Time Complexity}
\label{sec:afs:approach:complexity}

In this section, we analyze the time complexity of alternative feature selection.
In particular, we study the scalability regarding the number of features~$n \in \mathbb{N}$, also considering the feature-set size~$k \in \mathbb{N}$ and the number of alternatives~$a \in \mathbb{N}_0$.
Section~\ref{sec:afs:approach:complexity:exhaustive} discusses exhaustive search, which works for arbitrary feature-selection methods, while Section~\ref{sec:afs:approach:complexity:univariate} examines the optimization problem with univariate feature qualities (cf.~Equation~\ref{eq:afs:univariate-filter}).
Section~\ref{sec:afs:approach:complexity:summary} summarizes key results.

\subsubsection{Exhaustive Search for Arbitrary Feature-Selection Methods}
\label{sec:afs:approach:complexity:exhaustive}

An exhaustive search over the entire search space is the arguably simplest though inefficient approach to finding alternative feature sets.
This approach provides an upper bound for the time complexity of a runtime-optimal search algorithm.
In this section, we assume unit costs for elementary arithmetic operations like addition, multiplication, and comparison of two numbers.

\paragraph{Conventional feature selection}

In general, the search space of feature selection grows exponentially with~$n$, even without alternatives.
In particular, there are $2^n - 1$ possibilities to form a single non-empty feature set of arbitrary size.
For a fixed feature-set size~$k$, there are $\binom{n}{k} = \frac{n!}{k! \cdot (n-k)!} \leq n^k$ solution candidates.
In an exhaustive search, we iterate over these feature sets:
\begin{proposition}[Complexity of exhaustive conventional feature selection]
	Exhaustive search for one feature set of size~$k \in \mathbb{N}$ from $n$~features has a time complexity of~$O(n^k)$ without the cost of evaluating the objective.
	\label{prop:afs:complexity-exhaustive-conventional}
\end{proposition}
Evaluating the objective means computing the quality of each solution candidate so that we can determine the best feature set in the end.
The cost of this step depends on the feature-selection method but should usually be polynomial in~$n$.
Even better, since feature-set quality typically only depends on selected features rather than unselected ones, this cost may be polynomial in~$k \ll n$.

If we assume $k \ll n,~k \in O(1)$, i.e., $k$ being a small constant, independent from~$n$, then the complexity in Proposition~\ref{prop:afs:complexity-exhaustive-conventional} is polynomial rather than exponential in~$n$.
This assumption makes sense for feature selection, where one typically wants to obtain a small feature set from a high-dimensional dataset.
However, the exponent~$k$ may still render an exhaustive search practically infeasible.
In terms of parameterized complexity, the problem resides in class~$\mathcal{XP}$ since the complexity term has the form $O(f(k) \cdot n^{g(k)})$~\cite{downey1997parameterized}, here with parameter~$k$ and functions $f(k) = 1$, $g(k) = k$.

\paragraph{Sequential search}

Like conventional feature selection, sequential search for alternatives (cf.~Definition~\ref{def:afs:sequential-alternative}) optimizes feature sets one at a time.
However, not all size-$k$ feature sets are valid anymore.
In particular, the constraints for alternatives put an extra cost on each solution candidate.
Constraint checking involves iterating over all existing feature sets and features to compute the dissimilarity between sets (cf.~Equation~\ref{eq:afs:afs-sequential-complete}).
This procedure entails a cost of~$O(a \cdot n)$ for each new alternative and~$O((a+1)^2 \cdot n)$ for the whole sequential search with $a$~alternatives.
Combining this cost with Proposition~\ref{prop:afs:complexity-exhaustive-conventional}, we obtain the following proposition:
\begin{proposition}[Complexity of exhaustive sequential search]
	Exhaustive sequential search (cf.~Definition~\ref{def:afs:alternative-feature-selection-sequential}) for $a \in \mathbb{N}_0$~alternative feature sets of size~$k \in \mathbb{N}$ from $n$~features has a time complexity of~$O((a+1)^2 \cdot n^{k+1})$ without the cost of evaluating the objective.
	\label{prop:afs:complexity-exhaustive-sequential}
\end{proposition}
Thus, the runtime remains polynomial in~$n$ if $k$ is a small constant $k \in O(1)$, which places the problem in the parameterized complexity class~$\mathcal{XP}$.
Due to the fixed~$k$, only choosing $a \leq \binom{n}{k} \in O(n^k)$ admits valid alternatives and therefore does~$a$ not affect polynomiality.

\paragraph{Simultaneous search}

The simultaneous-search problem (cf.~Definition~\ref{def:afs:simultaneous-alternative}) enlarges the search space since it optimizes $a+1$ feature sets at once.
Thus, an exhaustive search over size-$k$ feature sets iterates over~$O((n^k)^{a+1}) = O(n^{k \cdot (a+1)})$ solution candidates.
Including the cost of constraint checking, we arrive at the following proposition:
\begin{proposition}[Complexity of exhaustive simultaneous search]
	Exhaustive simultaneous search (cf.~Definition~\ref{def:afs:alternative-feature-selection-simultaneous}) for $a \in \mathbb{N}_0$~alternative feature sets of size~$k \in \mathbb{N}$ from $n$~features has a time complexity of~$O((a+1)^2 \cdot n^{k \cdot (a+1) + 1})$ without the cost of evaluating the objective.
	\label{prop:afs:complexity-exhaustive-simultaneuos}
\end{proposition}
The scalability with~$n$ is worse than for exhaustive sequential search since the number of alternatives appears in the exponent now, except for a special case discussed in Appendix~\ref{sec:afs:appendix:complexity:exhaustive-simultaneous-special-case}.
Further, Proposition~\ref{prop:afs:complexity-exhaustive-simultaneuos} assumes that the constraints do not use linearization variables (cf.~Equations~\ref{eq:afs:product-linear} and~\ref{eq:afs:afs-simultaneous-complete}), which would enlarge the search space even further.
Finally, the complexity remains polynomial in~$n$ if $a$ and~$k$ are small and independent from~$n$, i.e., $a \cdot k \in O(1)$:
\begin{proposition}[Parameterized complexity of simultaneous-search problem]
	The simultaneous-search problem (cf.~Definition~\ref{def:afs:alternative-feature-selection-simultaneous}) for $a \in \mathbb{N}$~alternative feature sets of size~$k \in \mathbb{N}$ from $n$~features resides in the parameterized complexity class $\mathcal{XP}$ for the parameter~$a \cdot k$.
	\label{prop:afs:complexity-simultaneuos-xp}
\end{proposition}

\subsubsection{Univariate Feature Qualities}
\label{sec:afs:approach:complexity:univariate}

\paragraph{Motivation}

While the assumption $a \cdot k \in O(1)$ ensures polynomial runtime regarding~$n$ for arbitrary feature-selection methods, the optimization problem can still be hard without this assumption.
In the following, we derive complexity results for \emph{univariate feature qualities} (cf.~Equation~\ref{eq:afs:univariate-filter} and Appendix~\ref{sec:afs:appendix:univariate-complete-optimization-problem}).
This feature-selection method arguably has the simplest objective function, where the quality of a feature set is equal to the sum of the individual qualities of its constituent features.
This simplicity eases the transformation from and to well-known $\mathcal{NP}$-hard problems.
Appendix~\ref{sec:afs:appendix:complexity:related-work} discusses related work on these problems in detail.

In the following complexity analyses, we assume that the feature qualities~$q(X_{\cdot{}j},y)$ are given.
In particular, one can pre-compute these qualities before searching alternatives and treat them as constants in the optimization problem.
The complexity of this computation depends on the particular feature-quality measure and the number of data objects~$m$.
However, the number of features~$n$ should only affect the complexity linearly due to the univariate setting.

\paragraph{Min-aggregation with complete partitioning}

We start with three assumptions, which we will drop later:
First, we use a dissimilarity threshold of~$\tau = 1$, i.e., zero overlap of feature sets.
Second, all features must be part of one set.
Third, we analyze the simultaneous-search problem (cf.~Definition~\ref{def:afs:alternative-feature-selection-simultaneous}) with min-aggregation (cf.~Equation~\ref{eq:afs:afs-simultaneous-min-objective}).
We call the combination of the first two assumptions, which implies $n = (a+1) \cdot k$ if all sets have size~$k$, a \emph{complete partitioning}.
This scenario differs from $a \cdot k \in O(1)$, which yielded polynomial runtime regarding~$n$ in Proposition~\ref{prop:afs:complexity-simultaneuos-xp}.

A key factor for the hardness of partitioning is the number of solutions:
There are $\stirling{n}{a}$~ways to partition a set of $n$~elements into $a$~non-empty subsets, a Stirling number of the second kind~\cite{graham1994concrete}, which roughly scale like $a^n / a!$~\cite{moser1958stirling}, i.e., exponential in~$n$ for a fixed~$a$.
Even if the subset sizes are fixed, the scalability regarding~$n$ remains bad since it bases on a multinomial coefficient.

Our complete-partitioning scenario is a variant of the \textsc{Multi-Way Number Partitioning} problem:
Partition a multiset of $n$~integers into a fixed number of $a$~subsets such that the sums of all subsets are as equal as possible~\cite{korf2010objective}.
One problem formulation, called \textsc{Multiprocessor Scheduling} in~\cite{garey2003computers}, minimizes the maximum subset sum:
The goal is to assign tasks with different lengths to a fixed number of processors such that the maximum processor runtime is minimal.
Multiplying task lengths with~$-1$, one can turn the minimax problem of \textsc{Multiprocessor Scheduling} into the maximin formulation of the simultaneous-search problem with min-aggregation:
The tasks become features, the negative task lengths become univariate feature qualities, and the processors become feature sets. 
Since \textsc{Multiprocessor Scheduling} is $\mathcal{NP}$-complete, even for just two partitions~\cite{garey2003computers}, our problem is $\mathcal{NP}$-complete as well:
\begin{proposition}[Complexity of simultaneous-search problem with min-aggregation, complete partitioning, and unconstrained feature-set size]
	Assuming univariate feature qualities (cf.~Equation~\ref{eq:afs:univariate-filter}), a dissimilarity threshold~$\tau = 1$, unconstrained feature-set sizes, and all $n$~features have to be selected, the simultaneous-search problem (cf.~Definition~\ref{def:afs:alternative-feature-selection-simultaneous}) for alternative feature sets with min-aggregation (cf.~Equation~\ref{eq:afs:afs-simultaneous-min-objective}) is $\mathcal{NP}$-complete.
	\label{prop:afs:complexity-partitioning-min-unconstrained-k}
\end{proposition}
Since the assumptions in Proposition~\ref{prop:afs:complexity-partitioning-min-unconstrained-k} denote a special case of alternative feature selection, we directly obtain the following, more general proposition:
\begin{proposition}[Complexity of simultaneous-search problem with min-aggregation]
	The simultaneous-search problem (cf.~Definition~\ref{def:afs:alternative-feature-selection-simultaneous}) for alternative feature sets with min-aggregation (cf.~Equation~\ref{eq:afs:afs-simultaneous-min-objective}) is $\mathcal{NP}$-hard.
	\label{prop:afs:complexity-simultaneous-np}
\end{proposition}
While Proposition~\ref{prop:afs:complexity-partitioning-min-unconstrained-k} allowed arbitrary sets sizes, there are also existing partitioning problems for constrained~$k$, e.g., called \textsc{Balanced Number Partitioning} or \textsc{K-Partitioning}.
\textsc{K-Partitioning} with a minimax objective is $\mathcal{NP}$-hard~\cite{babel1998thek} and
can be transformed into our maximin objective as above:
\begin{proposition}[Complexity of simultaneous-search problem with min-aggregation, complete partitioning, and constrained feature-set size]
	Assuming univariate feature qualities (cf.~Equation~\ref{eq:afs:univariate-filter}), a dissimilarity threshold~$\tau = 1$, desired feature-set size~$k \in \mathbb{N}$, and all $n$~features have to be selected, the simulta\-neous-search problem (cf.~Definition~\ref{def:afs:alternative-feature-selection-simultaneous}) for alternative feature sets with min-aggregation (cf.~Equation~\ref{eq:afs:afs-simultaneous-min-objective}) is $\mathcal{NP}$-complete.
	\label{prop:afs:complexity-partitioning-min-constrained-k}
\end{proposition}

\paragraph{Min-aggregation with incomplete partitioning}

We now allow that some features may not be part of any feature set while we keep the assumption of zero feature-set overlap.
The problem of finding such an \emph{incomplete partitioning} still is $\mathcal{NP}$-complete in general (cf.~Appendix~\ref{sec:afs:appendix:complexity:proofs} for the proof):
\begin{proposition}[Complexity of simultaneous-search problem with min-aggregation, incomplete partitioning, and constrained feature-set size]
	Assuming univariate feature qualities (cf.~Equation~\ref{eq:afs:univariate-filter}), a dissimilarity threshold~$\tau = 1$, desired feature-set size~$k \in \mathbb{N}$, and \emph{not} all $n$~features have to be selected, the simultaneous-search problem (cf.~Definition~\ref{def:afs:alternative-feature-selection-simultaneous}) for alternative feature sets with min-aggregation (cf.~Equation~\ref{eq:afs:afs-simultaneous-min-objective}) is $\mathcal{NP}$-complete.
	\label{prop:afs:complexity-incomplete-partitioning-min-constrained-k}
\end{proposition}

\paragraph{Min-aggregation with overlapping feature sets}

The problem with $\tau < 1$, i.e., set overlap, also is $\mathcal{NP}$-hard in general (cf.~Appendix~\ref{sec:afs:appendix:complexity:proofs} for the proof):
\begin{proposition}[Complexity of simultaneous-search problem with min-aggregation, $\tau < 1$, and constrained feature-set size]
	Assuming univariate feature qualities (cf.~Equation~\ref{eq:afs:univariate-filter}), a dissimilarity threshold~$\tau < 1$, and desired feature-set size~$k \in \mathbb{N}$, the simultaneous-search problem (cf.~Definition~\ref{def:afs:alternative-feature-selection-simultaneous}) for alternative feature sets with min-aggregation (cf.~Equation~\ref{eq:afs:afs-simultaneous-min-objective}) is $\mathcal{NP}$-hard.
	\label{prop:afs:complexity-no-partitioning-min-constrained-k}
\end{proposition}

\paragraph{Sum-aggregation}

In contrast to the previous $\mathcal{NP}$-hardness results for min-aggregation, sum-aggregation (cf.~Equation~\ref{eq:afs:afs-simultaneous-sum-objective}) with $\tau=1$ admits polynomial-time algorithms (cf.~Appendix~\ref{sec:afs:appendix:complexity:proofs} for the proof):
\begin{proposition}[Complexity of search problems with sum-aggregation and $\tau=1$]
	Assuming univariate feature qualities (cf.~Equation~\ref{eq:afs:univariate-filter}) and a dissimilarity threshold~$\tau = 1$, the problems for (1) sequential search (cf.~Definition~\ref{def:afs:alternative-feature-selection-sequential}) and (2) simultaneous search with sum-aggregation (cf.~Definition~\ref{def:afs:alternative-feature-selection-simultaneous} and Equation~\ref{eq:afs:afs-simultaneous-sum-objective}) have a time complexity of $O(n \cdot \log n)$.
	\label{prop:afs:complexity-partitioning-sum}
\end{proposition}
This feasibility result applies to an arbitrary number of alternatives~$a$ and arbitrary feature-set sizes.
The key reason for polynomial runtime is that sum-aggregation does not require balancing the feature sets' qualities.
Thus, $\tau=1$ allows many solutions with the same summed objective value.
While at least one of these solutions also optimizes the objective with min-aggregation, most do not.
Hence, it is not a contradiction that optimizing with min-aggregation is considerably harder.

\subsubsection{Summary}
\label{sec:afs:approach:complexity:summary}

We showed that the simultaneous-search problem for alternative feature sets is $\mathcal{NP}$-hard in general (cf.~Proposition~\ref{prop:afs:complexity-simultaneous-np}).
We also placed it in the parameterized complexity class $\mathcal{XP}$ (cf.~Proposition~\ref{prop:afs:complexity-simultaneuos-xp}), having~$a$ and~$k$ as the parameters that drive the hardness of the problem.
For univariate feature qualities and min-aggregation, we obtained more specific $\mathcal{NP}$-hardness results for (1) complete partitioning, i.e., $\tau = 1$ and $(a+1) \cdot k = n$ (cf.~Proposition~\ref{prop:afs:complexity-partitioning-min-constrained-k}), (2) incomplete partitioning, i.e., $(a+1) \cdot k < n$ (cf.~Proposition~\ref{prop:afs:complexity-incomplete-partitioning-min-constrained-k}) and (3) feature set overlap, i.e., $\tau < 1$ (cf.~Proposition~\ref{prop:afs:complexity-no-partitioning-min-constrained-k}).
In contrast, we also inferred polynomial runtime for univariate feature qualities, sum-aggregation, and $\tau = 1$ (cf.~Proposition~\ref{prop:afs:complexity-partitioning-sum}).

\subsection{Heuristic Search for Univariate Feature Qualities}
\label{sec:afs:approach:univariate-heuristics}

In this section, we propose heuristic search methods for univariate feature qualities (cf.~Equation~\ref{eq:afs:univariate-filter} and Section~\ref{sec:afs:appendix:univariate-complete-optimization-problem}) and the Dice dissimilarity (cf.~Equation~\ref{eq:afs:dice}) as~$d(\cdot)$.
These heuristics complement the solver-based search that we discussed in Section~\ref{sec:afs:approach:objectives:white-box}.
The proposed heuristics may be faster than exact optimization at the expense of lower feature-set quality.
In particular, we describe \emph{Greedy Replacement} (cf.~Section~\ref{sec:afs:approach:univariate-heuristics:greedy-replacement}), which is a sequential search method, and \emph{Greedy Balancing} (cf.~Section~\ref{sec:afs:approach:univariate-heuristics:greedy-balancing}), which is a simultaneous search method.
Additionally, Appendix~\ref{sec:afs:appendix:greedy-depth} introduces \emph{Greedy Depth Search}.
All three heuristics leverage that the univariate objective sums up the individual qualities~$q_j$ of selected features and does not consider interactions between features.

\subsubsection{Greedy Replacement}
\label{sec:afs:approach:univariate-heuristics:greedy-replacement}

\emph{Greedy Replacement} is our first heuristic for alternative feature selection with univariate feature qualities.
This heuristic conducts a sequential search.

\paragraph{Algorithm}

\begin{algorithm}[t]
	\DontPrintSemicolon
	\KwIn{Univariate feature qualities~$q \in \mathbb{R}^n$, \newline
		Feature-set size~$k \in \mathbb{N}$, \newline
		Number of alternatives~$a \in \mathbb{N}_0$, \newline
		Dissimilarity threshold~$\tau \in (0, 1]$}
	\KwOut{List of feature-selection decision vectors~$s^{(\cdot)}$}
	\BlankLine
	$\mathit{indices} \leftarrow$ sort\_indices($q$, order=descending) \tcp*{Order by qualities} \label{al:afs:greedy-replacement:line:sorting}
	$s \leftarrow \{0\}^n$ \tcp*{Initial selection for all alternatives} \label{al:afs:greedy-replacement:line:common-features-start}
	$\mathit{feature\_position} \leftarrow 1$ \tcp*{Index of index of current feature}
	\While{$\mathit{feature\_position} \leq \lfloor (1 - \tau) \cdot k \rfloor$}{
		$j \leftarrow \mathit{indices}[\mathit{feature\_position}]$ \tcp*{Index feature by quality}
		$s_j \leftarrow 1$ \;
		$\mathit{feature\_position} \leftarrow \mathit{feature\_position} + 1$\;
	} \label{al:afs:greedy-replacement:line:common-features-end}
	$l \leftarrow 0$\ \tcp*{Number of current alternative} \label{al:afs:greedy-replacement:line:disjoint-features-start}
	\While{$l \leq a$ \textbf{and} $l \leq \frac{n - k}{\lceil \tau \cdot k \rceil}$}{ \label{al:afs:greedy-replacement:line:stop}
		$s^{(l)} \leftarrow s$ \tcp*{Select top $\lfloor (1 - \tau) \cdot k \rfloor$ features} \label{al:afs:greedy-replacement:line:copy-common-selection}
		\For(\tcp*[f]{Select remaining $\lceil \tau \cdot k \rceil$ features}){$\_ \leftarrow 1$ \KwTo $\lceil \tau \cdot k \rceil$}{ \label{al:afs:greedy-replacement:line:one-disjoint-feature-start}
			$j \leftarrow \mathit{indices}[\mathit{feature\_position}]$\;
			$s^{(l)}_j \leftarrow 1$\;
			$\mathit{feature\_position} \leftarrow \mathit{feature\_position} + 1$\; \label{al:afs:greedy-replacement:line:one-disjoint-feature-end}
		}
		$l \leftarrow l + 1$\;
	} \label{al:afs:greedy-replacement:line:disjoint-features-end}
	\Return{$s^{(0)}, \dots, s^{(l)}$}
	\caption{\emph{Greedy Replacement} for alternative feature selection.}
	\label{al:afs:greedy-replacement}
\end{algorithm}

Algorithm~\ref{al:afs:greedy-replacement} outlines \emph{Greedy Replacement}.
It retains a fixed subset of features in each alternative while iteratively replacing the remaining ones.
We start by sorting the features decreasingly based on their qualities~$q_j$ (Line~\ref{al:afs:greedy-replacement:line:sorting}).
For a fixed feature-set size~$k$, a dissimilarity threshold~$\tau$, and using the Dice dissimilarity (cf.~Equation~\ref{eq:afs:dice}), one subset with $\lfloor (1 - \tau) \cdot k \rfloor$~features can be contained in all alternatives without violating the dissimilarity threshold (cf.~Equation~\ref{eq:afs:dice-rearranged-equal-size}).
Thus, our algorithm indeed selects the $\lfloor (1 - \tau) \cdot k \rfloor$~features with highest quality in each alternative~$s^{(\cdot)}$ (Lines~\ref{al:afs:greedy-replacement:line:common-features-start}--\ref{al:afs:greedy-replacement:line:common-features-end}).
We fill the remaining spots in the sets by iterating over the alternatives and remaining features (Lines~\ref{al:afs:greedy-replacement:line:disjoint-features-start}--\ref{al:afs:greedy-replacement:line:disjoint-features-end}).
For each alternative, we select the $\lceil \tau \cdot k \rceil$~highest-quality features not used in any prior alternative, thereby satisfying the dissimilarity threshold.
We continue this procedure until we reach the desired number of alternatives~$a$ or until there are not enough unused features to form further alternatives (Line~\ref{al:afs:greedy-replacement:line:stop}).
\begin{example}[Algorithm of \emph{Greedy Replacement}]
	With $n=10$ features, feature-set size~$k=5$, and $\tau=0.4$, each feature set must differ by $\lceil \tau \cdot k \rceil = 2$ features from the other feature sets.
	The original feature set~$s^{(0)}$ consists of the top $k=5$ features regarding quality~$q_j$.
	The first alternative~$s^{(1)}$ consists of the top $\lfloor (1 - \tau) \cdot k \rfloor = 3$ features plus the sixth- and seventh-best feature.
	The second alternative~$s^{(2)}$ consists of the top three features plus the eighth- and ninth-best one.
	The algorithm cannot continue beyond $l=2$ since there are not enough unused features to form further alternatives in the same manner.
	\label{ex:afs:greedy-replacement:algorithm}
\end{example}
In general, the $l$-th alternative consists of the top $\lfloor (1 - \tau) \cdot k \rfloor$ features plus the features $k + (l-1) \cdot \lceil \tau \cdot k \rceil + 1$ to $k + l \cdot \lceil \tau \cdot k \rceil$ in descending quality order.

\paragraph{Time complexity}

Sorting $n$~feature qualities (Line~\ref{al:afs:greedy-replacement:line:sorting}) has a time complexity of $O(n \cdot \log n)$.
Next, the algorithm iterates over the features and processes each feature at most once.
In particular, after selecting a feature in an alternative, $\mathit{feature\_position}$ increases by~1.
The maximum value of this variable depends on~$a$ and~$k$ (Line~\ref{al:afs:greedy-replacement:line:stop}) but cannot exceed the total number of features~$n$.
For each $\mathit{feature\_position}$, the algorithm conducts a constant number of update operations (Lines~\ref{al:afs:greedy-replacement:line:one-disjoint-feature-start}--\ref{al:afs:greedy-replacement:line:one-disjoint-feature-end}).
Assuming each array access takes $O(1)$, the total update cost is $O(n)$.
Further, each alternative $s^{(l)}$ gets initialized as the selection~$s$ of the top $\lfloor (1 - \tau) \cdot k \rfloor$ features (Line~\ref{al:afs:greedy-replacement:line:copy-common-selection}), which the algorithm determines once before the main loop (Lines~\ref{al:afs:greedy-replacement:line:common-features-start}--\ref{al:afs:greedy-replacement:line:common-features-end}).
Initializing $a$~arrays costs $O(a \cdot n)$.
Since the algorithm can only yield $a < n$ alternatives, the overall time complexity is~$O(n^2)$, i.e., polynomial in~$n$.

\paragraph{Quality}

While not optimizing exactly, \emph{Greedy Replacement} still offers an approximation guarantee relative to exact search methods:
\begin{proposition}[Approximation quality of \emph{Greedy Replacement}]
	Assume non-negative univariate feature qualities of $n$~features (cf.~Equation~\ref{eq:afs:univariate-filter}), $a \in \mathbb{N}_0$~alternatives, the Dice dissimilarity (cf.~Equation~\ref{eq:afs:dice}) as~$d(\cdot)$, a dissimilarity threshold~$\tau \in (0,1]$, desired feature-set size~$k \in \mathbb{N}$, and $k + a \cdot \lceil \tau \cdot k \rceil \leq n$.
	Under these conditions, \emph{Greedy Replacement} reaches at least a fraction of $\frac{\lfloor (1 - \tau) \cdot k \rfloor}{k}$ of the optimal objective values of the optimization problems for (1) sequential search (cf.~Definition~\ref{def:afs:alternative-feature-selection-sequential}), (2) simultaneous search with sum-aggregation (cf.~Definition~\ref{def:afs:alternative-feature-selection-simultaneous} and Equation~\ref{eq:afs:afs-simultaneous-sum-objective}), and (3) simultaneous search with min-aggregation (cf.~Definition~\ref{def:afs:alternative-feature-selection-simultaneous} and Equation~\ref{eq:afs:afs-simultaneous-min-objective}).
	\label{prop:afs:approximation-greedy-replacement}
\end{proposition}

\begin{proof}
	For univariate feature qualities, the quality of a feature set is the sum of the qualities of the contained features.
	\emph{Greedy Replacement} includes the $\lfloor (1 - \tau) \cdot k \rfloor$ highest-quality features in each alternative of size~$k$, while the remaining $\lceil \tau \cdot k \rceil$ features may have an arbitrary quality.
	In comparison, the optimal original, i.e., unconstrained, feature set of size~$k$ contains the top $k$ features, which are the union of the top $\lfloor (1 - \tau) \cdot k \rfloor$ features and the next-best $\lceil \tau \cdot k \rceil$ features.
	Due to quality sorting, each of the next-best $\lceil \tau \cdot k \rceil$ features has at most the quality of each of the top $\lfloor (1 - \tau) \cdot k \rfloor$ features, i.e., contributes the same or less to the summed quality of the feature set.
	Hence, assuming non-negative qualities, each alternative yielded by \emph{Greedy Replacement} has at least a quality of $\lfloor (1 - \tau) \cdot k \rfloor / k$ relative to the optimal original feature set of size~$k$ since the $\lfloor (1 - \tau) \cdot k \rfloor$ highest-quality features are part of both feature sets.
	Next, the optimal original feature set of size~$k$ upper-bounds the quality of any feature set of size~$k$.
	Consequently, alternative feature sets found by exact sequential or simultaneous search can also not be better.
	Thus, the quality bound of the heuristic solution relative to the exact solution also applies to the minimum and sum of qualities over multiple alternative feature sets.
\end{proof}
In particular, \emph{Greedy Replacement} yields a constant-factor approximation for the three optimization problems mentioned in Proposition~\ref{prop:afs:approximation-greedy-replacement}.
The condition $k + a \cdot \lceil \tau \cdot k \rceil \leq n$ describes scenarios where \emph{Greedy Replacement} can yield all desired alternatives, i.e., does not run out of unused features.
As the heuristic has polynomial runtime, alternative feature selection lies in the complexity class $\mathcal{APX}$~\cite{khanna1998syntactic} under the specified conditions:
\begin{proposition}[Approximation complexity of alternative feature selection]
	Assume non-negative univariate feature qualities of $n$~features (cf.~Equation~\ref{eq:afs:univariate-filter}), $a \in \mathbb{N}_0$~alternatives, the Dice dissimilarity (cf.~Equation~\ref{eq:afs:dice}) as~$d(\cdot)$, a dissimilarity threshold~$\tau \in [0,1]$, desired feature-set size~$k \in \mathbb{N}$, and $k + a \cdot \lceil \tau \cdot k \rceil \leq n$.
	Under these conditions, the optimization problems for (1) sequential search (cf.~Definition~\ref{def:afs:alternative-feature-selection-sequential}), (2) simultaneous search with sum-aggregation (cf.~Definition~\ref{def:afs:alternative-feature-selection-simultaneous} and Equation~\ref{eq:afs:afs-simultaneous-sum-objective}), and (3) simultaneous search with min-aggregation (cf.~Definition~\ref{def:afs:alternative-feature-selection-simultaneous} and Equation~\ref{eq:afs:afs-simultaneous-min-objective}) reside in the complexity class~$\mathcal{APX}$.
	\label{prop:afs:approximation-apx}
\end{proposition}
For~$\tau = 1$, \emph{Greedy Replacement} even yields the same objective values as exact sequential search and exact simultaneous search with sum-aggregation since it becomes identical to a procedure we outlined in our complexity analysis earlier (cf.~Proposition~\ref{prop:afs:complexity-partitioning-sum}).
In contrast, for arbitrary~$\tau$, the following example shows that the heuristic can be worse than exact sequential search for as few as $a=2$ alternatives:
\begin{example}[Quality of \emph{Greedy Replacement} vs. exact search]
	Consider $n=6$~features with univariate feature qualities $q = (9,8,7,3,2,1)$, feature-set size $k=2$, number of alternatives~$a=2$, the Dice dissimilarity (cf.~Equation~\ref{eq:afs:dice}) as~$d(\cdot)$, and dissimilarity threshold~$\tau = 0.5$, which permits an overlap of one feature between sets here.
	Exact sequential search and exact simultaneous search, for min- and sum-aggregation, yield the selection $s^{(0)} = (1,1,0,0,0,0)$, $s^{(1)} = (1,0,1,0,0,0)$, and $s^{(2)} = (0,1,1,0,0,0)$, with a summed quality of $\,17+16+15=48$.
	\emph{Greedy Replacement} yields the selection $s^{(0)} = (1,1,0,0,0,0)$, $s^{(1)} = (1,0,1,0,0,0)$, and $s^{(2)} = (1,0,0,1,0,0)$, with a summed quality of $\,17+16+12=45$.
	\label{ex:afs:greedy-replacement:worse-than-exact}
\end{example}
While the first two feature sets are identical between exact and heuristic search, the quality of $s^{(2)}$ is lower for the heuristic (12 vs.~15).
In particular, by always selecting all the top $\lfloor (1 - \tau) \cdot k \rfloor$~features, the heuristic misses out on feature sets only involving some or none of them.

For min-aggregation in simultaneous search, $a=1$ alternative already suffices for the heuristic being potentially worse than exact search:
\begin{example}[Quality of \emph{Greedy Replacement} vs. min-aggregation]
	Consider $n=6$~features with univariate feature qualities $q = (9,8,7,3,2,1)$, feature-set size~$k=3$, number of alternatives~$a=1$, the Dice dissimilarity (cf.~Equation~\ref{eq:afs:dice}) as~$d(\cdot)$, and dissimilarity threshold~$\tau = 0.5$, which permits an overlap of one feature between sets here.
	Exact simultaneous search with min-aggregation yields the selection $s^{(0)} = (1,1,0,0,1,0)$ and $s^{(1)} = (1,0,1,1,0,0)$, with a quality of $\,\min \{19,19\} = 19$.
	\emph{Greedy Replacement} and exact sequential search yield the selection $s^{(0)} = (1,1,1,0,0,0)$ and $s^{(1)} = (1,0,0,1,1,0)$, with a quality of $\,\min \{24,14\} = 14$.
	Exact simultaneous search with sum-aggregation may yield either of these two solutions or the selection $s^{(0)} = (1,1,0,1,0,0)$ and $s^{(1)} = (1,0,1,0,1,0)$, all with the same sum-aggregated quality of~38 but different min-aggregated quality.
	\label{ex:afs:greedy-replacement:worse-than-min-agg}
\end{example}
In particular, \emph{Greedy Replacement} does not balance feature-set qualities since it is a sequential search method.
We alleviate this issue with the heuristic~\emph{Greedy Balancing} (cf.~Section~\ref{sec:afs:approach:univariate-heuristics:greedy-balancing}).

\paragraph{Limitations}

Proposition~\ref{prop:afs:approximation-greedy-replacement} and Examples~\ref{ex:afs:greedy-replacement:worse-than-exact},~\ref{ex:afs:greedy-replacement:worse-than-min-agg} already showed the potential quality loss of \emph{Greedy Replacement} compared to an exact search for alternatives.
Further, the heuristic only works as long as some features have not been part of any feature set yet, i.e., $k + a \cdot \lceil \tau \cdot k \rceil \leq n$.
Once the heuristic runs out of unused features, one would need to switch the search method.
Thus, to obtain a high number of alternatives~$a$ with the heuristic, the following conditions are beneficial:
The number of features~$n$ should be high, the feature-set size~$k$ should be low, and the dissimilarity threshold~$\tau$ should be low.
These conditions align well with typical feature-selection scenarios where~$k \ll n$.

Another drawback is that \emph{Greedy Replacement} assumes a very simple notion of feature-set quality.
If the latter becomes more complex than a sum of univariate qualities, quality-based feature ordering (Line~\ref{al:afs:greedy-replacement:line:sorting}) may be impossible or suboptimal.
Further, \emph{Greedy Replacement} cannot accommodate additional constraints on feature sets, e.g., based on domain knowledge.
Finally, the heuristic assumes the same size~$k$ for all feature sets.

\subsubsection{Greedy Balancing}
\label{sec:afs:approach:univariate-heuristics:greedy-balancing}

\emph{Greedy Balancing} adapts \emph{Greedy Replacement} to obtain more balanced feature-set qualities by employing a simultaneous search method.
In particular, it tries to distribute the individual feature qualities evenly to alternatives.

\begin{algorithm}[tp]
	\DontPrintSemicolon
	\KwIn{Univariate feature qualities~$q \in \mathbb{R}^n$, \newline
		Feature-set size~$k \in \mathbb{N}$, \newline
		Number of alternatives~$a \in \mathbb{N}_0$, \newline
		Dissimilarity threshold~$\tau \in [0, 1]$}
	\KwOut{List of feature-selection decision vectors~$s^{(0)}, \dots, s^{(a)}$}
	\BlankLine
	\lIf{$\lceil \tau \cdot k \rceil \cdot a + k > n$}{ \label{al:afs:greedy-balancing:line:stop-early}
		\Return{$\emptyset$}
	}
	$\mathit{indices} \leftarrow$ sort\_indices($q$, order=descending) \tcp*{Order by qualities} \label{al:afs:greedy-balancing:line:sorting} \label{al:afs:greedy-balancing:line:common-features-start} 
	\For(\tcp*[f]{Initial selection for all alternatives}){$l \leftarrow 0$ \KwTo $a$}{ \label{al:afs:greedy-balancing:line:array-initialization-start}
		$s^{(l)} \leftarrow \{0\}^n$ \; \label{al:afs:greedy-balancing:line:array-initialization-end}
	}
	$\mathit{feature\_position} \leftarrow 1$ \tcp*{Index of index of current feature}
	\While(\tcp*[f]{Select top features}){$\mathit{feature\_position} \leq \lfloor (1 - \tau) \cdot k \rfloor$}{ \label{al:afs:greedy-balancing:line:common-features-loop}
		$j \leftarrow \mathit{indices}[\mathit{feature\_position}]$ \tcp*{Index feature by quality}
		\For(\tcp*[f]{Same features in all alternatives}){$l \leftarrow 0$ \KwTo $a$}{ \label{al:afs:greedy-balancing:line:common-features-alternatives-start}
			$s^{(l)}_j \leftarrow 1$ \;
		}
		$\mathit{feature\_position} \leftarrow \mathit{feature\_position} + 1$\;
	} \label{al:afs:greedy-balancing:line:common-features-end}
	\For{$l \leftarrow 0$ \KwTo $a$}{ \label{al:afs:greedy-balancing:line:disjoint-features-start}
		$Q^{(l)} \leftarrow 0$\ \tcp*{Relative quality of each alternative}
	}
	\While(\tcp*[f]{Fill all positions}){$\mathit{feature\_position} \leq \lceil \tau \cdot k \rceil \cdot a + k$}{ \label{al:afs:greedy-balancing:line:stop} \label{al:afs:greedy-balancing:line:disjoint-features-loop}
		$Q_\text{min} \leftarrow \infty$ \tcp*{Find alternative with lowest quality}
		$l_\text{min} \leftarrow -1$ \;
		\For{$l \leftarrow 0$ \KwTo $a$}{ \label{al:afs:greedy-balancing:line:disjoint-features-alternatives-start}
			\If(\tcp*[f]{Check cardinality}){$Q^{(l)} < Q_\text{min}$ \textbf{and} $\sum_{j=1}^{n} s^{(l)}_j < k$}{ \label{al:afs:greedy-balancing:line:cardinality-check}
				$Q_\text{min} \leftarrow Q^{(l)}$ \;
				$l_\text{min} \leftarrow l$ \;
			}
		}
		$j \leftarrow \mathit{indices}[\mathit{feature\_position}]$ \tcp*{Index feature by quality}
		$s^{(l_\text{min})}_j \leftarrow 1$ \tcp*{Add to lowest-quality, non-full alternative}
		$Q^{(l_\text{min})} \leftarrow Q^{(l_\text{min})} + q_j$ \tcp*{Update quality of that alternative}
		$\mathit{feature\_position} \leftarrow \mathit{feature\_position} + 1$\;
	} \label{al:afs:greedy-balancing:line:disjoint-features-end}
	\Return{$s^{(0)}, \dots, s^{(a)}$}
	\caption{\emph{Greedy Balancing} for alternative feature selection.}
	\label{al:afs:greedy-balancing}
\end{algorithm}

\paragraph{Algorithm}

Algorithm~\ref{al:afs:greedy-balancing} outlines \emph{Greedy Balancing}.
First, we check whether the algorithm should terminate early, i.e., whether the number of features~$n$ is not high enough to satisfy the desired user parameters~$k$, $a$, and~$\tau$ (Line~\ref{al:afs:greedy-balancing:line:stop-early}).
Next, we select the first $\lfloor (1 - \tau) \cdot k \rfloor$ features in each alternative like in \emph{Greedy Replacement} (cf.~Algorithm~\ref{al:afs:greedy-replacement}), i.e., we pick the features with the highest quality~$q_j$ (Lines~\ref{al:afs:greedy-balancing:line:common-features-start}--\ref{al:afs:greedy-balancing:line:common-features-end}).

For the remaining spots in the alternatives, we use a Longest Processing Time (LPT) heuristic (Lines~\ref{al:afs:greedy-balancing:line:disjoint-features-start}--\ref{al:afs:greedy-balancing:line:disjoint-features-end}).
Such heuristics are common for \textsc{Multiprocessor Scheduling} and \textsc{Balanced Number Partitioning} problems~\cite{babel1998thek, chen20023partitioning, lawrinenko2018reduction} (cf.~Section~\ref{sec:afs:appendix:complexity:related-work}).
In particular, we continue iterating over features by decreasing quality.
We assign each feature to the alternative that currently has the lowest summed quality~$Q^{(l)}$ and whose size~$k$ has not been reached yet.
We continue this procedure until all alternatives have reached size~$k$ (Line~\ref{al:afs:greedy-balancing:line:stop}).
\begin{example}[Algorithm of \emph{Greedy Balancing}]
	Consider $n=6$~features with univariate feature qualities $q = (9,8,7,3,2,1)$, feature-set size~$k=4$, number of alternatives~$a=1$, and dissimilarity threshold~$\tau = 0.5$.
	The features with qualities~$9$ and $8$ become part of both feature sets, $s^{(0)}$ and $s^{(1)}$, since $\lfloor (1 - \tau) \cdot k \rfloor = 2$ (Lines~\ref{al:afs:greedy-balancing:line:common-features-start}--\ref{al:afs:greedy-balancing:line:common-features-end}).
	At this point, both alternatives have the same relative quality $Q^{(0)} = Q^{(1)} = 0$, which ignores the quality of the shared features.
	Now the LPT heuristic becomes active (Lines~\ref{al:afs:greedy-balancing:line:disjoint-features-start}--\ref{al:afs:greedy-balancing:line:disjoint-features-end}).
	The feature with quality~$7$ is added to $s^{(0)}$, which causes $Q^{(0)} > Q^{(1)}$ (i.e., $7 > 0$).
	Thus, the feature with quality~3 is added to $s^{(1)}$.
	As $Q^{(0)} > Q^{(1)}$ (i.e., $7 > 3$) still holds, the feature with quality~2 becomes part of $s^{(1)}$ as well.
	Because $s^{(1)}$ has reached size~$k = 4$, the feature with quality~1 is added to $s^{(0)}$, even if the latter still has a higher relative quality (i.e., $7 > 5$).
	Now both alternatives have reached their desired size and $n = 6 = \lceil 0.5 \cdot 4 \rceil \cdot 1 + 4 = \lceil \tau \cdot k \rceil \cdot a + k$ (Line~\ref{al:afs:greedy-balancing:line:stop}).
	Thus, the algorithm terminates.
	The solution consists of $s^{(0)} = (1,1,1,0,0,1)$ and $s^{(1)} = (1,1,0,1,1,0)$.
	\label{ex:afs:greedy-balancing:algorithm}
\end{example}

\paragraph{Time complexity}

Like \emph{Greedy Replacement}, \emph{Greedy Balancing} has an upfront cost of $O(n \cdot \log n)$ for sorting feature qualities (Line~\ref{al:afs:greedy-balancing:line:sorting}) and then iterates over $O(n)$ $\mathit{feature\_position}$s (Lines~\ref{al:afs:greedy-balancing:line:common-features-loop} and~\ref{al:afs:greedy-balancing:line:disjoint-features-loop}).
For each $\mathit{feature\_position}$, the algorithm iterates over $a$~alternatives (Lines~\ref{al:afs:greedy-balancing:line:common-features-alternatives-start} and~\ref{al:afs:greedy-balancing:line:disjoint-features-alternatives-start}) and conducts a constant number of operations each, which yields total update costs of $O(a \cdot n)$.
This figure assumes cardinality checks (Line~\ref{al:afs:greedy-balancing:line:cardinality-check}) can be done in $O(1)$, e.g., by storing the current feature-set sizes.
There is also a total cost of~$O(a \cdot n)$ for array initialization (Lines~\ref{al:afs:greedy-balancing:line:array-initialization-start}--\ref{al:afs:greedy-balancing:line:array-initialization-end}).
Since $a < n$ (Line~\ref{al:afs:greedy-balancing:line:stop-early}), the overall time complexity of \emph{Greedy Balancing} is $O(n^2)$, as for \emph{Greedy Replacement}.

\paragraph{Quality}

\emph{Greedy Balancing} selects the same features as \emph{Greedy Replacement} and only changes their assignment to the feature sets.
Thus, the summed feature-set quality remains the same, while the minimum feature-set quality may be higher due to balancing.
Hence, the quality guarantee of \emph{Greedy Replacement} (cf.~Proposition~\ref{prop:afs:approximation-greedy-replacement}) holds here as well:
\begin{proposition}[Approximation quality of \emph{Greedy Balancing}]
	Assume non-negative univariate feature qualities of $n$~features (cf.~Equation~\ref{eq:afs:univariate-filter}), $a \in \mathbb{N}_0$~alternatives, the Dice dissimilarity (cf.~Equation~\ref{eq:afs:dice}) as~$d(\cdot)$, a dissimilarity threshold~$\tau \in [0,1]$, desired feature-set size~$k \in \mathbb{N}$, and $k + a \cdot \lceil \tau \cdot k \rceil \leq n$.
	Under these conditions, \emph{Greedy Balancing} reaches at least a fraction of $\frac{\lfloor (1 - \tau) \cdot k \rfloor}{k}$ of the optimal objective values of the optimization problems for (1) sequential search (cf.~Definition~\ref{def:afs:alternative-feature-selection-sequential}), (2) simultaneous search with sum-aggregation (cf.~Definition~\ref{def:afs:alternative-feature-selection-simultaneous} and Equation~\ref{eq:afs:afs-simultaneous-sum-objective}), and (3) simultaneous search with min-aggregation (cf.~Definition~\ref{def:afs:alternative-feature-selection-simultaneous} and Equation~\ref{eq:afs:afs-simultaneous-min-objective}).
	\label{prop:afs:approximation-greedy-balancing}
\end{proposition}
For min-aggregation in the objective, \emph{Greedy Balancing} can even be better than exact sequential search, as Example~\ref{ex:afs:greedy-replacement:worse-than-min-agg} shows, where the heuristic search would yield the same solution as exact simultaneous search with min-aggregation.
However, the heuristic can also be worse than exact sequential search and exact simultaneous search, as Example~\ref{ex:afs:greedy-replacement:worse-than-exact} shows, where \emph{Greedy Balancing} would yield the same solution as \emph{Greedy Replacement}.

\paragraph{Limitations}

\emph{Greedy Balancing} shares several limitations with \emph{Greedy Replacement}, e.g., it may be worse than exact search, assumes univariate feature qualities, and does not work if the number of features~$n$ is too low relative to~$k$, $a$, and $\tau$.
In the latter case, \emph{Greedy Balancing} yields no solution due to its simultaneous nature, while \emph{Greedy Replacement} yields at least some alternatives.
However, if running out of features is not an issue, \emph{Greedy Balancing} has the advantage of more balanced feature-set qualities.
Also, one could easily adapt \emph{Greedy Balancing} to yield the largest feasible number of alternatives in case~$a$ alternatives are infeasible.

\section{Related Work}
\label{sec:afs:related-work}

In this section, we review related work from the fields of feature selection (cf.~Section~\ref{sec:afs:related-work:feature-selection}), subgroup discovery (cf.~Section~\ref{sec:afs:related-work:subgroup-discovery}), clustering (cf.~Section~\ref{sec:afs:related-work:clustering}), subspace clustering and subspace search (cf.~Section~\ref{sec:afs:related-work:subspace}), explainable artificial intelligence (cf.~Section~\ref{sec:afs:related-work:xai}),
and Rashomon sets (cf.~Section~\ref{sec:afs-related-work:rashomon-sets}).
To the best of our knowledge, searching for optimal alternative feature sets in the sense of this paper is novel.
However, there is literature on optimal alternatives outside the field of feature selection.
Also, there are works on finding multiple, diverse feature sets.

\subsection{Feature Selection}
\label{sec:afs:related-work:feature-selection}

\paragraph{Conventional feature selection}

Most feature-selection methods only yield one solution~\cite{borboudakis2021extending}, though some exceptions exist.
Nevertheless, none of the following approaches searches for optimal alternatives in our sense.

\cite{siddiqi2020genetic}~proposes a genetic algorithm that iteratively updates a population of multiple feature sets.
To foster diversity, the algorithm's fitness criterion does not only consider feature-set quality but also a penalty on feature-set overlap in the population.
However, users cannot control the admissible overlap, i.e., there is no parameter comparable to~$\tau$.
In contrast, the genetic algorithm's parameter for the population size corresponds to the number of alternatives.

\cite{emmanouilidis1999selecting}~employs multi-objective genetic algorithms to obtain prediction models with different complexity and diverse feature sets.
However, the two objectives are prediction performance and feature-set size, while diversity only influences the genetic selection step under particular circumstances.

\cite{mueller2021feature}~clusters features and forms alternatives by picking one feature from each cluster.
However, they do this to reduce the number of features for subsequent model selection and model evaluation, not as a guided search for alternatives.

\paragraph{Ensemble feature selection}

Ensemble feature selection~\cite{saeys2008robust, seijo2017ensemble} combines feature-selection results, e.g., obtained by different feature-selection methods or on different samples of the data.
Fostering diverse feature sets might be a sub-goal to improve prediction performance, but it is usually only an intermediate step.
This focus differs from our goal of finding optimal alternatives.

\cite{woznica2012model}~obtains feature sets or rankings on bootstrap samples of the data.
Next, an aggregation strategy creates one or multiple diverse feature sets.
The authors propose using k-medoid clustering and frequent itemset mining for the latter.
While these approaches allow to control the number of feature sets, there is no parameter for their dissimilarity.
Also, aggregation builds on bootstrap sampling instead of being allowed to form arbitrary alternatives.

\cite{liu2019subspace}~builds an ensemble prediction model from classifiers trained on different feature sets.
To this end, a genetic algorithm iteratively evolves a population of feature sets.
Diversity is one of multiple fitness criteria, with the Hamming distance quantifying the dissimilarity of feature sets.
However, since feature diversity is only one of several objectives, users cannot control it directly.

\cite{shekar2017diverse} uses different univariate feature-quality measures to initialize a population of multiple feature sets.
Next, they iteratively form new feature sets by choosing those features from two existing feature sets that are only in either set.
This procedure does not guarantee how diverse the final feature sets are.

\cite{guru2018alternative}~computes feature relevance separately for each class and then combines the top features.
This procedure can yield alternatives but does not enforce dissimilarity.
Also, the number of alternatives is fixed to the number of classes.

\paragraph{Statistically equivalent feature sets}

Approaches for statistically equivalent feature sets~\cite{borboudakis2021extending, lagani2017feature} use statistical tests to determine features or feature sets that are equivalent for predictions.
E.g., a feature may be independent of the target given another feature.
A search algorithm conducts multiple such tests and outputs equivalent feature sets or a corresponding feature grouping.

Our notion of alternatives differs from equivalent feature sets in several aspects.
In particular, building optimal alternatives from equivalent feature sets is not straightforward.
Depending on how the statistical tests are configured, there can be an arbitrary number of equivalent feature sets without explicit quality-based ordering.
Instead, we always provide a fixed number of alternatives.
Also, our alternatives need not have equivalent quality but should be optimal under constraints.
Further, our dissimilarity threshold allows controlling overlap between feature sets instead of eliminating all redundancies.

\paragraph{Constrained feature selection}

We define alternatives via constraints on feature sets.
There already is work on other kinds of constraints in feature selection, e.g., for feature cost~\cite{paclik2002feature}, feature groups~\cite{yuan2006model}, or domain knowledge~\cite{bach2022empirical, groves2015toward}.
These approaches are orthogonal to our work, as such constraints do not explicitly foster optimal alternatives.
At most, they might implicitly lead to alternative solutions~\cite{bach2022empirical}.
Further, most of the approaches are tied to particular constraint types, while our integer-programming formulation supports such constraints besides the ones for alternatives.
\cite{bach2022empirical} is an exception in that regard since it models feature selection as a Satisfiability Modulo Theories (\textsc{SMT}) optimization problem~\cite{barrett2018satisfiability, nieuwenhuis2006sat}, which also admits our constraints for alternatives.

\subsection{Subgroup Discovery}
\label{sec:afs:related-work:subgroup-discovery}

\cite{leeuwen2012diverse}~presents six strategies to foster diversity in subgroup set discovery, which searches for interesting regions in the data space, i.e., combinations of conditions on feature values, rather than only selecting features.
Three strategies yield a fixed number of alternatives and the other three a variable number.
The strategies become part of beam search, i.e., a heuristic search procedure, while we mainly consider exact optimization.
Also, the criteria for alternatives differ from ours.
The strategy \emph{fixed-size description-based selection} prunes subgroups with the same quality as previously found ones if they differ by at most one feature-value condition.
In contrast, we require dissimilarity independent from the quality, have a flexible dissimilarity threshold, and support simultaneous besides sequential search for alternatives.
Another strategy, \emph{variable-size description-based selection}, limits the total number of subgroups a feature may occur in but does not constrain subgroup overlap per se.
The four remaining strategies in~\cite{leeuwen2012diverse} have no obvious counterpart in our feature-selection scenario.

\subsection{Clustering}
\label{sec:afs:related-work:clustering}

Finding alternative solutions has been addressed extensively in the field of clustering.
\cite{bailey2014alternative} gives a taxonomy and describes algorithms for alternative clustering.
Our problem definition in Sections~\ref{sec:afs:approach:problem} and~\ref{sec:afs:approach:constraints} is, on a high level, inspired by the one in~\cite{bailey2014alternative}:
Find multiple solutions that maximize quality while minimizing similarity.
\cite{bailey2014alternative} also distinguishes between singular/multiple alternatives and sequential/simultaneous search.
They mention constraint-based search for alternatives as one of several solution paradigms.
Further, feature selection can help to find alternative clusterings~\cite{tao2012novel}.
Nevertheless, the problem definition for alternatives in clustering and feature selection is fundamentally different.
First, the notion of dissimilarity differs, as we want to find differently composed feature sets while alternative clustering targets at different assignments of data objects to clusters.
Second, our objective function, i.e., feature-set quality, relates to a supervised prediction scenario while clustering is unsupervised.

Two exemplary approaches for alternative clustering are \emph{COALA}~\cite{bae2006coala} and \emph{MAXIMUS}~\cite{bae2010clustering}.
COALA~\cite{bae2006coala} imposes \emph{cannot-link constraints} on pairs of data objects rather than constraining features:
Data objects from the same cluster in the original clustering should be assigned to different clusters in the alternative clustering.
In each step of its iterative clustering procedure, COALA compares the quality of an action observing the constraints to another one violating them.
Based on a threshold on the quality ratio, either action is taken.
MAXIMUS~\cite{bae2010clustering} employs an integer program to formulate dissimilarity between clusterings.
In particular, it wants to maximize the dissimilarity of the feature-value distributions in clusters between the clusterings.
The output of the integer program leads to constraints for a subsequent clustering procedure.

\subsection{Subspace Clustering and Subspace Search}
\label{sec:afs:related-work:subspace}

Finding multiple useful feature sets plays a role in subspace clustering~\cite{gunnemann2009detection, hu2018subspace, mueller2009relevant} and subspace search~\cite{fouche2021efficient, nguyen20134s, trittenbach2019dimension}.
These approaches strive to improve the results of data-mining algorithms by using subspaces, i.e., feature sets, rather than the full space, i.e., all features.
While some subspace approaches only consider individual subspaces, others explicitly try to remove redundancy between subspaces~\cite{gunnemann2009detection, hu2018subspace, nguyen20134s} or foster subspace diversity~\cite{fouche2021efficient, trittenbach2019dimension}.
In particular, \cite{hu2018subspace} surveys subspace-clustering approaches yielding multiple results and discusses the redundancy aspect.
However, subspace clustering and -search approaches differ from alternative feature selection in at least one of the following aspects:

First, the objective differs, i.e., definitions of subspace quality deviate from feature-set quality in our scenario.
Second, definitions of subspace redundancy may consider dissimilarity between projections of the entire data, i.e., data objects with feature values, into subspaces, while our notion of dissimilarity purely bases on binary feature-selection decisions.
Third, controlling dissimilarity in subspace approaches is often less user-friendly than with our parameter~$\tau$.
E.g., dissimilarity might be a regularization term in the objective rather than a hard constraint, or there might not be an explicit control parameter at all.

\subsection{Explainable Artificial Intelligence (XAI)}
\label{sec:afs:related-work:xai}

In the field of XAI, alternative explanations might provide additional insights into predictions, enable users to develop and test different hypotheses, appeal to different kinds of users, and foster trust in the predictions~\cite{kim2021multi, wang2019designing}.
In contrast, obtaining significantly different explanations for the same prediction might raise doubts about how meaningful the explanations are~\cite{jain2019attention}.
Finding diverse explanations has been studied for various explainers, e.g., for counterfactuals~\cite{dandl2020multi, karimi2020model, mohammadi2021scaling, mothilal2020explaining, russell2019efficient, wachter2017counterfactual}, criticisms~\cite{kim2016examples}, and semifactual explanations~\cite{artelt2022even}.
There are several approaches to foster diversity, e.g., ensembling different kinds of explanations~\cite{silva2019produce}, considering multiple local minima~\cite{wachter2017counterfactual}, using a search algorithm that maintains diversity~\cite{dandl2020multi}, extending the optimization objective~\cite{artelt2022even, kim2016examples, mothilal2020explaining}, or introducing constraints~\cite{karimi2020model, mohammadi2021scaling, russell2019efficient}.
The last option is similar to the way we enforce alternatives.
Of the various mentioned approaches, only~\cite{artelt2022even, mohammadi2021scaling, mothilal2020explaining} introduce a parameter to control the diversity of solutions.
Of these three works, only~\cite{mohammadi2021scaling} offers a user-friendly dissimilarity threshold in~$[0,1]$, while the other two approaches employ a regularization parameter in the objective.

Despite similarities, all the previously mentioned XAI techniques tackle different problems than alternative feature selection.
In particular, they provide local explanations, i.e., target at prediction outcomes for individual data objects and build on feature values.
In contrast, we are interested in the global prediction quality of feature sets.
For example, counterfactual explanations~\cite{guidotti2022counterfactual, stepin2021survey, verma2020counterfactual} alter feature \emph{values} \emph{as little as possible} to produce an alternative prediction \emph{outcome}.
In contrast, alternative feature sets might alter the feature \emph{selection} \emph{significantly} while trying to maintain the original prediction \emph{quality}.

\subsection{Rashomon Sets}
\label{sec:afs-related-work:rashomon-sets}

A Rashomon set is a set of prediction models that reach a certain, e.g., close-to-optimal, prediction performance~\cite{fisher2019all}.
Despite similar performance, these models may still assign different feature importance scores, leading to different explanations~\cite{laberge2023partial}.
Thus, Rashomon sets may yield partial information about alternative feature sets.
However, approaches for Rashomon sets do not explicitly search for alternative feature sets as a whole, i.e., feature sets satisfying a dissimilarity threshold relative to other sets.
Instead, these approaches focus on the range of each feature's importance over prediction models.
Further, our notion of alternatives is not bound to model-based feature importance but encompasses a broader range of feature-selection methods.
Finally, we use importance scores from one instead of multiple models to find importance-based alternatives.

\section{Experimental Design}
\label{sec:afs:experimental-design}

In this section, we introduce our experimental design.
After a brief overview of its goal and components (cf.~Section~\ref{sec:afs:experimental-design:overview}), we describe its components in detail:
evaluation metrics (cf.~Section~\ref{sec:afs:experimental-design:evaluation}), methods (cf.~Section~\ref{sec:afs:experimental-design:approaches}), and datasets (cf.~Section~\ref{sec:afs:experimental-design:datasets}).
Finally, we briefly outline our implementation (cf.~Section~\ref{sec:afs:experimental-design:implementation}).

\subsection{Overview}
\label{sec:afs:experimental-design:overview}

We conduct experiments with 30 binary-classification datasets.
As evaluation metrics, we consider feature-set quality and runtime.
We compare five feature-selection methods, representing different notions of feature-set quality.
Also, we train prediction models with the resulting feature sets and analyze prediction performance.
To find alternatives, we consider simultaneous and sequential search, both with solver-based and heuristic search methods.
We systematically vary the two user parameters for searching alternatives, i.e., the number of alternatives~$a$ and the dissimilarity threshold~$\tau$.

\subsection{Evaluation Metrics}
\label{sec:afs:experimental-design:evaluation}

\paragraph{Feature-set quality}

We evaluate feature-set quality with two metrics.
First, we report the \emph{objective value}~$Q(s,X,y)$ of the feature-selection methods, which guided the search for alternatives.
Second, we train prediction models with the found feature sets.
We report \emph{prediction performance} in terms of the Matthews correlation coefficient (MCC)~\cite{matthews1975comparison}.
This coefficient is insensitive to class imbalance, reaches its maximum of~1 for perfect predictions, and is~0 for random guessing as well as constant predictions.

We conduct a stratified five-fold cross-validation.
In particular, the search for alternatives and model training only use the training data, while we employ the test data for evaluation:
For the test-set objective value, we compute the objective on the test set but with the feature selection from the training set.
For the test-set prediction performance, we predict on the test set but use a prediction model trained with these features on the training set.

\paragraph{Runtime}

We consider two metrics related to runtime.

First, we analyze the \emph{optimization time}.
For white-box feature-selection methods in solver-based search for alternatives, we sum the measured runtime of solver calls.
We exclude the time for computing feature qualities and feature dependencies for the objective since one can compute these values once per dataset and then reuse them in each solver call.
For \emph{Greedy Wrapper} feature selection and the heuristic search methods for alternatives, we measure the runtime of the corresponding search algorithms.
For \emph{Greedy Wrapper}, this search procedure involves multiple solver calls and trainings of a prediction model.

Second, we examine the \emph{optimization status}, which can take four values for the solver-based search.
If the solver finished before the timeout, it either found an \emph{optimal} solution or proved the problem \emph{infeasible}, i.e., no solution exists.
If the solver reached its timeout, it either found a \emph{feasible} solution without proving its optimality or found no valid solution, though one might exist, so the problem is \emph{not solved}.
For the heuristic search methods, we only use \emph{not solved} and \emph{feasible} as statuses, as these search methods are neither guaranteed to find the optimum nor do they prove infeasibility if they terminate early.

\subsection{Methods}
\label{sec:afs:experimental-design:approaches}

We employ multiple methods for making predictions (cf.~Section~\ref{sec:afs:experimental-design:approaches:prediction}), feature selection (cf.~Section~\ref{sec:afs:experimental-design:approaches:feature-selection}), and searching alternatives (cf.~Section~\ref{sec:afs:experimental-design:approaches:alternatives}).

\subsubsection{Prediction}
\label{sec:afs:experimental-design:approaches:prediction}

As prediction models, we use decision trees~\cite{breiman1984classification}, which admit learning complex, non-linear dependencies from the data.
Preliminary experiments with random forests~\cite{breiman2001random} and k-nearest neighbors yielded similar insights.
We leave the trees' hyperparameters at their defaults, except for using information gain instead of Gini impurity as the split criterion, to be consistent with our filter feature-selection methods.
Note that tree models also select features themselves, so they may not use all features from the alternative feature sets.
However, this is not an issue for our study.
We are interested in which performance the models achieve if limited to certain feature sets, not how they use each available feature.

\subsubsection{Feature Selection (Objective Functions)}
\label{sec:afs:experimental-design:approaches:feature-selection}

We search for alternatives under different notions of feature-set quality in the objective function.
We choose five well-known feature-selection methods that are easy to parameterize and cover the different categories from Section~\ref{sec:afs:fundamentals:quality} except \emph{embedded}, as explained in Section~\ref{sec:afs:approach:objectives:embedding}.
However, we use feature importance from an embedded method as post-hoc importance scores.
Four feature-selection methods allow a white-box formulation of the optimization problem, while \emph{Greedy Wrapper} is black-box.
With each feature-selection method, we enforce fixed feature-set sizes of $k \in \{5,10\}$.

\paragraph{Filter feature selection}

We evaluate three filter methods, all using mutual information~\cite{kraskov2004estimating} as the dependency measure~$q(\cdot)$.
This measure can capture arbitrary dependencies rather than, e.g., just linear ones.
\emph{MI} denotes a univariate filter (cf.~Equation~\ref{eq:afs:univariate-filter}), while \emph{FCBF} (cf.~Equation~\ref{eq:afs:fcbf}) and \emph{mRMR} (cf.~Equation~\ref{eq:afs:mrmr-linear}) are multivariate.
We normalize the mutual-information values per dataset and cross-validation fold to improve comparability:
For \emph{FCBF} and \emph{MI}, we scale the individual features' qualities with a constant such that the overall objective value is in $[0, 1]$.
For \emph{mRMR}, we min-max-normalize all mutual-information values to $[0,1]$, so the overall objective is in $[-1,1]$.

\paragraph{Wrapper feature selection}

As a wrapper method, we employ our hill-climbing procedure \emph{Greedy Wrapper} (cf.~Algorithm~\ref{al:afs:greedy-wrapper}) with $\mathit{max\_iters} = 1000$.
To evaluate feature-set quality in the wrapper, we apply a stratified 80:20 holdout split and train decision trees.
$Q(s,X,y)$ corresponds to the prediction performance in terms of MCC on the 20\% validation part.

\paragraph{Post-hoc feature importance}

As a post-hoc importance measure called \emph{Model Gain}, we use importance scores from \emph{scikit-learn's} decision trees.
There, importance expresses a feature's contribution towards optimizing the split criterion of the tree, for which we choose information gain.
These importances are normalized to sum up to~1 by default.
We plug them into Equation~\ref{eq:afs:univariate-filter}, i.e., treat them like univariate filter scores, though they actually originate from trees trained with all features and thus are not univariate.

\subsubsection{Alternatives (Constraints)}
\label{sec:afs:experimental-design:approaches:alternatives}

\paragraph{Overview}

In our evaluation, we categorize search methods for alternatives in two dimensions that are orthogonal to each other:
Solver-based vs. heuristic and sequential vs. simultaneous.
Also, we analyze the impact of the user parameters~$a$ and~$\tau$.

\paragraph{Solver-based search methods}

For the four feature-selection methods with white-box objectives, we use the integer-programming solver \emph{SCIP}~\cite{bestuzheva2021scip} to solve the underlying optimization problems exactly.
Given sufficient solving time, these alternatives are globally optimal.
For \emph{Greedy Wrapper}, the search procedure (Algorithm~\ref{al:afs:greedy-wrapper}) is heuristic (though still solver-based, so we place it in this category) and might not cover the entire search space.
There, the solver only assists in finding valid solutions but does not optimize quality.

For each feature selection method, we analyze \emph{sequential} (cf.~Definition~\ref{def:afs:alternative-feature-selection-sequential}) and \emph{simultaneous} (cf.~Definition~\ref{def:afs:alternative-feature-selection-simultaneous}) solver-based search for alternatives.
For the latter, we employ sum-aggregation (cf.~Equation~\ref{eq:afs:afs-simultaneous-sum-objective}) and min-aggregation (cf.~Equation~\ref{eq:afs:afs-simultaneous-min-objective}) in the objective.
In figures and tables, we use the abbreviations \emph{seq.}, \emph{sim. (sum)}, and \emph{sim. (min)}.

\paragraph{Heuristic search methods}

We also evaluate heuristic search methods, which do not use a solver (cf.~Section~\ref{sec:afs:approach:univariate-heuristics}).
In particular, we employ \emph{Greedy Replacement} (cf.~Algorithm~\ref{al:afs:greedy-replacement}), which is a sequential search method, and \emph{Greedy Balancing} (cf.~Algorithm~\ref{al:afs:greedy-balancing}), which is a simultaneous search method.
In figures and tables, we use the abbreviations \emph{rep.} and \emph{bal.}.
Since these heuristics assume univariate feature qualities (cf.~Equation~\ref{eq:afs:univariate-filter}), we only combine them with the univariate feature-selection methods \emph{MI} and \emph{Model Gain}.

\paragraph{Search parametrization}

We vary the parameters of the search systematically:
We evaluate $a \in \{1, \dots, 10\}$~alternatives for sequential search methods and $a \in \{1, \dots, 5\}$ for simultaneous search methods due to the higher runtime of the latter.
For the dissimilarity threshold~$\tau$, we analyze all possible sizes of the feature-set overlap in the Dice dissimilarity (cf.~Equations~\ref{eq:afs:dice} and~\ref{eq:afs:dice-rearranged-equal-size}).
Thus, for $k=5$, we consider $\tau \in \{0.2, 0.4, 0.6, 0.8, 1.0\}$, corresponding to an overlap of four to zero features.
For $k=10$, we consider $\tau \in \{0.1, 0.2, \dots, 1.0\}$.
We exclude $\tau = 0$, which would allow returning duplicate feature sets.

\paragraph{Timeout}

In solver-based search, we employ a timeout to enable a large-scale evaluation and account for the high variance of solver runtime.
In particular, we grant each solver call 60~s multiplied by the number of feature sets sought.
Thus, solver-based sequential search conducts multiple solver calls with 60~s timeout each, while solver-based simultaneous search conducts one solver call with proportionally more time, e.g., 300~s for five feature sets.
For 84\% of the feature sets in our evaluation, the solver finished before the timeout.

\paragraph{Competitors for search methods}

As discussed in Section~\ref{sec:afs:related-work}, related work pursues different objective functions, operates with different notions of alternatives, and may only target particular feature-selection methods.
All these points prevent a meaningful comparison to our search methods.
E.g., a feature set deemed alternative in related work may violate our constraints for alternatives.
Further, we can still put the feature-set quality into perspective by comparing alternatives to each other.
In particular, the original feature set, i.e., without constraints for alternatives, serves as a natural reference point.

\begin{table}[p]
	\centering
	\caption{
		Datasets from PMLB used in our experiments.
		$m$~denotes the number of data objects and $n$~the number of features.
		In dataset names, we replaced \emph{GAMETES\_Epistasis} with \emph{GE\_} and \emph{GAMETES\_Heterogeneity} with \emph{GH\_} to reduce the table's width.
	}
	\begin{tabular}{lrr}
		\toprule
		Dataset & $m$ & $n$ \\
		\midrule
		backache & 180 & 32 \\
		chess & 3196 & 36 \\
		churn & 5000 & 20 \\
		clean1 & 476 & 168 \\
		clean2 & 6598 & 168 \\
		coil2000 & 9822 & 85 \\
		credit\_a & 690 & 15 \\
		credit\_g & 1000 & 20 \\
		dis & 3772 & 29 \\
		GE\_2\_Way\_20atts\_0.1H\_EDM\_1\_1 & 1600 & 20 \\
		GE\_2\_Way\_20atts\_0.4H\_EDM\_1\_1 & 1600 & 20 \\
		GE\_3\_Way\_20atts\_0.2H\_EDM\_1\_1 & 1600 & 20 \\
		GH\_20atts\_1600\_Het\_0.4\_0.2\_50\_EDM\_2\_001 & 1600 & 20 \\
		GH\_20atts\_1600\_Het\_0.4\_0.2\_75\_EDM\_2\_001 & 1600 & 20 \\
		hepatitis & 155 & 19 \\
		Hill\_Valley\_with\_noise & 1212 & 100 \\
		horse\_colic & 368 & 22 \\
		house\_votes\_84 & 435 & 16 \\
		hypothyroid & 3163 & 25 \\
		ionosphere & 351 & 34 \\
		molecular\_biology\_promoters & 106 & 57 \\
		mushroom & 8124 & 22 \\
		ring & 7400 & 20 \\
		sonar & 208 & 60 \\
		spambase & 4601 & 57 \\
		spect & 267 & 22 \\
		spectf & 349 & 44 \\
		tokyo1 & 959 & 44 \\
		twonorm & 7400 & 20 \\
		wdbc & 569 & 30 \\
		\bottomrule
	\end{tabular}
	\label{tab:afs:datasets}
\end{table}

\subsection{Datasets}
\label{sec:afs:experimental-design:datasets}

We use datasets from the Penn Machine Learning Benchmarks (PMLB)~\cite{olson2017pmlb,romano2021pmlb}.
To harmonize evaluation, we only consider binary-classification datasets, though alternative feature selection also works for regression and multi-class problems.
We exclude datasets with less than 100 data objects since they may entail a high uncertainty when assessing feature-set quality.
Also, we exclude datasets with less than 15 features to leave room for alternatives.
Next, we exclude one dataset with 1000 features, which would dominate the overall runtime.
Finally, we manually exclude datasets that seem to be duplicated or modified versions of other datasets.
Consequently, we obtain 30 datasets with 106 to 9822 data objects and 15 to 168 features (cf.~Table~\ref{tab:afs:datasets}).
The datasets do not contain any missing values.
Categorical features have an ordinal encoding by default.

\subsection{Implementation and Execution}
\label{sec:afs:experimental-design:implementation}

We implemented our experimental pipeline in Python~3.8, using \emph{scikit-learn}~\cite{pedregosa2011scikit-learn} for machine learning and the integer-programming solver \emph{SCIP}~\cite{bestuzheva2021scip} via the package \emph{OR-Tools}~\cite{perron2022or-tools} for solver-based search.
The code is available on \emph{GitHub}\footnote{\url{https://github.com/Jakob-Bach/Alternative-Feature-Selection}} and additionally backed up in the \emph{Software Heritage archive}\footnote{\href{https://archive.softwareheritage.org/swh:1:dir:6b679eb1b901c281b7c7e7fdc9dbdaec2f627c7a;origin=https://github.com/Jakob-Bach/Alternative-Feature-Selection;visit=swh:1:snp:e1e4ba6be250fa9a7de03fa2103ac206abdb8d17;anchor=swh:1:rev:7d027c23820fce067cce8953fdfdac719722949b}{swh:1:dir:6b679eb1b901c281b7c7e7fdc9dbdaec2f627c7a}}.
A requirements file in our repository specifies the versions of all dependencies.
Further, we released the methods for alternative feature selection as the Python package \emph{alfese}\footnote{\url{https://pypi.org/project/alfese/}} on \emph{PyPI} to ease reuse.
Finally, we published all experimental data\footnote{\url{https://doi.org/10.35097/4ttgrpx92p30jwww}}.

Our experimental pipeline parallelizes over datasets, cross-validation folds, and feature-selection methods, while each of these experimental tasks runs single-threaded.
We ran the pipeline on a server with 160~GB RAM and an \emph{AMD EPYC 7551} CPU, having 32~physical cores and a base clock of 2.0~GHz.
With this hardware, the parallelized pipeline run took approximately 249~hours, i.e., about 10.4~days.

\section{Evaluation}
\label{sec:afs:evaluation}

In this section, we evaluate our experiments.
After comparing feature-selection methods without constraints (cf.~Section~\ref{sec:afs:evaluation:feature-selection}), we discuss the parametrization for searching alternatives: the search methods (cf.~Section~\ref{sec:afs:evaluation:search-methods}) and the two user parameters~$a$ and~$\tau$ (cf.~Section~\ref{sec:afs:evaluation:parameters}).
Section~\ref{sec:afs:evaluation:summary} summarizes key findings.

\subsection{Feature-Selection Methods}
\label{sec:afs:evaluation:feature-selection}

\begin{figure}[t]
	\centering
	\begin{subfigure}[t]{0.48\textwidth}
		\centering
		\includegraphics[width=\textwidth, trim=15 35 5 15, clip]{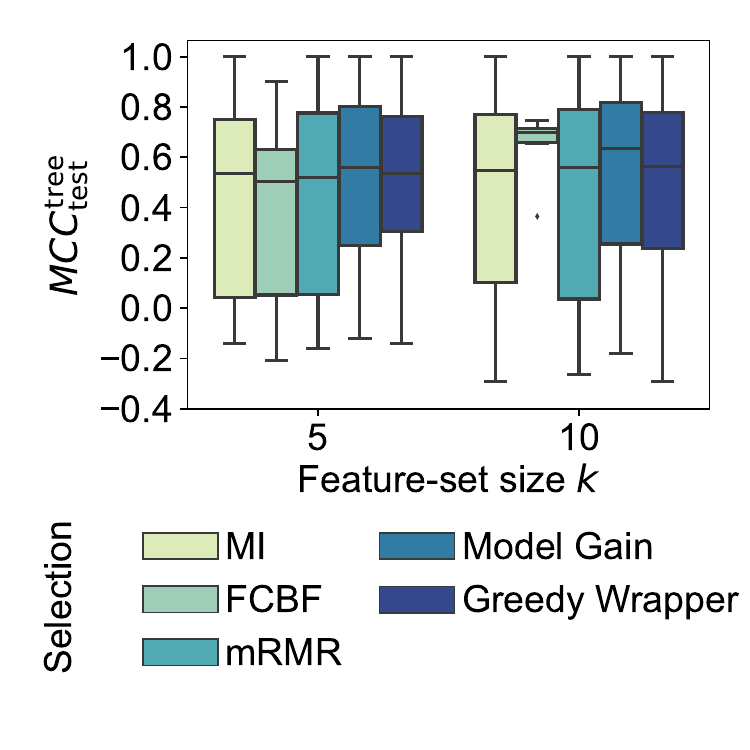}
		\caption{Test-set prediction performance by feature-set size~$k$.}
		\label{fig:afs:impact-fs-method-k-decision-tree-test-mcc}
	\end{subfigure}
	\hfill
	\begin{subfigure}[t]{0.48\textwidth}
		\centering
		\includegraphics[width=\textwidth, trim=15 35 5 15, clip]{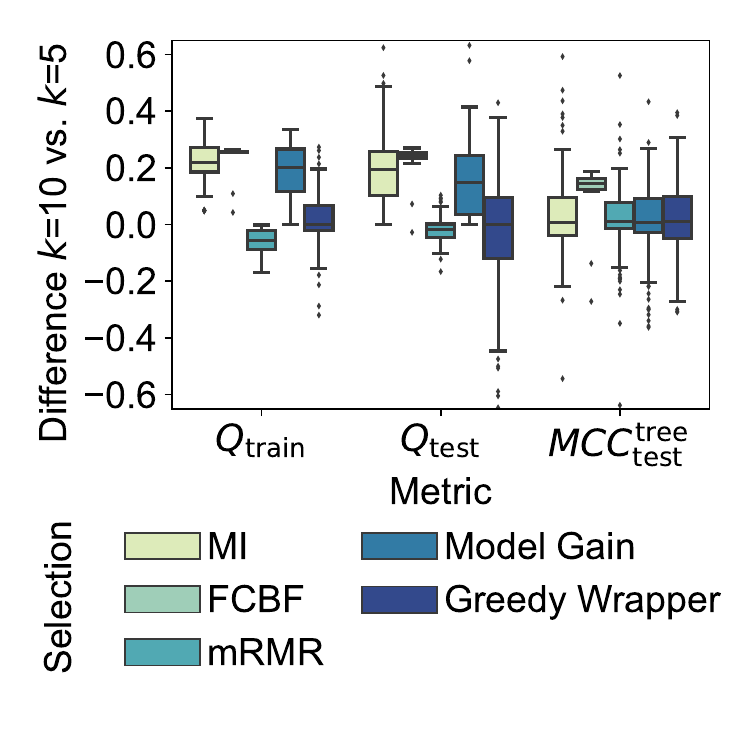}
		\caption{
			Difference in feature-set quality between $k=10$ and $k=5$ by evaluation metric.
			Y-axis is truncated to improve readability.
		}
		\label{fig:afs:impact-fs-method-k-metric-diff}
	\end{subfigure}
	\caption{
		Distribution of feature-set quality over datasets and cross-validation folds, by feature-selection method.
		Results from the original feature sets of solver-based sequential search.
	}
	\label{fig:afs:impact-fs-method-k-quality}
\end{figure}

\paragraph{Prediction performance}

The five feature-selection methods in our experiments employ different objective functions~$Q(s,X,y)$, so comparing objective values between them does not make sense.
However, we can analyze the prediction performance of the obtained feature sets.
Figure~\ref{fig:afs:impact-fs-method-k-decision-tree-test-mcc} displays the distribution of a decision tree's test-set MCC on the original feature sets, i.e., without constraints, for the feature-selection methods.
The mean test-set MCC is 0.53 for \emph{Model Gain} and \emph{Greedy Wrapper}, 0.47 for \emph{MI}, 0.46 for \emph{mRMR}, and 0.43 for \emph{FCBF}.
Thus, the analyses of alternative feature sets in Sections~\ref{sec:afs:evaluation:search-methods} and~\ref{sec:afs:evaluation:parameters} focus on \emph{Model Gain} while still discussing the remaining feature-selection methods.

The univariate, model-free method \emph{MI} keeps up surprisingly well with more sophisticated methods.
It uses the same objective function as \emph{Model Gain} but obtains its feature qualities from a bivariate dependency measure rather than a prediction model.

\emph{Greedy Wrapper} uses prediction performance to assess feature-set quality but employs a search heuristic instead of optimizing globally.
In particular, it performed 661 iterations on average to determine the original feature sets.
However, the number of possible feature sets is higher, e.g., already $\binom{15}{5} = 3003$ for the lowest-dimensional dataset in our evaluation (cf.~Table~\ref{tab:afs:datasets}) and $k=5$.
The still high prediction performance comes at the expense of high runtime (cf.~Table~\ref{tab:afs:impact-search-fs-method-optimization-time}), so we prefer \emph{Model Gain} for later evaluations.

\emph{FCBF}'s results may be taken with a grain of salt:
The original feature set in solver-based sequential search is already infeasible, i.e., no solution satisfied the constraints, in 71\% of the cases for \emph{FCBF} but never for the other feature-selection methods.
Over all solver-based search runs, even 89\% of the feature sets were infeasible for \emph{FCBF} but only 18\% for \emph{Model Gain}.
In particular, the combination of feature-correlation constraints in our formulation of \emph{FCBF} (cf.~Equation~\ref{eq:afs:fcbf}) with a fixed feature-set size~$k$ may make the problem infeasible, especially if~$k$ gets larger.

\paragraph{Influence of feature-set size~$k$}

One could expect larger feature sets to exhibit a higher feature-set quality than smaller ones, but the picture in our experiments is more nuanced.
In particular, quality may not increase proportionally with $k$ or may even decrease.
As Figure~\ref{fig:afs:impact-fs-method-k-metric-diff} shows for the original feature sets of solver-based sequential search, \emph{MI} and \emph{Model Gain} exhibit an increase of the training-set objective value~$Q_\text{train}$ from~$k=5$ to~$k=10$, i.e., the difference depicted in Figure~\ref{fig:afs:impact-fs-method-k-metric-diff} is positive.
As these objectives are monotonic and the feature qualities are non-negative, a decrease in the training-set objective value is impossible.
In contrast, \emph{Greedy Wrapper} with its black-box objective does not necessarily benefit from more features.
The latter insight also applies to \emph{mRMR}, which normalizes its objective with the number of selected features and penalizes feature redundancy.
For \emph{FCBF}, the fraction of feasible feature sets changes considerably from $k=5$ to $k=10$, so the overall quality between these two settings should not be compared.
As Figure~\ref{fig:afs:impact-fs-method-k-metric-diff} also displays, the benefit of larger feature sets is even less clear for prediction performance.
In particular, all feature-selection methods except \emph{FCBF} show a median difference in test-set MCC close to zero when comparing $k=5$ to $k=10$.

\begin{figure}[p]
	\centering
	\begin{subfigure}[t]{\textwidth}
		\centering
		\includegraphics[width=\textwidth, trim=15 25 35 15, clip]{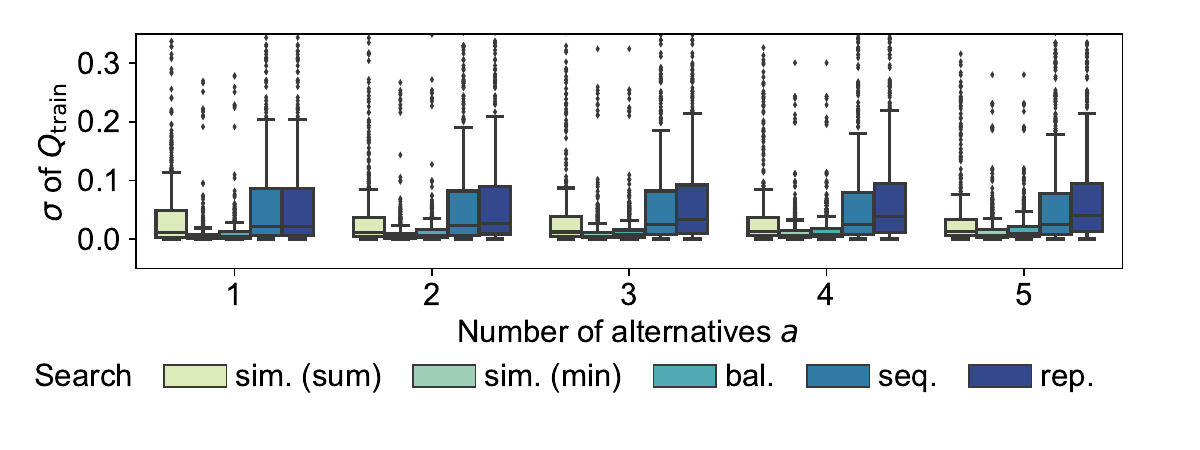}
		\caption{Training-set objective value.}
		\label{fig:afs:impact-search-stddev-train-objective}
	\end{subfigure}
	\\ \vspace{\baselineskip}
	\begin{subfigure}[t]{\textwidth}
		\centering
		\includegraphics[width=\textwidth, trim=15 25 35 15, clip]{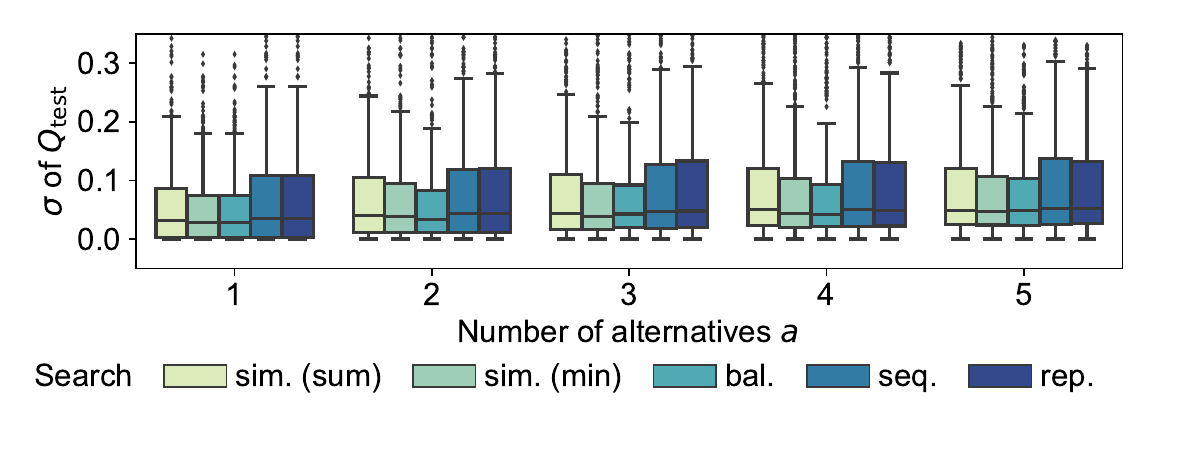}
		\caption{Test-set objective value.}
		\label{fig:afs:impact-search-stddev-test-objective}
	\end{subfigure}
	\\ \vspace{\baselineskip}
	\begin{subfigure}[t]{\textwidth}
		\centering
		\includegraphics[width=\textwidth, trim=15 25 35 15, clip]{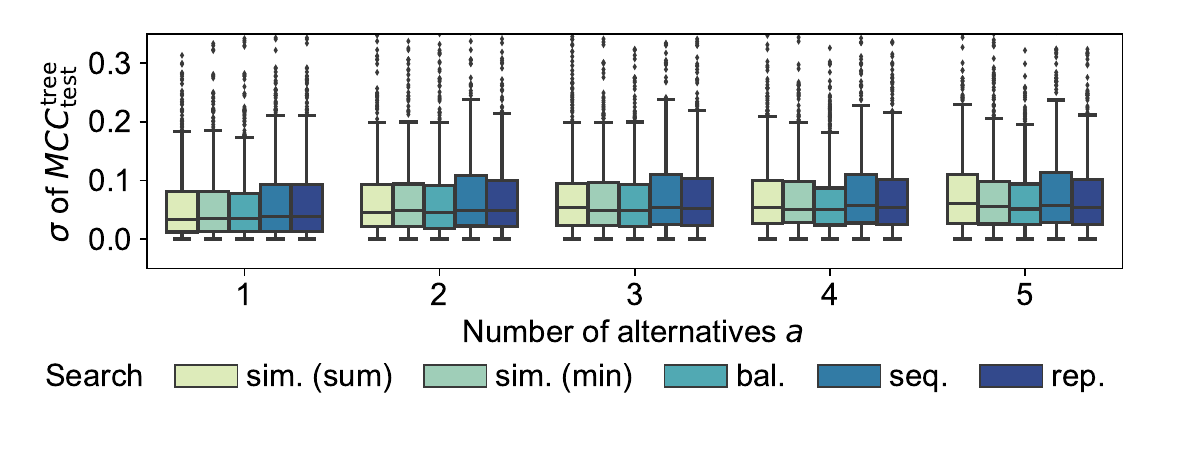}
		\caption{Test-set prediction performance.}
		\label{fig:afs:impact-search-stddev-decision-tree-test-mcc}
	\end{subfigure}
	\caption{
		Distribution of standard deviation of feature-set quality within search runs over datasets, cross-validation-folds, and~$\tau$, by search method for alternatives and number of alternatives~$a$.
		Results with \emph{Model Gain} as the feature-selection method and $k=5$.
		Excludes search settings where at least one combination of search method and~$a$ yielded no valid solution.
		Y-axes are truncated to improve readability.
	}
	\label{fig:afs:impact-search-stddev-quality}
\end{figure}

\subsection{Search Methods for Alternatives}
\label{sec:afs:evaluation:search-methods}

\paragraph{Variance in feature-set quality}

As expected, the search method influences how much the training-set objective value~$Q$ varies between multiple alternatives obtained for the same experimental settings, i.e., within one search run for alternatives.
Figure~\ref{fig:afs:impact-search-stddev-train-objective} visualizes this result for \emph{Model Gain} as the feature-selection method and $k=5$.
The figure shows how the standard deviation of the training-set objective value within individual search runs for alternatives is distributed over other experimental settings, e.g., datasets and cross-validation folds.
In particular, the quality of multiple alternatives varies more if they are found by solver-based sequential search rather than solver-based simultaneous search.
For solver-based simultaneous search, min-aggregation yields considerably more homogeneous feature-set quality than sum-aggregation.
These findings apply to all white-box feature-selection methods but not \emph{Greedy Wrapper}.

The heuristic search method \emph{Greedy Balancing} yields a small variance of training-set objective value within search runs, only slightly higher than for solver-based simultaneous search with min-aggregation.
In contrast, \emph{Greedy Replacement} rather mimics solver-based sequential search, having a substantial variance of quality.
Additionally, the variance of \emph{Greedy Replacement} noticeably grows with the number of alternatives~$a$.

As Figures~\ref{fig:afs:impact-search-stddev-test-objective} and~\ref{fig:afs:impact-search-stddev-decision-tree-test-mcc} show, the variance of feature-set quality differs considerably less between the search methods on the test set, for the objective value as well as prediction performance.
This effect might result from overfitting:
Even if the training-set quality is similar, some alternatives might generalize better than others.
Thus, this variance caused by overfitting could alleviate the effect caused by the choice of search method.

\begin{figure}[p]
	\centering
	\begin{subfigure}[t]{\textwidth}
		\centering
		\includegraphics[width=\textwidth, trim=15 25 35 15, clip]{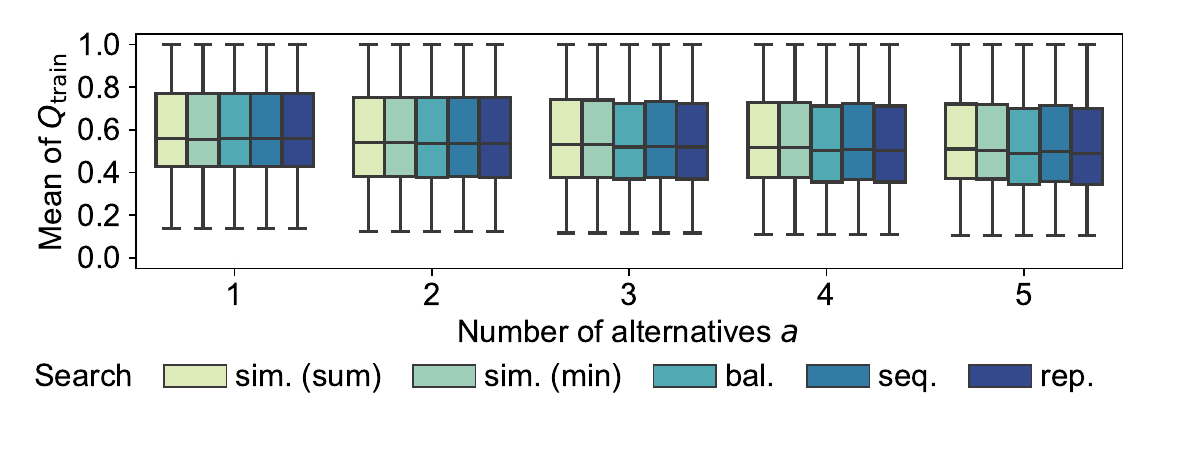}
		\caption{Training-set objective value.}
		\label{fig:afs:impact-search-mean-train-objective}
	\end{subfigure}
	\\ \vspace{\baselineskip}
	\begin{subfigure}[t]{\textwidth}
		\centering
		\includegraphics[width=\textwidth, trim=15 25 35 15, clip]{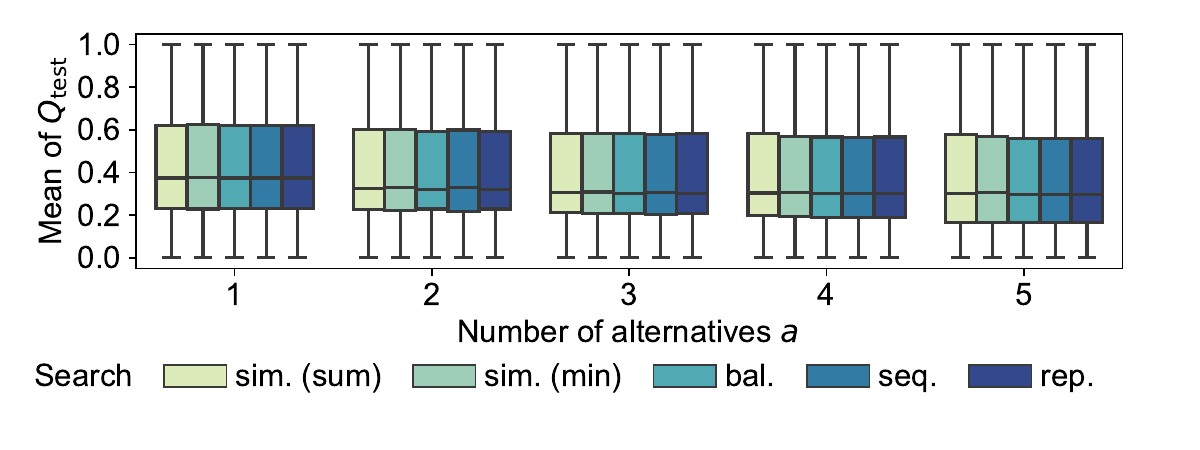}
		\caption{Test-set objective value.}
		\label{fig:afs:impact-search-mean-test-objective}
	\end{subfigure}
	\\ \vspace{\baselineskip}
	\begin{subfigure}[t]{\textwidth}
		\centering
		\includegraphics[width=\textwidth, trim=15 25 35 14, clip]{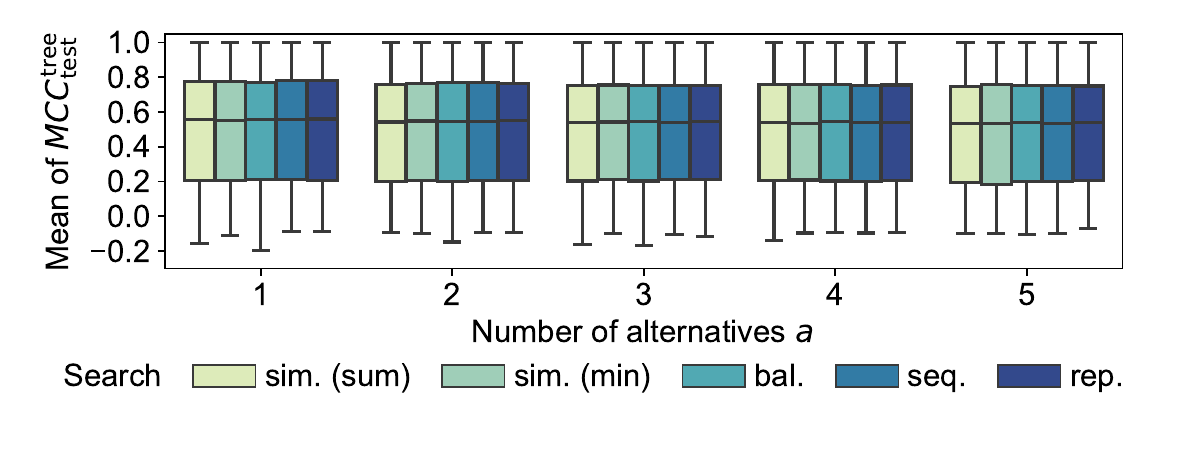}
		\caption{Test-set prediction performance.}
		\label{fig:afs:impact-search-mean-decision-tree-test-mcc}
	\end{subfigure}
	\caption{
		Distribution of mean feature-set quality within search runs over datasets, cross-validation-folds, and~$\tau$, by search method for alternatives and number of alternatives~$a$.
		Results with \emph{Model Gain} as the feature-selection method and $k=5$.
		Excludes search settings where at least one combination of search method and~$a$ yielded no valid solution.
		Y-axes are truncated to improve readability.
	}
	\label{fig:afs:impact-search-mean-quality}
\end{figure}

\paragraph{Average value of feature-set quality}

While obtaining alternatives of homogeneous quality can be one goal of simultaneous search, another selling point would be reaching higher average quality than sequential search.
However, this potential advantage rarely materialized in our experiments.
In particular, Figure~\ref{fig:afs:impact-search-mean-train-objective} compares the distribution of the mean training-set objective value in search runs with \emph{Model Gain} as the feature-selection method and $k=5$.
We observe that all search methods yield very similar overall distributions of average feature-set quality.
Comparing solver-based sequential and simultaneous search on each experimental setting separately and then aggregating also shows a mean quality difference close to zero.
Further, outliers can occur in both directions, i.e., either solver-based search method may yield higher quality.

The mean test-set objective value in Figure~\ref{fig:afs:impact-search-mean-test-objective} and the mean test-set prediction performance in Figure~\ref{fig:afs:impact-search-mean-decision-tree-test-mcc} also exhibit a negligible quality difference between the search methods.
Finally, the other four feature-selection methods do not show a general quality advantage of solver-based simultaneous search either.

\begin{figure}[p]
	\centering
	\begin{subfigure}[t]{0.48\textwidth}
		\centering
		\includegraphics[width=\textwidth, trim=15 30 15 15, clip]{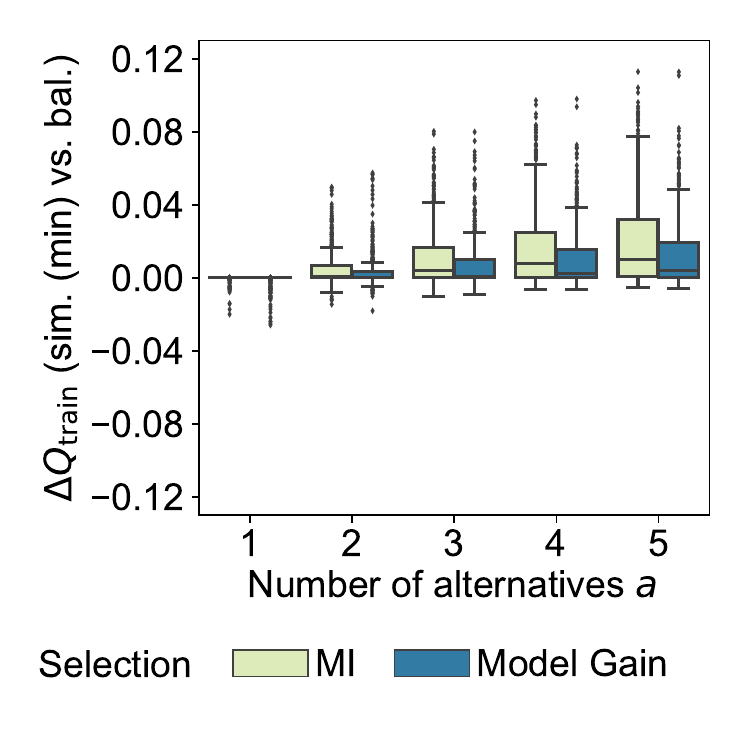}
		\caption{
			Difference between solver-based simultaneous (min-aggregation) search and \emph{Greedy Balancing}, by the number of alternatives~$a$.
		}
		\label{fig:afs:impact-search-heuristics-metric-diff-sim-num-alternatives}
	\end{subfigure}
	\hfill
	\begin{subfigure}[t]{0.48\textwidth}
		\centering
		\includegraphics[width=\textwidth, trim=15 30 15 15, clip]{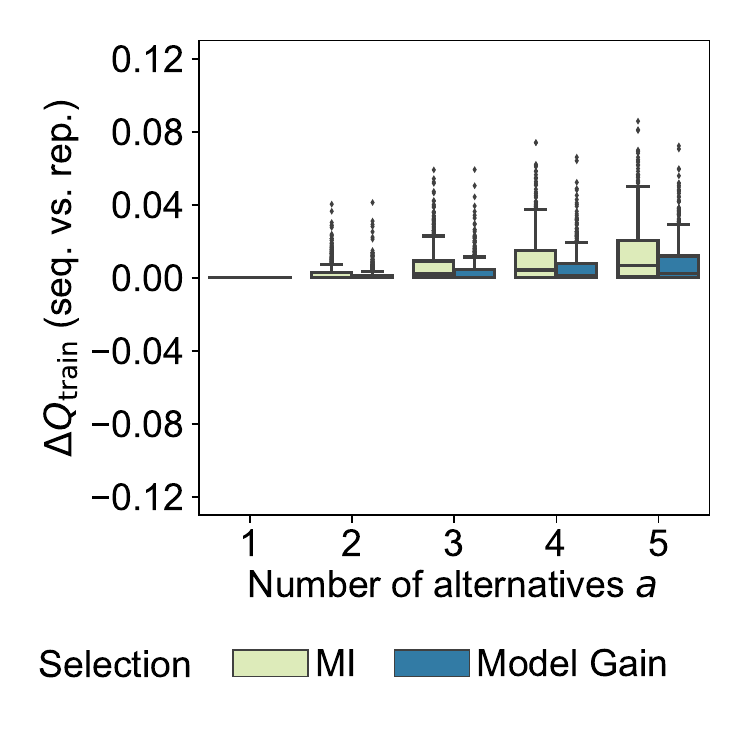}
		\caption{
			Difference between solver-based sequential search and \emph{Greedy Replacement}, by the number of alternatives~$a$.
		}
		\label{fig:afs:impact-search-heuristics-metric-diff-seq-num-alternatives}
	\end{subfigure}
	\\ \vspace{\baselineskip}
	\begin{subfigure}[t]{0.48\textwidth}
		\centering
		\includegraphics[width=\textwidth, trim=15 30 15 15, clip]{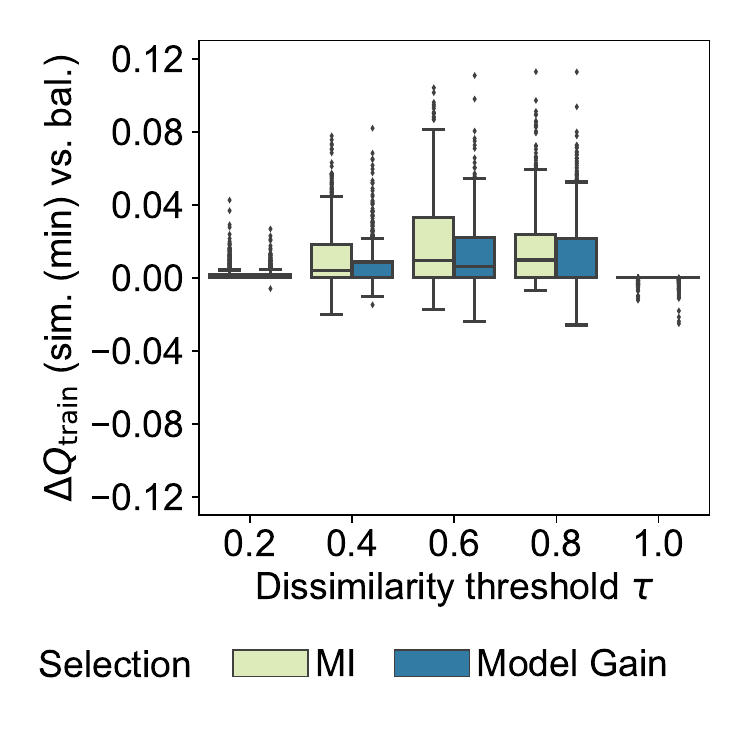}
		\caption{
			Difference between solver-based simultaneous (min-aggregation) search and \emph{Greedy Balancing}, by the dissimilarity threshold~$\tau$.
		}
		\label{fig:afs:impact-search-heuristics-metric-diff-sim-tau}
	\end{subfigure}
	\hfill
	\begin{subfigure}[t]{0.48\textwidth}
		\centering
		\includegraphics[width=\textwidth, trim=15 30 15 15, clip]{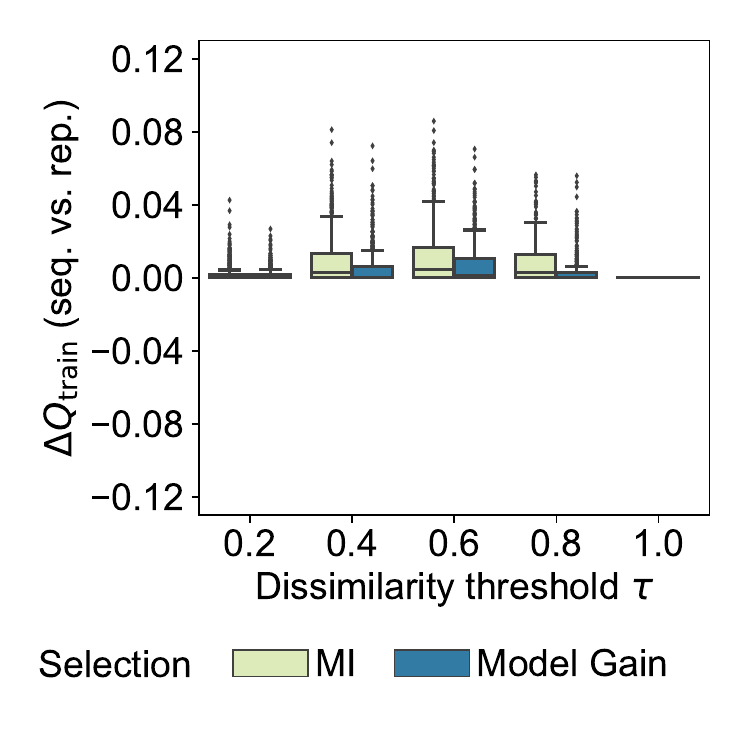}
		\caption{
			Difference between solver-based sequential search and \emph{Greedy Replacement}, by the dissimilarity threshold~$\tau$.
		}
		\label{fig:afs:impact-search-heuristics-metric-diff-seq-tau}
	\end{subfigure}
	\caption{
		Distribution of difference in mean training feature-set quality within a search run between solver-based and heuristic search methods over remaining experimental settings, by feature-selection method.
		Results with $k=5$.
		Excludes search settings where at least one combination of search method and~$a$ yielded no valid solution.
	}
	\label{fig:afs:impact-search-heuristics-metric-diff}
\end{figure}

\paragraph{Quality difference of heuristics}

We now look closer at the quality difference between solver-based and heuristic search.
Figure~\ref{fig:afs:impact-search-heuristics-metric-diff} compares the mean training feature-set quality within search runs for each experimental setting separately.
In particular, we compare solver-based simultaneous search with min-aggregation to \emph{Greedy Balancing} (cf.~Figures~\ref{fig:afs:impact-search-heuristics-metric-diff-sim-num-alternatives} and~\ref{fig:afs:impact-search-heuristics-metric-diff-sim-tau}) and solver-based sequential search to \emph{Greedy Replacement} (cf.~Figures~\ref{fig:afs:impact-search-heuristics-metric-diff-seq-num-alternatives} and~\ref{fig:afs:impact-search-heuristics-metric-diff-seq-tau}).
Positive values in Figure~\ref{fig:afs:impact-search-heuristics-metric-diff} express that solver-based search yields higher quality; negative values favor heuristic search.
The latter can occur if solver-based search yields suboptimal solutions due to timeouts, as we analyze later (cf.~Table~\ref{tab:afs:impact-search-fs-method-optimization-status}).

Figures~\ref{fig:afs:impact-search-heuristics-metric-diff-sim-num-alternatives} and~\ref{fig:afs:impact-search-heuristics-metric-diff-seq-num-alternatives} compare solver-based and heuristic search over the number of alternatives~$a$.
For $a=1$, solver-based and heuristic search yield the same mean training feature-set quality except when timeouts occur.
The more alternatives are desired, the more advantageous a solver-based search is quality-wise.
As in prior analyses, the picture is less clear on the test set, which shows a smaller quality difference here.
Further, the difference in mean training feature-set quality between \emph{Greedy Balancing} and solver-based simultaneous search (cf.~Figure~\ref{fig:afs:impact-search-heuristics-metric-diff-sim-num-alternatives}) grows faster with~$a$ than between \emph{Greedy Replacement} and solver-based sequential search (cf.~Figure~\ref{fig:afs:impact-search-heuristics-metric-diff-seq-num-alternatives}).
As an explanation, consider that simultaneous search may generally develop a quality advantage over sequential search for more alternatives.
\emph{Greedy Balancing} selects the same features as \emph{Greedy Replacement}, i.e., a sequential search heuristic, and only distributes them differently into feature sets, yielding the same mean feature-set quality.

Figures~\ref{fig:afs:impact-search-heuristics-metric-diff-sim-tau} and~\ref{fig:afs:impact-search-heuristics-metric-diff-seq-tau} compare solver-based and heuristic search over the dissimilarity threshold~$\tau$.
Unlike for~$a$, the difference in mean training feature-set quality between solver-based and heuristic search does not increase over the whole range of~$\tau$ but shows an increase followed by a decrease.
In particular, $\tau=1$ allows the two heuristic search methods to reach the same mean training feature-set quality as the solver-based methods.
This observation corresponds to our theoretical result that optimizing the summed quality of alternatives with~$\tau=1$ admits polynomial-time algorithms (cf.~Proposition~\ref{prop:afs:complexity-partitioning-sum}).

\begin{table}[t]
	\centering
	\caption{
		Frequency of optimization statuses (cf.~Section~\ref{sec:afs:experimental-design:evaluation}) over datasets, cross-validation folds, $a$, and $\tau$, by feature-selection method and search method for alternatives.
		Results with $k=5$, $a \in \{1,2,3,4,5\}$, and excluding \emph{Greedy Wrapper}, which does not use the solver for optimizing.
		Each row adds up to 100\%.
	}
	\begin{tabular}{llrrrr}
		\toprule
		\multirow{2}{*}{Feature sel.} & \multirow{2}{*}{Search} & \multicolumn{4}{c}{Optimization status} \\
		\cmidrule(lr){3-6}
		& & Infeasible & Not solved & Feasible & Optimal \\
		\midrule
		FCBF & seq. & 74.51\% & 0.00\% & 0.00\% & 25.49\% \\
		FCBF & sim. (min) & 73.07\% & 0.00\% & 1.60\% & 25.33\% \\
		FCBF & sim. (sum) & 73.07\% & 0.00\% & 2.19\% & 24.75\% \\
		MI & bal. & 0.00\% & 9.20\% & 90.80\% & 0.00\% \\
		MI & rep. & 0.00\% & 9.20\% & 90.80\% & 0.00\% \\
		MI & seq. & 4.93\% & 0.00\% & 0.00\% & 95.07\% \\
		MI & sim. (min) & 4.67\% & 0.00\% & 9.33\% & 86.00\% \\
		MI & sim. (sum) & 4.67\% & 0.00\% & 3.04\% & 92.29\% \\
		Model Gain & bal. & 0.00\% & 9.20\% & 90.80\% & 0.00\% \\
		Model Gain & rep. & 0.00\% & 9.20\% & 90.80\% & 0.00\% \\
		Model Gain & seq. & 4.93\% & 0.00\% & 0.00\% & 95.07\% \\
		Model Gain & sim. (min) & 4.67\% & 0.00\% & 5.28\% & 90.05\% \\
		Model Gain & sim. (sum) & 4.67\% & 0.00\% & 1.87\% & 93.47\% \\
		mRMR & seq. & 4.88\% & 0.00\% & 9.55\% & 85.57\% \\
		mRMR & sim. (min) & 4.67\% & 0.00\% & 48.64\% & 46.69\% \\
		mRMR & sim. (sum) & 4.67\% & 0.00\% & 67.04\% & 28.29\% \\
		\bottomrule
	\end{tabular}
	\label{tab:afs:impact-search-fs-method-optimization-status}
\end{table}

\begin{table}[t]
	\centering
	\caption{
		Frequency of optimization statuses (cf.~Section~\ref{sec:afs:experimental-design:evaluation}) over datasets, cross-validation folds, feature-selection methods, and~$\tau$, by number of alternatives~$a$.
		Results from solver-based simultaneous search with sum-aggregation, $k=5$, and excluding \emph{Greedy Wrapper}.
		Each row adds up to 100\%.
	}
	\begin{tabular}{rrrr}
		\toprule
		\multirow{2}{*}{$a$} & \multicolumn{3}{c}{Optimization status} \\
		\cmidrule(lr){2-4}
		& Infeasible & Feasible & Optimal \\
		\midrule
		1 & 16.10\% & 7.60\% & 76.30\% \\
		2 & 17.50\% & 13.27\% & 69.23\% \\
		3 & 20.00\% & 20.20\% & 59.80\% \\
		4 & 27.00\% & 21.43\% & 51.57\% \\
		5 & 28.23\% & 30.17\% & 41.60\% \\
		\bottomrule
	\end{tabular}
	\label{tab:afs:impact-num-alternatives-optimization-status}
\end{table}

\paragraph{Optimization status}

Suboptimal search results are one reason why solver-based simultaneous search does not consistently beat solver-based sequential search quality-wise.
For \emph{Greedy Wrapper}, the search is heuristic per se and does not cover the entire search space.
For all feature-selection methods, the solver can time out.
Table~\ref{tab:afs:impact-search-fs-method-optimization-status} shows that solver-based simultaneous search has a higher likelihood of timeouts than solver-based sequential search, likely due to the larger size of the optimization problem (cf.~Table~\ref{tab:afs:seq-sim-comparison}).
In particular, for up to five alternatives and $k=5$, all solver-based sequential searches for \emph{FCBF}, \emph{MI}, and \emph{Model Gain} finished within the timeout, i.e., yielded the optimal feature set or ascertained infeasibility, while \emph{mRMR} had about 10\% timeouts.
In contrast, for solver-based simultaneous search with sum-aggregation, all feature-selection methods experience timeouts:
1--3\% of the searches for \emph{FCBF}, \emph{MI}, and \emph{Model Gain}, and 67\% of the searches for \emph{mRMR} found a feasible solution but could not prove optimality.
Such timeout-affected simultaneous solutions can be worse than optimal sequential solutions.

\emph{mRMR} is especially prone to suboptimal solutions, likely because it has a more complex objective (cf.~Equation~\ref{eq:afs:mrmr-linear}) than~\emph{MI} and \emph{Model Gain}.
\emph{FCBF} often results in infeasible optimization problems since its constraints, which prevent the selection of redundant features (cf.~Equation~\ref{eq:afs:fcbf}), might prevent finding any valid feature set of size~$k$.
Min-aggregation instead of sum-aggregation in solver-based simultaneous search exhibits more timeouts for \emph{MI} and \emph{Model Gain} but less for \emph{FCBF} and \emph{mRMR}.
Still, solver-based sequential search incurs fewer timeouts for all of these four feature-selection methods.

Also, note that the fraction of timeouts in solver-based simultaneous search strongly depends on the number of alternatives~$a$, as Table~\ref{tab:afs:impact-num-alternatives-optimization-status} displays:
For $k=5$ and sum-aggregation, roughly 8\% of the white-box searches timed out for~$a=1$, but 20\% for~$a=3$ and 30\% for~$a=5$.
While we grant solver-based simultaneous searches proportionally more time for multiple alternatives (cf.~Section~\ref{sec:afs:experimental-design:approaches:alternatives}), the observed increase in timeouts suggests that runtime increases super-proportionally with~$a$, as we analyze later.

For the heuristic search methods, Table~\ref{tab:afs:impact-search-fs-method-optimization-status} shows that \emph{Greedy Replacement} more often did not find a valid alternative (9.20\%) than solver-based sequential search (4.93\%).
A similar phenomenon occurred for \emph{Greedy Balancing} (9.20\%) compared to solver-based simultaneous search (4.67\%).
Such a result can be expected since both heuristic search methods stop early as soon as each feature is part of at least one alternative.

\begin{table}[t]
	\centering
	\caption{
		Mean optimization time over datasets, cross-validation folds, $a$, and $\tau$, by feature-selection method and search method for alternatives.
		Results with $k=5$ and $a \in \{1,2,3,4,5\}$.
	}
	\begin{tabular}{lrrrrr}
		\toprule
		\multirow{2}{*}{Feature selection} & \multicolumn{5}{c}{Optimization time} \\
		\cmidrule(lr){2-6}
		& Bal. & Rep. & Seq. & Sim. (min) & Sim. (sum) \\
		\midrule
		FCBF & --- & --- & 0.22~s & 11.41~s & 12.62~s \\
		Greedy Wrapper & --- & --- & 52.62~s & 68.39~s & 70.36~s \\
		MI & 0.00~s & 0.00~s & 0.03~s & 47.39~s & 24.56~s \\
		Model Gain & 0.00~s & 0.00~s & 0.03~s & 30.38~s & 19.08~s \\
		mRMR & --- & --- & 33.59~s & 156.00~s & 189.25~s \\
		\bottomrule
	\end{tabular}
	\label{tab:afs:impact-search-fs-method-optimization-time}
\end{table}

\begin{table}[t]
	\centering
	\caption{
		Mean optimization time over datasets, cross-validation folds, and $\tau$, by number of alternatives and feature-selection method.
		Results from solver-based simultaneous search with sum-aggregation and $k=5$.
	}
	\begin{tabular}{lrrrrr}
		\toprule
		\multirow{2}{*}{$a$} & \multicolumn{5}{c}{Optimization time} \\
		\cmidrule(lr){2-6}
		& FCBF & Greedy Wrapper & MI & Model Gain & mRMR \\
		\midrule
		1 & 0.45~s & 28.44~s & 0.03~s & 0.02~s & 44.68~s \\
		2 & 0.86~s & 41.76~s & 0.09~s & 0.08~s & 117.62~s \\
		3 & 2.97~s & 62.70~s & 0.30~s & 0.27~s & 208.14~s \\
		4 & 13.32~s & 96.65~s & 3.68~s & 3.47~s & 258.19~s \\
		5 & 45.52~s & 122.26~s & 118.72~s & 91.58~s & 317.63~s \\
		\bottomrule
	\end{tabular}
	\label{tab:afs:impact-num-alternatives-fs-method-optimization-time}
\end{table}

\paragraph{Optimization time}

As Table~\ref{tab:afs:impact-search-fs-method-optimization-time} shows, solver-based sequential search is faster on average than solver-based simultaneous search for all five feature-selection methods.
In particular, the difference is up to three orders of magnitude for the four white-box feature-selection methods.
Further, \emph{FCBF}, \emph{MI}, and \emph{Model Gain} experience a dramatic increase in optimization time with the number of alternatives~$a$ in solver-based simultaneous search, as Table~\ref{tab:afs:impact-num-alternatives-fs-method-optimization-time} displays.
In contrast, the runtime increase is considerably less for solver-based sequential search, which shows an approximately linear trend with the number of alternatives.

Table~\ref{tab:afs:impact-search-fs-method-optimization-time} also shows that the optimization time of the heuristic search methods for \emph{MI} and \emph{Model Gain} is negligible.
In particular, \emph{Greedy Replacement} and \emph{Greedy Balancing} never took longer than 1~ms per search run for alternatives.
These results highlight the runtime advantage of the heuristics, particularly of \emph{Greedy Balancing} for simultaneous search.

An interesting question for practitioners is how the runtime relates to~$n$, the number of features in the dataset.
One could expect a positive correlation since the problem instance increases with~$n$.
Roughly speaking, this trend appears in our experimental data indeed.
However, the observed trend is rather noisy, particularly for solver-based simultaneous search, and some higher-dimensional datasets even show lower average runtimes than lower-dimensional ones.
This result indicates that other factors than~$n$ influence runtime as well, e.g., other experimental settings or the solver's internal search heuristics.

Based on all results described in this section, we focus on solver-based sequential search in the next section.
In particular, it was significantly faster than solver-based simultaneous search while yielding similar feature-set quality.

\subsection{User Parameters \texorpdfstring{$a$ And $\tau$}{}}
\label{sec:afs:evaluation:parameters}

\begin{figure}[p]
	\centering
	\begin{subfigure}[t]{0.48\textwidth}
		\centering
		\includegraphics[width=\textwidth, trim=15 17 10 15, clip]{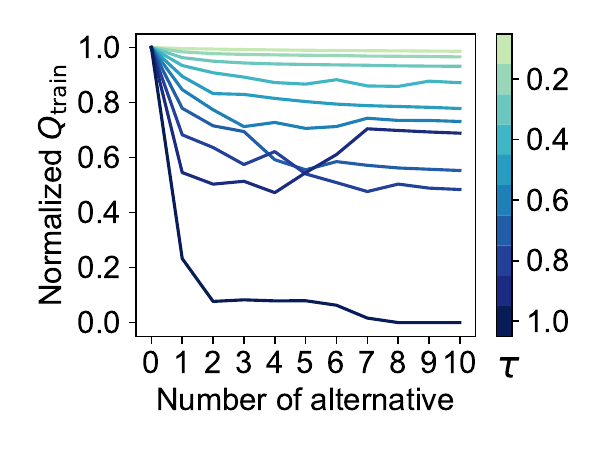}
		\caption{
			Training-set objective value.
			Infeasible feature sets excluded.
		}
		\label{fig:afs:impact-num-alternatives-tau-train-objective-max}
	\end{subfigure}
	\hfill
	\begin{subfigure}[t]{0.48\textwidth}
		\centering
		\includegraphics[width=\textwidth, trim=15 17 10 15, clip]{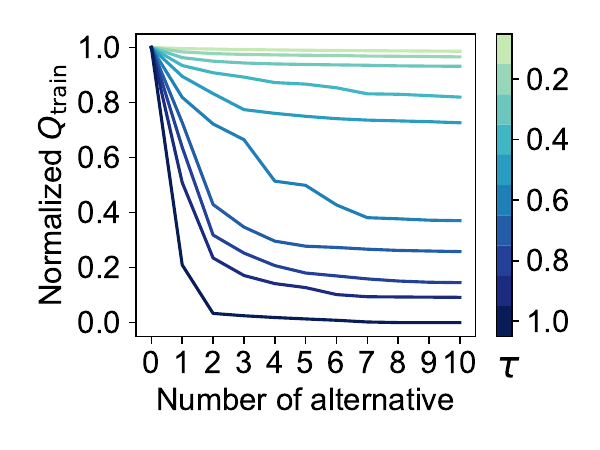}
		\caption{
			Training-set objective value.
			Infeasible feature sets assigned a quality of~0.
		}
		\label{fig:afs:impact-num-alternatives-tau-train-objective-max-fillna}
	\end{subfigure}
	\\ \vspace{\baselineskip}
	\begin{subfigure}[t]{0.48\textwidth}
		\centering
		\includegraphics[width=\textwidth, trim=15 17 10 15, clip]{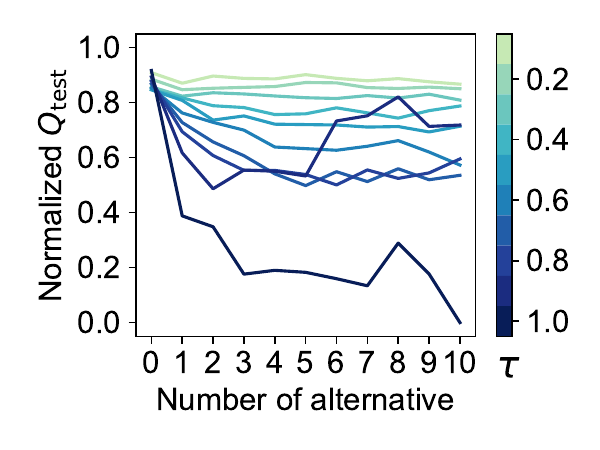}
		\caption{
			Test-set objective value.
			Infeasible feature sets excluded.
		}
		\label{fig:afs:impact-num-alternatives-tau-test-objective-max}
	\end{subfigure}
	\hfill
	\begin{subfigure}[t]{0.48\textwidth}
		\centering
		\includegraphics[width=\textwidth, trim=15 17 10 15, clip]{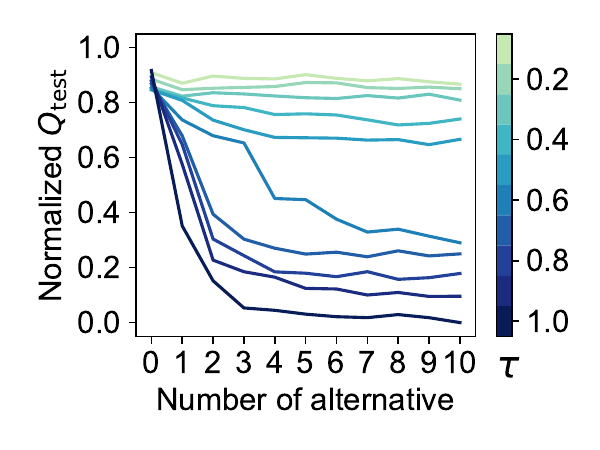}
		\caption{
			Test-set objective value.
			Infeasible feature sets assigned a quality of~0.
		}
		\label{fig:afs:impact-num-alternatives-tau-test-objective-max-fillna}
	\end{subfigure}
	\\ \vspace{\baselineskip}
	\begin{subfigure}[t]{0.48\textwidth}
		\centering
		\includegraphics[width=\textwidth, trim=15 17 10 15, clip]{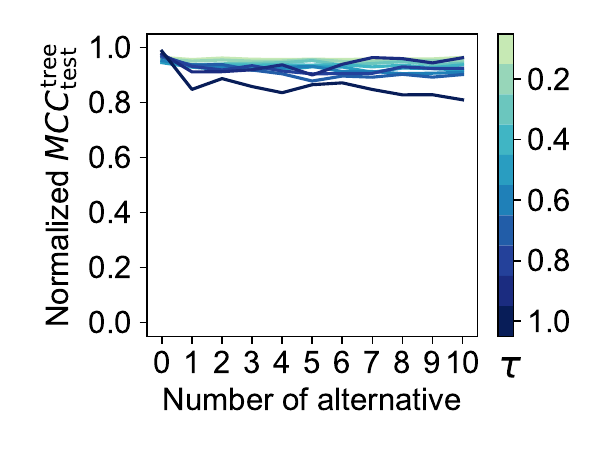}
		\caption{
			Test-set prediction performance.
			Infeasible feature sets excluded.
		}
		\label{fig:afs:impact-num-alternatives-tau-decision-tree-test-mcc-max}
	\end{subfigure}
	\hfill
	\begin{subfigure}[t]{0.48\textwidth}
		\centering
		\includegraphics[width=\textwidth, trim=15 17 10 15, clip]{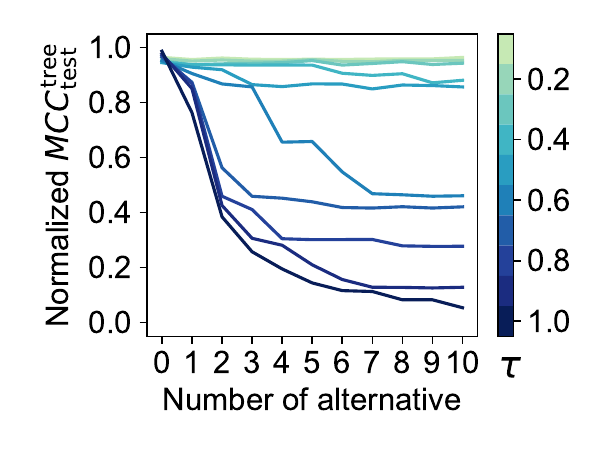}
		\caption{
			Test-set prediction performance.
			Infeasible feature sets assigned a quality of~0.
		}
		\label{fig:afs:impact-num-alternatives-tau-decision-tree-test-mcc-max-fillna}
	\end{subfigure}
	\caption{
		Mean feature-set quality over datasets and cross-validation folds, max-normalized per search run for alternatives, by the number of alternative and dissimilarity threshold~$\tau$.
		Results from solver-based sequential search with \emph{Model Gain} as the feature-selection method and $k=10$.
	}
	\label{fig:afs:impact-num-alternatives-tau-quality}
\end{figure}

\paragraph{Feature-set quality}

Higher values of the two user parameters introduce more (for~$a$) or stronger (for~$\tau$) constraints into the optimization problem of alternative feature selection.
Thus, one would expect a corresponding decrease in feature-set quality.
Figure~\ref{fig:afs:impact-num-alternatives-tau-quality} illustrates this trend for \emph{Model Gain} as the feature-selection method and $k=10$.
Since the maximum feature-set quality varies among datasets, we max-normalize quality in this figure.
In particular, we set the highest feature-set quality in each search run for alternatives to~1 and scale the other feature-set qualities accordingly.
For prediction performance in terms of MCC, we shift its range from~$[-1,1]$ to~$[0,1]$ before normalization.

Figure~\ref{fig:afs:impact-num-alternatives-tau-quality} shows that multiple alternatives may have a similar quality.
Further, the training-set objective value (cf.~Figure~\ref{fig:afs:impact-num-alternatives-tau-train-objective-max}) decreases most from the original feature set, i.e., the zeroth alternative, to the first alternative, but less beyond.
Also, the decrease strongly depends on the dissimilarity threshold~$\tau$.
For a low dissimilarity threshold like $\tau=0.1$, the training-set objective value barely drops over the number of alternatives.
Additionally, note that Figure~\ref{fig:afs:impact-num-alternatives-tau-quality} averages the normalized feature-set quality over multiple datasets.
In our experiments, datasets with more features tend to experience a smaller decrease in quality over~$a$ and~$\tau$.
As higher-dimensional datasets offer more options for alternatives, this observation makes sense.
However, this effect is not guaranteed since datasets with many features could also contain many useless features instead of interesting alternatives.

The overall decrease in quality is slightly less pronounced for the test-set objective value (Figure~\ref{fig:afs:impact-num-alternatives-tau-test-objective-max}) than on the training set (Figure~\ref{fig:afs:impact-num-alternatives-tau-train-objective-max}) since overfitting might occur.
In particular, the original feature set can even have lower test-set quality than the subsequent alternatives.
The trend becomes even less clear for prediction performance, which varies little over~$a$ and~$\tau$ in our experiments (cf.~Figure~\ref{fig:afs:impact-num-alternatives-tau-decision-tree-test-mcc-max}).
In general, the optimization objective~$Q$ may only partially indicate actual prediction performance since the former may use a simplified feature-set quality criterion.
Indeed, the correlation between optimization objective~$Q$ and prediction MCC is only weak to moderate in our experiments.

\begin{figure}[t]
	\centering
	\begin{subfigure}[t]{0.48\textwidth}
		\centering
		\includegraphics[width=\textwidth, trim=15 15 10 15, clip]{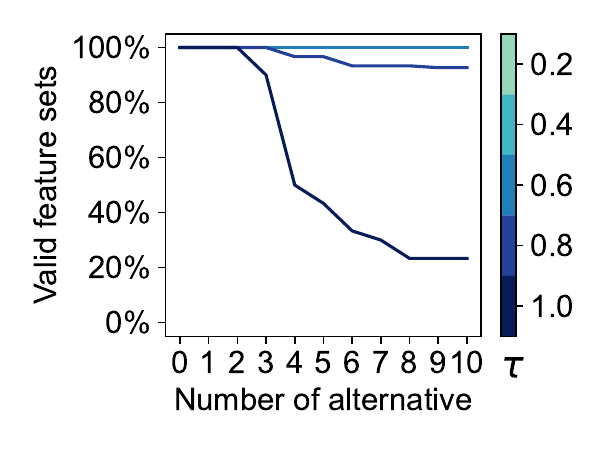}
		\caption{Feature-set size~$k=5$.}
		\label{fig:afs:impact-num-alternatives-tau-optimization-status-k-5}
	\end{subfigure}
	\hfill
	\begin{subfigure}[t]{0.48\textwidth}
		\centering
		\includegraphics[width=\textwidth, trim=15 15 10 15, clip]{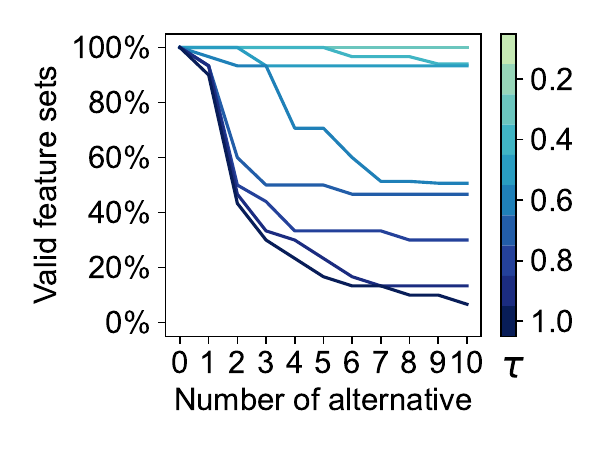}
		\caption{Feature-set size~$k=10$.}
		\label{fig:afs:impact-num-alternatives-tau-optimization-status-k-10}
	\end{subfigure}
	\caption{
		Frequency of optimization runs yielding a valid feature set over datasets and cross-validation folds, by the number of alternative and dissimilarity threshold~$\tau$.
		Results from solver-based sequential search with \emph{Model Gain} as the feature-selection method.
	}
	\label{fig:afs:impact-num-alternatives-tau-optimization-status}
\end{figure}

\paragraph{Optimization status}

The previous observations refer to the quality of the found feature sets.
However, the more alternatives one desires and the more they should differ, the likelier an infeasible optimization problem is.
Figure~\ref{fig:afs:impact-num-alternatives-tau-optimization-status} visualizes the fraction of valid feature sets over the number of alternatives and dissimilarity threshold~$\tau$, showing the expected trend.
Additionally, Figures~\ref{fig:afs:impact-num-alternatives-tau-train-objective-max-fillna},~\ref{fig:afs:impact-num-alternatives-tau-test-objective-max-fillna}, and~\ref{fig:afs:impact-num-alternatives-tau-decision-tree-test-mcc-max-fillna} display the same data as Figures~\ref{fig:afs:impact-num-alternatives-tau-train-objective-max},~\ref{fig:afs:impact-num-alternatives-tau-test-objective-max}, and~\ref{fig:afs:impact-num-alternatives-tau-decision-tree-test-mcc-max} but with the quality of infeasible feature sets set to zero instead of excluding these feature sets from evaluation.
Consequently, the decrease in feature-set quality over~$a$ and~$\tau$ is noticeably stronger.
In contrast, if only considering valid feature sets, the mean quality in our experiments can increase over the number of alternatives, e.g., as visible in Figures~\ref{fig:afs:impact-num-alternatives-tau-train-objective-max} and~\ref{fig:afs:impact-num-alternatives-tau-test-objective-max} for $\tau=0.9$.
This counterintuitive phenomenon can occur because some datasets run out of valid feature sets sooner than others, so the average quality may be determined for different sets of datasets at each number of alternatives.

\begin{figure}[t]
	\centering
	\begin{subfigure}[t]{0.48\textwidth}
		\centering
		\includegraphics[width=\textwidth, trim=15 17 10 15, clip]{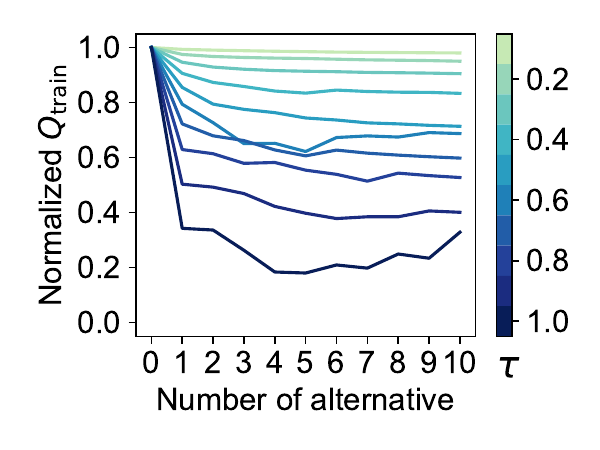}
		\caption{
			\emph{MI} for feature selection.
		}
		\label{fig:afs:impact-num-alternatives-tau-train-objective-max-mi}
	\end{subfigure}
	\hfill
	\begin{subfigure}[t]{0.48\textwidth}
		\centering
		\includegraphics[width=\textwidth, trim=15 17 10 15, clip]{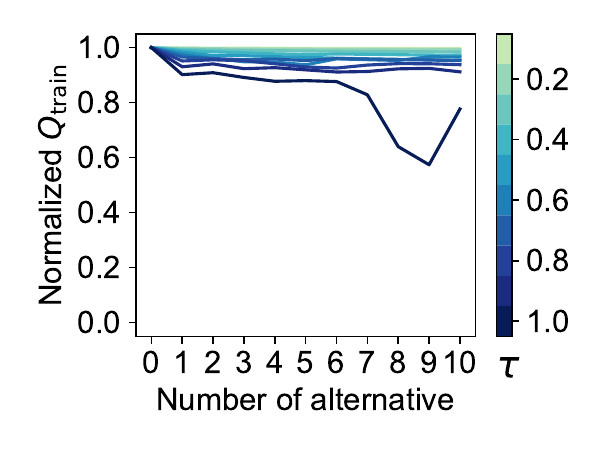}
		\caption{
			\emph{mRMR} for feature selection.
		}
		\label{fig:afs:impact-num-alternatives-tau-train-objective-max-mrmr}
	\end{subfigure}
	\\ \vspace{\baselineskip}
	\begin{subfigure}[t]{0.48\textwidth}
		\centering
		\includegraphics[width=\textwidth, trim=15 17 10 15, clip]{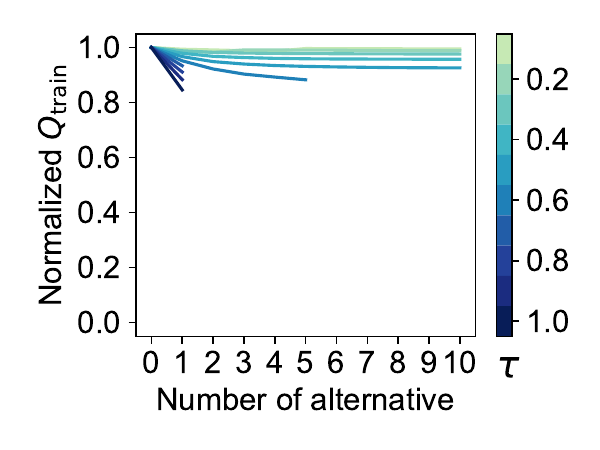}
		\caption{
			\emph{FCBF} for feature selection.
		}
		\label{fig:afs:impact-num-alternatives-tau-train-objective-max-fcbf}
	\end{subfigure}
	\hfill
	\begin{subfigure}[t]{0.48\textwidth}
		\centering
		\includegraphics[width=\textwidth, trim=15 17 10 15, clip]{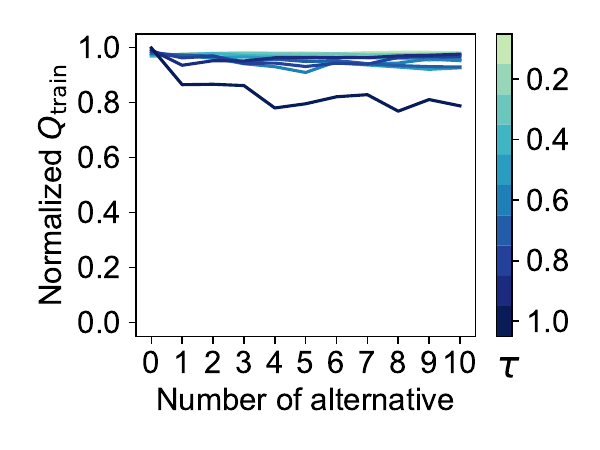}
		\caption{
			\emph{Greedy Wrapper} for feature selection.
		}
		\label{fig:afs:impact-num-alternatives-tau-train-objective-max-greedy-wrapper}
	\end{subfigure}
	\caption{
		Mean training-set objective value over datasets and cross-validation folds, max-normalized per search run for alternatives, by the number of alternative and dissimilarity threshold~$\tau$.
		Results from solver-based sequential search with $k=10$.
		Infeasible feature sets excluded.
	}
	\label{fig:afs:impact-num-alternatives-tau-train-objective-fs-method}
\end{figure}

\paragraph{Influence of feature-selection method}

While we discussed \emph{Model Gain} before, the decrease in objective value over~$a$ and~$\tau$ occurs to different extents for the other feature-selection methods in our experiments, as Figure~\ref{fig:afs:impact-num-alternatives-tau-train-objective-fs-method} displays.
In this figure, we shifted the objectives of \emph{Greedy Wrapper} and \emph{mRMR} to~$[0, 1]$ before max-normalization since their original range is~$[-1,1]$.
\emph{MI} (cf.~Figure~\ref{fig:afs:impact-num-alternatives-tau-train-objective-max-mi}) shows a similar behavior as \emph{Model Gain} (cf.~Figure~\ref{fig:afs:impact-num-alternatives-tau-train-objective-max}), which may result from both feature-selection methods using the same objective function, though with different feature qualities.
In contrast, \emph{mRMR} (cf.~Figure~\ref{fig:afs:impact-num-alternatives-tau-train-objective-max-mrmr}) exhibits a considerably smaller effect of increasing~$\tau$.
For \emph{FCBF} (cf.~Figure~\ref{fig:afs:impact-num-alternatives-tau-train-objective-max-fcbf}), the additional constraints on feature-feature correlation (cf.~Equation~\ref{eq:afs:fcbf}) cause many infeasible results (cf.~Table~\ref{tab:afs:impact-search-fs-method-optimization-status}), so we cannot determine the average feature-set quality for some combinations of~$a$ and~$\tau$.
For \emph{Greedy Wrapper} (cf.~Figure~\ref{fig:afs:impact-num-alternatives-tau-train-objective-max-greedy-wrapper}), $a$ and~$\tau$ barely make any difference, which may be explained by the heuristic, inexact search procedure.

\subsection{Summary}
\label{sec:afs:evaluation:summary}

\paragraph{Feature-selection methods (cf.~Section~\ref{sec:afs:evaluation:feature-selection})}

Among the feature-selection methods, \emph{Model Gain} yielded the best average prediction performance.
The simple univariate \emph{MI} also turned out competitive, while \emph{Greedy Wrapper} and \emph{mRMR} required high optimization times, and our constraint-based version of \emph{FCBF} yielded many infeasible solutions.
Selecting $k=10$ instead of $k=5$ features had only a small impact on prediction performance, so users may stick to smaller feature-set sizes if such a setting benefits interpretability.

\paragraph{Search methods for alternatives (cf.~Section~\ref{sec:afs:evaluation:search-methods})}

Solver-based simultaneous search, particularly with min-aggregation, considerably reduced the variance of the training-set objective value over alternatives compared to solver-based sequential search, as we desired.
However, results were less clear on the test set and when using prediction performance to measure feature-set quality.
Further, the average quality of alternatives was similar to solver-based sequential search.
In addition, the latter was considerably faster and led to fewer solver timeouts, particularly when increasing the number of alternatives.
Also, sequential search allows users to stop searching after each alternative.

The heuristic search methods \emph{Greedy Replacement} and \emph{Greedy Balancing} for univariate feature qualities achieved a good feature-set quality relative to solver-based search, particularly for a low number of alternatives and on the test set.
For a high number of alternatives, the training feature-set quality may differ more, and the heuristics may stop early despite the existence of further alternatives.
As a positive point, both the heuristics' runtime was negligible.
Also, \emph{Greedy Balancing} achieved a low variance of training-set objective value between alternatives, similar to solver-based simultaneous search with min-aggregation.

\paragraph{User parameters $a$ and $\tau$ (cf.~Section~\ref{sec:afs:evaluation:parameters})}

Feature-set quality tended to decrease with a larger number of alternatives~$a$ and dissimilarity threshold~$\tau$, so these parameters give users control over alternatives.
The decrease was highest from the original feature set to the first alternative but smaller beyond, resulting in multiple alternatives of similar quality.
Also, the decrease was more prominent on the training set than on the test set.
Further, the strength of this decrease depended on the feature-selection method;
\emph{MI} and \emph{Model Gain} showed the largest effect.
Independent from the feature-selection method, the frequency of infeasible solutions increased with~$a$ and~$\tau$ due to stronger constraints.

\section{Conclusions and Future Work}
\label{sec:afs:conclusion}

In this section, we summarize our work (cf.~Section~\ref{sec:afs:conclusion:conclusion}) and give an outlook on potential future work (cf.~Section~\ref{sec:afs:conclusion:future-work}).

\subsection{Conclusions}
\label{sec:afs:conclusion:conclusion}

Feature-selection methods are a valuable tool to foster interpretable predictions.
Conventional feature-selection methods typically yield only one feature set.
However, users may be interested in obtaining multiple, sufficiently diverse feature sets of high quality.
Such alternative feature sets may provide alternative explanations for predictions from the data.

In this article, we defined alternative feature selection as an optimization problem.
We formalized alternatives via constraints that are independent of the feature-selection method, can be combined with other constraints on feature sets, and allow users to control diversity according to their needs with two parameters, i.e., the number of alternatives~$a$ and a dissimilarity threshold~$\tau$.
Further, we discussed how to integrate different categories of conventional feature-selection methods as objectives.
We also analyzed the complexity of this optimization problem and proved $\mathcal{NP}$-hardness, even for simple notions of feature-set quality.
Additionally, we showed that the problem gives way to a constant-factor approximation under certain conditions, and we proposed corresponding heuristic search methods.
Finally, we evaluated alternative feature selection with 30 classification datasets and five feature-selection methods.
We compared sequential and simultaneous search for alternatives, both with solver-based and heuristic search methods, and varied the number of alternatives as well as the dissimilarity threshold for alternatives.

\subsection{Future Work}
\label{sec:afs:conclusion:future-work}

\paragraph{Feature selection (objective function)}

One could search for alternatives with other feature-selection methods than the five we analyzed.
In particular, we implemented only one procedure to find alternatives for wrapper feature selection (cf.~Section~\ref{sec:afs:approach:objectives:black-box}).
Embedded feature selection, which we did not evaluate, would also need adapted search methods for alternatives (cf.~Section~\ref{sec:afs:approach:objectives:embedding}).

\paragraph{Alternatives (constraints)}

One could vary the definition of alternatives, e.g., the set-dissimilarity measure (cf.~Section~\ref{sec:afs:approach:constraints:single}), the quality aggregation for simultaneous alternatives (cf.~Appendix~\ref{sec:afs:appendix:simultaneous-objective-aggregation}), or the overall optimization problem (cf.~Section~\ref{sec:afs:approach:problem}).
While we made general and straightforward decisions for each of these points, particular applications might demand other formalizations of alternatives.
E.g., one could use soft instead of hard constraints.

\paragraph{Time complexity}

Appendix~\ref{sec:afs:appendix:complexity:future-work} discusses how one could extend our complexity analysis of alternative feature selection (cf.~Section~\ref{sec:afs:approach:complexity}).

\paragraph{Runtime}

Our experiments (cf.~Section~\ref{sec:afs:evaluation:search-methods}) and theoretical analyses (cf.~Section~\ref{sec:afs:approach:constraints:multiple}) revealed that exact simultaneous search scales poorly with the number of alternatives.
One could conceive a more efficient problem formulation.
Further, one could limit the solver runtime and take the intermediate results once the timeout is reached.
We already used a fixed timeout in our experiments, but studying the exact influence of timeouts on feature-set quality is an open topic.
Next, one could use a different solver, e.g., one for non-linear optimization, so the auxiliary variables from Equation~\ref{eq:afs:product-linear} become superfluous.
Finally, one could develop further simultaneous heuristics (cf.~Section~\ref{sec:afs:approach:univariate-heuristics:greedy-balancing}).

\paragraph{Datasets}

In the current article, we conducted a broad quantitative evaluation of alternative feature selection on datasets from various domains. (cf.~Section~\ref{sec:afs:experimental-design:datasets}).
While we uncovered several general trends, the existence and quality of alternatives naturally depend on the dataset.
Thus, practitioners may employ alternative feature selection in domain-specific case studies and evaluate the alternative feature sets qualitatively, thereby assessing their usefulness for interpreting predictions.

\appendix

\section{Appendix}
\label{sec:afs:appendix}

In this section, we provide supplementary materials.
Section~\ref{sec:afs:appendix:simultaneous-objective-aggregation} discusses possible aggregation operators for the objective of the simultaneous-search problem (cf.~Definition~\ref{def:afs:alternative-feature-selection-simultaneous}).
Section~\ref{sec:afs:appendix:multivariate-filter-objectives} discusses additional objective functions for multivariate filter feature selection (cf.~Section~\ref{sec:afs:approach:objectives:white-box}).
Section~\ref{sec:afs:appendix:univariate-complete-optimization-problem} provides complete definitions of the alternative-feature-selection problem (cf.~Section~\ref{sec:afs:approach:constraints}) for the univariate objective (cf.~Equation~\ref{eq:afs:univariate-filter}).
Section~\ref{sec:afs:appendix:univariate-pre-selection} proposes how to speed up optimization for the univariate objective (cf.~Equation~\ref{eq:afs:univariate-filter}).
Section~\ref{sec:afs:appendix:complexity} complements the complexity analysis (cf.~Section~\ref{sec:afs:approach:complexity}).
Section~\ref{sec:afs:appendix:greedy-depth} proposes another heuristic search method for the univariate objective (cf.~Equation~\ref{eq:afs:univariate-filter}), complementing Section~\ref{sec:afs:approach:univariate-heuristics}.

\subsection{Aggregation Operators for the Simultaneous-Search Problem}
\label{sec:afs:appendix:simultaneous-objective-aggregation}

In this section, we discuss operators to aggregate the feature-set quality of multiple alternatives in the objective of the simultaneous-search problem (cf.~Definition~\ref{def:afs:alternative-feature-selection-simultaneous}).

\paragraph{Sum-aggregation}

The arguably simplest way to aggregate the qualities of multiple feature sets is to sum them up, which we call \emph{sum-aggregation}:
\begin{equation}
	\max_{s^{(0)}, \dots, s^{(a)}} \sum_{l=0}^a Q(s^{(l)},X,y)
	\label{eq:afs:afs-simultaneous-sum-objective}
\end{equation}
While this objective fosters a high average quality of feature sets, it does not guarantee that the alternatives have similar quality:
\begin{example}[Sum-aggregation]
Consider $n=6$~features with univariate feature qualities (cf.~Equation~\ref{eq:afs:univariate-filter}) $q = (9,8,7,3,2,1)$, feature-set size~$k=3$, number of alternatives~$a=2$, the Dice dissimilarity (cf.~Equation~\ref{eq:afs:dice}) as~$d(\cdot)$, and dissimilarity threshold~$\tau = 0.5$, which permits an overlap of one feature between sets here (cf.~Equation~\ref{eq:afs:dice-rearranged-equal-size}).
Exact sequential search (cf.~Definition~\ref{def:afs:alternative-feature-selection-sequential}) yields the selection $s^{(0)} = (1,1,1,0,0,0)$, $s^{(1)} = (1,0,0,1,1,0)$, and $s^{(2)} = (0,1,0,1,0,1)$, with a summed quality of $\,24+14+12=50$.
One possible exact simultaneous-search (cf.~Definition~\ref{def:afs:alternative-feature-selection-simultaneous}) solution consists of the feature sets $s^{(0)} = (1,1,0,1,0,0)$, $s^{(1)} = (1,0,1,0,1,0)$, and $s^{(2)} = (0,1,1,0,0,1)$, with a summed quality of $\,20+18+16=54$.
Another possible exact simultaneous-search solution is $s^{(0)} = (1,1,0,0,0,1)$, $s^{(1)} = (1,0,1,0,1,0)$, and $s^{(2)} = (0,1,1,1,0,0)$, with a summed quality of $\,18+18+18=54$.
\label{ex:afs:sum-aggregation}
\end{example}
This example allows several insights.
First, exact sequential search yields worse quality than exact simultaneous search here, i.e., 50 vs.~54.
Second, the feature-set qualities of the sequential solution, i.e., 24, 14, and~12, differ significantly.
Third, exact simultaneous search can yield multiple solutions whose feature-set quality is differently balanced.
Here, the feature-set qualities in the second simultaneous-search solution, i.e., 18, 18, and~18, are more balanced than in the first, i.e., 20, 18, and~16.
However, both solutions are equally optimal for sum-aggregation.

\paragraph{Min-aggregation}

To actively foster balanced feature-set qualities in simultaneous search, we propose \emph{min-aggregation} in the objective:
\begin{equation}
	\max_{s^{(0)}, \dots, s^{(a)}} \min_{l \in \{0, \dots, a\}} Q(s^{(l)},X,y) \\
	\label{eq:afs:afs-simultaneous-min-objective}
\end{equation}
In the terminology of social choice theory, this objective uses an egalitarian rule instead of a utilitarian one~\cite{myerson1981utilitarianism}.
In particular, \emph{min-aggregation} maximizes the quality of the worst selected alternative.
Thereby, it incentivizes all alternatives to have high quality and implicitly balances their quality.

Note that optimizing the objective with either sum-aggregation or min-aggregation does not necessarily optimize the other.
We already showed a solution optimizing sum-aggregation but not min-aggregation (cf.~Example~\ref{ex:afs:sum-aggregation}).
In the following, we demonstrate the other direction:
\begin{example}[Min-aggregation]
Consider $n=6$~features with univariate feature qualities (cf.~Equation~\ref{eq:afs:univariate-filter}) $q = (11,10,6,5,4,1)$, feature-set size~$k=3$, number of alternatives~$a=1$, the Dice dissimilarity (cf.~Equation~\ref{eq:afs:dice}) as~$d(\cdot)$, and dissimilarity threshold~$\tau = 0.5$, which permits an overlap of one feature between sets here (cf.~Equation~\ref{eq:afs:dice-rearranged-equal-size}).
One simultaneous-search (cf.~Definition~\ref{def:afs:alternative-feature-selection-simultaneous}) solution with min-aggregation (cf.~Equation~\ref{eq:afs:afs-simultaneous-min-objective}) is $s^{(0)} = (1,1,0,0,1,0)$ and $s^{(1)} = (1,0,1,1,0,0)$, with a summed quality of $\,25+22=47$.
Another solution is $s^{(0)} = (1,1,0,0,0,1)$ and $s^{(1)} = (1,0,1,1,0,0)$, with a summed quality of $\,22+22=44$.
\label{ex:afs:min-aggregation}
\end{example}
While both solutions have the same minimum feature-set quality, only the first solution optimizes the objective with sum-aggregation.
In particular, min-aggregation permits reducing the quality of feature sets as long as the latter remains above the minimum quality of all sets.

From the technical perspective, Equation~\ref{eq:afs:afs-simultaneous-min-objective} has the disadvantage of being non-linear regarding the decision variables $s^{(0)}, \dots, s^{(a)}$.
However, we can linearize it with one constraint per feature set and an auxiliary variable~$Q_{\text{min}}$:
\begin{equation}
	\begin{aligned}
		\max_{s^{(0)}, \dots, s^{(a)}} &\quad &Q_{\text{min}} & \\
		\text{subject to:} &\quad \forall l \in \{0, \dots, a\}: &Q_{\text{min}} &\leq Q(s^{(l)},X,y) \\
		&\quad & Q_{\text{min}} &\in \mathbb{R}
	\end{aligned}
	\label{eq:afs:afs-simultaneous-min-objective-linear}
\end{equation}
As we maximize~$Q_{\text{min}}$, this variable will implicitly assume the actual minimum value of~$Q(s^{(l)},X,y)$ with equality since the solution would not be optimal otherwise.
This situation relieves us from introducing further auxiliary variables that are usually necessary when linearizing maximum or minimum expressions~\cite{mosek2022modeling}.

\paragraph{Further approaches for balancing quality}

Min-aggregation provides no control or guarantee of how much the feature-set qualities will actually differ between alternatives since it only incentivizes high quality for all sets.
One can alleviate this issue by adapting the objective or constraints.
First, related work on \textsc{Multi-Way Number Partitioning} also uses other objectives for balancing~\cite{korf2010objective, lawrinenko2017identical} (cf.~Section~\ref{sec:afs:appendix:complexity:related-work}).
E.g., one could minimize the difference between maximum and minimum feature-set quality.
Second, one could use sum-aggregation but constrain the minimum or maximum quality of sets, or the difference between the qualities.
However, such constraint-based approaches introduce one or several parameters bounding feature-set quality, which are difficult to determine a priori.
Third, one could treat balancing qualities as another objective besides maximizing the summed quality.
One can then optimize two objectives simultaneously, filtering results for Pareto-optimal solutions or optimizing a weighted combination of the two objectives.
In both cases, users may need to define an acceptable trade-off between the objectives.
It is an open question if a solution that jointly optimizes min- and sum-aggregation always exists.
If yes, then optimizing a weighted combination of the two objectives would also optimize each of them on its own, assuming positive weights.

\subsection{Further Objectives for Multivariate Filter Methods}
\label{sec:afs:appendix:multivariate-filter-objectives}

While Section~\ref{sec:afs:approach:objectives:white-box} already addressed FCBF and mRMR as multivariate filter feature-selection methods, we discuss the objectives of CFS and Relief here.

\paragraph{CFS}

Correlation-based Feature Selection (CFS)~\cite{hall1999correlation, hall2000correlation} follows a similar principle as mRMR but uses the ratio instead of the difference between a relevance term and a redundancy term for feature-set quality.
Using a bivariate dependency measure $q(\cdot)$ to quantify correlation, the objective is as follows:
\begin{equation}
	Q_{\text{CFS}}(s,X,y) = \frac{\sum_{j=1}^{n} q(X_{\cdot{}j},y) \cdot s_j}{\sqrt{\sum_{j=1}^{n} s_j + \sum_{j_1=1}^{n} \sum_{\substack{j_2=1 \\ j_2 \neq j_1}}^{n} q(X_{\cdot{}j_1}, X_{\cdot{}j_2}) \cdot s_{j_1} \cdot s_{j_2}}}
	\label{eq:afs:cfs}
\end{equation}
One can square this objective to remove the square root in the denominator~\cite{nguyen2010towards}.
Nevertheless, the objective remains non-linear in the decision variables~$s$ since it involves a fraction and multiplications between variables.
However, one can linearize the objective with additional variables and constraints~\cite{nguyen2010improving, nguyen2010towards}, allowing to formulate alternative feature selection for CFS as a linear problem.

\paragraph{Relief}

Relief~\cite{kira1992feature, robnik1997adaptation} builds on the idea that data objects with a similar value of the prediction target should have similar feature values, but data objects that differ in their target should differ in their feature values.
Relief assigns a score to each feature by sampling data objects and quantifying the difference in feature values and target values compared to their nearest neighbors.
We deem Relief to be multivariate since the nearest-neighbor computations involve all features instead of considering them independently.
However, the resulting feature scores can directly be put into the univariate objective (cf.~Equation~\ref{eq:afs:univariate-filter}) to obtain a linear problem.
One can also use Relief scores in CFS to consider feature redundancy~\cite{hall1999correlation, hall2000correlation}, which the default Relief does not.

\subsection{Complete Specifications of the Optimization Problem for the Univariate Objective}
\label{sec:afs:appendix:univariate-complete-optimization-problem}

In this section, we provide complete specifications of the alternative-feature-selection problem for sequential and simultaneous search as integer linear problem.
In particular, we combine all relevant definitions and equations from Section~\ref{sec:afs:approach}.
We use the objective of univariate filter feature selection (cf.~Equation~\ref{eq:afs:univariate-filter}).
The corresponding feature qualities $q(\cdot)$ are constants in the optimization problem.
Further, we use the Dice dissimilarity (cf.~Equation~\ref{eq:afs:dice-rearranged-equal-size}) to measure feature-set dissimilarity for alternatives.
The dissimilarity threshold~$\tau \in [0,1]$ is a user-defined constant.
Finally, we assume fixed, user-defined feature-set sizes~$k \in \mathbb{N}$.

\paragraph{Sequential-search problem}

In the sequential case (cf.~Definition~\ref{def:afs:alternative-feature-selection-sequential} and Equation~\ref{eq:afs:afs-sequential}), only one feature set~$F_s$ is variable in the optimization problem, while the existing feature sets $F_{\bar{s}} \in \mathbb{F}$ with their selection vectors $\bar{s}$ are constants.
\begin{equation}
	\begin{aligned}
		\max_s &\quad & Q_{\text{uni}}(s,X,y) &= \sum_{j=1}^{n} q(X_{\cdot{}j},y) \cdot s_j \\
		\text{subject to:} &\quad \forall F_{\bar{s}} \in \mathbb{F}: & \sum_{j=1}^n s_j \cdot \bar{s}_j &\leq (1 - \tau) \cdot k \\
		&\quad & \sum_{j=1}^n s_j &= k \\
		&\quad & s &\in \{0,1\}^n
	\end{aligned}
	\label{eq:afs:afs-sequential-complete}
\end{equation}
\paragraph{Simultaneous-search problem}

In the simultaneous case (cf.~Definition~\ref{def:afs:alternative-feature-selection-simultaneous} and Equation~\ref{eq:afs:afs-simultaneous}), all feature sets are variable.
$a \in \mathbb{N}_0$ denotes the number of alternatives, which corresponds to the number of feature sets minus one.
Next, we introduce auxiliary variables to linearize products between variables (cf.~Equation~\ref{eq:afs:product-linear}).
Finally, we use sum-aggregation (cf.~Equation~\ref{eq:afs:afs-simultaneous-sum-objective}) over alternatives in the objective here.
\begin{equation}
	\begin{aligned}
		\max_{s^{(0)}, \dots, s^{(a)}} &\quad & \sum_l Q_{\text{uni}}(s^{(l)},X,y) &= \sum_l \sum_j q(X_{\cdot{}j},y) \cdot s^{(l)}_j\\
		\text{subject to:} &\quad \forall l_1~\forall l_2: & \sum_j t^{(l_1,l_2)}_j &\leq (1 - \tau) \cdot k \\
		&\quad \forall l_1~\forall l_2~\forall j: & t^{(l_1,l_2)}_j &\leq s^{(l_1)}_j \\
		&\quad \forall l_1~\forall l_2~\forall j: & t^{(l_1,l_2)}_j &\leq s^{(l_2)}_j \\
		&\quad \forall l_1~\forall l_2~\forall j: & 1 + t^{(l_1,l_2)}_j &\geq s^{(l_1)}_j + s^{(l_2)}_j \\
		&\quad \forall l: & \sum_j s^{(l)}_j &= k \\
		&\quad \forall l: & s^{(l)} &\in \{0,1\}^n \\
		&\quad \forall l_1~\forall l_2: & t^{(l_1,l_2)} &\in \{0,1\}^n \\
		\text{with indices:} &\quad & l &\in \{0, \dots, a\} \\
		&\quad & l_1 &\in \{1, \dots, a\} \\
		&\quad & l_2 &\in \{0, \dots, l_1-1\} \\
		&\quad & j &\in \{1, \dots, n\}
	\end{aligned}
	\label{eq:afs:afs-simultaneous-complete}
\end{equation}

\subsection{Pre-Selection for the Univariate Objective}
\label{sec:afs:appendix:univariate-pre-selection}

In this section, we describe how to potentially speed up the optimization of the univariate objective (cf.~Equation~\ref{eq:afs:univariate-filter}) by \emph{pre-selection} if the user-defined feature-set sizes~$k$ and the number of alternatives~$a$ are small.

The univariate objective is monotonic in the features' qualities~$q(X_{\cdot{}j},y)$ and the selection decisions~$s_j$.
In particular, the objective is non-decreasing when replacing a feature with another one of higher quality.
Further, unless some feature qualities are negative, selecting more features does not decrease the objective.
Sum-aggregation (cf.~Equation~\ref{eq:afs:afs-simultaneous-sum-objective}) and min-aggregation (cf.~Equation~\ref{eq:afs:afs-simultaneous-min-objective}) for the simultaneous-search problem are monotonic as well.

Thus, assuming $(a + 1) \cdot k < n$, it suffices to use the $(a + 1) \cdot k$ highest feature qualities when searching for an optimal solution out of $a + 1$ feature sets.
Due to monotonicity, the remaining feature qualities cannot improve the objective, so one can drop them before optimization.
We call this step \emph{pre-selection}.
While there might also be optimal solutions using the dropped features, their objective value cannot be higher than with pre-selection.
For example, such solutions can arise in case of multiple identical qualities or for min-aggregation in the objective (cf.~Example~\ref{ex:afs:min-aggregation}).
Also, the optimal solution might not contain all pre-selected features, i.e., pre-selection over-approximates the set of selected features.

One can conduct pre-selection before using a solver or any other search mechanism, e.g., exhaustive search.
The latter generally has polynomial runtime regarding~$n$ assuming small, constant $a$ and $k$, i.e., $a \cdot k \in O(1)$ (cf.~Section~\ref{sec:afs:approach:complexity:exhaustive}).
With pre-selection, the pure search cost would even become independent from~$n$, i.e., $O(1)$ under that assumption.
However, one would need to determine the highest feature qualities first, e.g., by sorting all qualities in~$O(n \cdot \log n)$ or iteratively determining the maximum quality in~$O((a+1) \cdot k \cdot n)$.

\subsection{Time Complexity}
\label{sec:afs:appendix:complexity}

In this section, we provide details for our analysis of time complexity (cf.~Section~\ref{sec:afs:approach:complexity}).
In particular, we discuss a special case of exhaustive simultaneous search (cf.~Section~\ref{sec:afs:appendix:complexity:exhaustive-simultaneous-special-case}), outline related work (cf.~Section~\ref{sec:afs:appendix:complexity:related-work}), provide proofs (cf.~Section~\ref{sec:afs:appendix:complexity:proofs}), and describe future work (cf.~Section~\ref{sec:afs:appendix:complexity:future-work}).

\subsubsection{A Special Case of Exhaustive Simultaneous Search}
\label{sec:afs:appendix:complexity:exhaustive-simultaneous-special-case}

The complexity of exhaustive simultaneous search is lower than in Proposition~\ref{prop:afs:complexity-exhaustive-simultaneuos} for the special case~$0 < \tau \cdot k \leq 1$, i.e., if feature sets need to differ in only one feature.
There, each feature set is an alternative to each other unless both sets are identical.
Thus, each set of $a + 1$ distinct feature sets constitutes a valid solution, and further constraint checking is unnecessary.
Hence, instead of iterating over sets of feature sets, one can iterate over individual feature sets and maintain a buffer containing the $a + 1$ feature sets with the highest quality.
For each feature set iterated over, one needs to determine if its quality is higher than the lowest feature-set quality in the buffer and replace it if yes.
This procedure has a runtime of $O((a + 1) \cdot n^k)$ without the cost of evaluating the objective.
I.e., unlike in Proposition~\ref{prop:afs:complexity-exhaustive-simultaneuos}, the number of alternatives~$a$ is not part of the exponent anymore, and the cost corresponds to the search for one feature set times the cost of updating the buffer.
For large $a$, one can implement the buffer as a heap, thereby reducing the linear factor regarding~$a$ to a logarithmic one.

\subsubsection{Related Work}
\label{sec:afs:appendix:complexity:related-work}

In this section, we discuss related work on $\mathcal{NP}$-hard problems that resemble alternative feature-selection with univariate feature qualities (cf.~Equation~\ref{eq:afs:univariate-filter}), providing background for Section~\ref{sec:afs:approach:complexity:univariate}.

\paragraph{Integer programming}

The univariate objective and several other feature-selection methods allow us to phrase alternative feature selection as a 0-1 integer linear program (cf.~Section~\ref{sec:afs:approach:objectives:white-box}).
\textsc{Integer Programming} is $\mathcal{NP}$-complete in general, even for binary decision variables~\cite{garey2003computers, karp1972reducibility}.
Thus, alternative feature selection with a white-box objective suitable for \textsc{Integer Programming} resides in $\mathcal{NP}$.
However, it could still be easier since alternative feature selection only uses particular constraint types instead of expressing arbitrary integer linear problems.
Vice versa, the membership in $\mathcal{NP}$ based on \textsc{Integer Programming} assumes a particular encoding of alternative feature selection, i.e., each constraint is stored separately and counts towards the problem's input size.
If we instead define the input size only as the number of features~$n$ or the total encoding length of the objective function plus parameters~$a$, $k$, and $\tau$, the problem could be harder than $\mathcal{NP}$, e.g., for a high number of alternatives.
In particular, increasing the number of alternatives would increase the encoding length logarithmically but the cost of constraint checking quadratically.

\paragraph{Multi-way number partitioning / multiprocessor scheduling}

The literature provides different formulations of \textsc{Multi-Way Number Partitioning} and \textsc{Multiprocessor Scheduling}.
In particular, different objectives formalize the notion of balanced subset sums and can lead to different optimal solutions~\cite{korf2010objective, lawrinenko2017identical}.
The maximin formulation we use for min-aggregation in the simultaneous-search problem (cf.~Definition~\ref{def:afs:alternative-feature-selection-simultaneous} and Equation~\ref{eq:afs:afs-simultaneous-min-objective}) is one such notion.

There are several exact algorithms to solve \textsc{Multi-Way Number Partitioning}, e.g., using branch-and-bound approaches that might have exponential runtime~\cite{haouari2008maximizing, schreiber2018optimal, walter2017improved}.
For a fixed number of partitions, the problem is weakly $\mathcal{NP}$-complete since it admits pseudo-polynomial algorithms~\cite{garey2003computers, korf2009multi}.
Such algorithms run in polynomial time if the input numbers are bounded to a particular size known in advance.
Since our feature qualities typically are real numbers, one would need to scale and discretize them to apply such an algorithm.
Also, for an arbitrary number of partitions, the problem is strongly $\mathcal{NP}$-complete, so no pseudo-polynomial algorithm can exist unless $\mathcal{P}=\mathcal{NP}$~\cite{garey2003computers}.

However, $\mathcal{NP}$-completeness does not exclude the existence of approximation routines that run in polynomial time and have a guaranteed quality relative to the optimal solution.
For example, \cite{alon1998approximation, deuermeyer1982scheduling, woeginger1997polynomial}~present such algorithms for the maximin formulation of \textsc{Multi-Way Number Partitioning}, which corresponds to our objective with min-aggregation.
In particular, \cite{alon1998approximation, woeginger1997polynomial} describe polynomial-time approximation schemes (PTAS), which can provide a solution arbitrarily close to the optimum.
However, the runtime depends on the desired approximation ratio and can grow exponentially the more precision is desired.
Unless $\mathcal{P}=\mathcal{NP}$, the strong $\mathcal{NP}$-completeness of the problem prevents the existence of a fully polynomial-time approximation scheme (FPTAS), which would only polynomially depend on the precision of approximation~\cite{alon1998approximation, woeginger1997polynomial}.
However, an FPTAS does exist for each fixed number of partitions~\cite{sahni1976algorithms}.
Further, besides approximations, the problem also has polynomial-time exact algorithms if certain parameters of the problem are fixed, e.g., the number of unique numbers to be partitioned or the largest number~\cite{mnich2018parameterized}.
Thus, the problem is fixed-parameter tractable ($\mathcal{FPT}$) for an appropriate definition of `parameter'.

\paragraph{Balanced number partitioning / k-partitioning}

While the previous approaches considered sets of arbitrary sizes, there are number-partitioning problems with constrained~$k$ as well, e.g., called \textsc{Balanced Number Partitioning} or \textsc{K-Partitioning}.
The problem formulations differ in their objective and cardinality constraints, e.g., if equalities or inequalities are used.

For the minimax objective, \cite{babel1998thek, michiels2012computer, zhang2011heuristic} propose heuristic algorithms, some with approximation guarantees.
\cite{babel1998thek} also provides a bound of the objective value relative to the unconstrained case.
Further, there is a PTAS for each fixed set size~$k$~\cite{michiels2012computer}.
Finally, the problem exhibits a polynomial-time exact algorithm for $k=2$~\cite{dellamico2004heuristic, dellamico2001bounds} and an FPTAS for $k=n/2$~\cite{woeginger2005comment}.

One can also loosen the cardinality constraints by requiring $\leq k$ instead of $= k$.
Further, the cardinality~$k$ might vary between partitions.
This generalized problem is strongly $\mathcal{NP}$-hard but has heuristics running in polynomial time~\cite{kellerer2011a32approximation}.
In particular, \cite{chen2016efficient} provides an efficient PTAS (EPTAS).

As another problem formulation, \cite{chen20023partitioning, he2003kappa, lawrinenko2018reduction} use a maximin objective as we do.
This objective was rarely addressed in combination with cardinality constraints in the literature~\cite{lawrinenko2018reduction}.
Also, all these three references use $\leq k$ constraints instead of $= k$.
Again, this problem is strongly $\mathcal{NP}$-hard~\cite{he2003kappa}, but \cite{chen20023partitioning, he2003kappa, lawrinenko2018reduction} propose approximation algorithms, partly with quality guarantees.

\paragraph{Other partitioning problems}

There are other $\mathcal{NP}$-complete problems that partition elements into non-overlapping subsets~\cite{garey2003computers}.
E.g., \textsc{Partition}~\cite{karp1972reducibility} asks if one can partition a set of elements with positive integer weights into two subsets with the same subset sum.
\textsc{3-Partition}~\cite{garey2003computers} demands a partitioning into three-element subsets with an identical, predefined subset sum of the elements' positive integer weights.
In contrast to these two problems, we do not require alternative feature sets to have the same quality.

\paragraph{Bin covering}

\textsc{Bin Covering}~\cite{assmann1984dual} distributes elements with individual weights into bins such that the number of bins is maximal and the summed weights in each bin surpass a predefined limit.
\cite{lawrinenko2017identical} noted a relationship between \textsc{Multi-Way Number Partitioning} and \textsc{Bin Covering}, which may improve solution approaches for either problem~\cite{walter2017lower, walter2017improved}.
In our case, we could maximize the number of alternatives such that each feature set's quality exceeds a threshold.

\paragraph{Multiple knapsack}

The simultaneous-search problem with sum-aggregation, $\tau=1$, and univariate feature qualities is a special case of the \textsc{Multiple Knapsack} problem~\cite{chekuri2005polynomial}.
The latter involves knapsacks, i.e., sets with individual capacities, and elements with individual weights and profits.
The goal is to assign elements to knapsacks such that the summed profit of selected elements is maximal.
Each element can be assigned to at most one knapsack, and the weights of all elements in the knapsack must not violate its capacity.
This problem is strongly $\mathcal{NP}$-complete in general, though it exhibits a PTAS~\cite{chekuri2005polynomial}.
However, our problem is a special case where the feature qualities act as profits, the feature-set sizes are capacities, and each feature has a weight of~1.
These uniform weights enable the polynomial-runtime result stated in Proposition~\ref{prop:afs:complexity-partitioning-sum}.

\subsubsection{Proofs}
\label{sec:afs:appendix:complexity:proofs}

In this section, we provide proofs for propositions from Section~\ref{sec:afs:approach:complexity:univariate}.

\paragraph{Proof of Proposition~\ref{prop:afs:complexity-incomplete-partitioning-min-constrained-k}}
\begin{proof}
Let~$I$ be an arbitrary problem instance of the simultaneous-search problem with min-aggregation, univariate feature qualities, a complete-partitioning scenario, and a fixed feature-set size~$k$ (cf.~Proposition~\ref{prop:afs:complexity-partitioning-min-constrained-k}).
We add a new feature~$f'$ to~$I$ and keep the parameters $a$, $k$, and $\tau$ as before, obtaining an instance~$I'$ of the incomplete-partitioning scenario since one feature will not be selected.
We set the quality~$q'$ of~$f'$ to be lower than all other feature qualities in~$I$.
Since the univariate objective monotonically increases in the selected feature qualities, selecting feature~$f'$ in a solution of~$I'$ does not have any benefit since~$f'$ would replace a feature with higher quality.
If~$f'$ is not selected, this solution of~$I'$ also solves~$I$.
However, if the qualities of the alternatives are not equal, $f'$ might still be chosen in a set that does not have the minimum quality of all sets since only the latter determines the objective value (cf.~Example~\ref{ex:afs:min-aggregation}).
In that case, we replace $f'$ with the remaining unselected feature; the objective value remains the same, and the solution becomes valid for~$I$.
Thus, in any case, we can easily transform a solution for~$I'$ to a solution for~$I$.

Overall, an algorithm for incomplete partitioning instances can also solve arbitrary complete-partitioning instances with negligible computational overhead.
Thus, a polynomial-time algorithm for incomplete partitioning could also solve complete partitioning polynomially.
However, the latter problem is $\mathcal{NP}$-complete (cf.~Proposition~\ref{prop:afs:complexity-partitioning-min-constrained-k}), so incomplete partitioning has to be $\mathcal{NP}$-hard.
Since checking a solution for incomplete partitioning needs only polynomial time, we obtain membership in $\mathcal{NP}$ and thereby $\mathcal{NP}$-completeness.
\end{proof}

\paragraph{Proof of Proposition~\ref{prop:afs:complexity-no-partitioning-min-constrained-k}}
\begin{proof}
Let~$I$ be an arbitrary problem instance of the simultaneous-search problem with min-aggregation, univariate feature qualities, a complete-partitioning scenario, the Dice dissimilarity (cf.~Equation~\ref{eq:afs:dice}) as~$d(\cdot)$, and a fixed feature-set size~$k$ (cf.~Proposition~\ref{prop:afs:complexity-partitioning-min-constrained-k}).
We create a new problem instance~$I'$ by adding a new feature~$f'$ and increasing the feature-set size to $k' = k + 1$.
Further, we set $\tau' = (k' - 1) / k'$, thereby allowing an overlap of at most one feature between feature sets.
Also, we choose~$f'$ to have a considerably higher quality~$q'$ than all other features.
Our goal is to force the selection of~$f'$ in all feature sets, no matter which other features are selected.
One possible choice is $q' = \sum_{j=1}^n q_j + \varepsilon$ for a small $\varepsilon \in \mathbb{R}_{> 0}$
This quality~$q'$ of~$f'$ is higher than of any feature set not containing it.
Thus, a solution for~$I'$ contains~$f'$ in each feature set, while the remaining features are part of exactly one feature set.
Hence, we can remove~$f'$ to get feature sets of size~$k = k' - 1$ that constitute an optimal solution for the original problem instance~$I$.

This transformation shows how an algorithm for problem instances with $\tau < 1$ can help solve arbitrary problem instances with $\tau = 1$.
Given the $\mathcal{NP}$-completeness of the latter problem (cf.~Proposition~\ref{prop:afs:complexity-partitioning-min-constrained-k}), we obtain $\mathcal{NP}$-hardness of the former.
\end{proof}
One can transfer this reduction from $\tau' = (k' - 1) / k'$ to all other $\tau > 0$.
In particular, for a given~$k$, there is only a finite number of $\tau$ values leading to different set overlaps, e.g., $\tau = \{0, 1/k, \dots, (k - 1) / k, 1\}$ for the Dice dissimilarity.
The proof for the highest overlap except~$\tau=0$ requires creating an instance $I'$ with $\tau'= 1/k$ from an instance with $\tau = 1$.
For this purpose, $k^2 - k$ features need to be added since $\tau' = k / k' = k / (k + k^2 -k) = 1/k$.
I.e., $k$ out of $k' = k^2$ features need to form a complete partitioning, while the remaining $k^2 - k$~features occur in each feature set and will be removed after solving~$I'$.
The number of features to be added is polynomial in~$k$ and thereby also polynomial in~$n$.

\paragraph{Proof of Proposition~\ref{prop:afs:complexity-partitioning-sum}}
\begin{proof}
We discuss the simultaneous-search problem (cf.~Definition~\ref{def:afs:alternative-feature-selection-simultaneous}) with sum-aggregation (cf.~Equation~\ref{eq:afs:afs-simultaneous-sum-objective}) first.
We leverage the monotonicity of the univariate objective with sum-aggregation.
In particular, this objective cannot decrease when selecting features of higher quality.
Thus, we order all features decreasingly by their quality, which yields the complexity of~$O(n \cdot \log n)$.
Next, we pick features in this order without replacement and assign them to sets until we have the user-defined number of alternatives with the user-defined feature-set sizes.
Apart from observing cardinality constraints, the actual assignment of the selected features to sets does not matter quality-wise since swapping features between sets does not change the summed objective.
Thus, one can fill the feature sets in an arbitrary order.
Each assignment runs in $O(1)$, e.g., using arrays to store feature-set membership, yielding $O(n)$ for all features.
Without cardinality constraints, only the number of alternatives needs to be satisfied.
Further, if all features need to be selected, i.e., for a complete partitioning, one need not sort the features.
Finally, if only a small fraction of features needs to be selected, one might slightly improve complexity to $O(k \cdot n)$ by iteratively picking the maximum instead of sorting all qualities.

For the sequential-search problem (cf.~Definition~\ref{def:afs:alternative-feature-selection-sequential}), we conduct the same quality-sorting procedure.
In contrast to the simultaneous-search problem, the actual assignment of features to sets matters since the sets have an explicit order.
In particular, each alternative should get the remaining highest-quality features until its user-defined size is reached.
The complexity is still dominated by sorting and therefore~$O(n \cdot \log n)$.
\end{proof}

\subsubsection{Future Work}
\label{sec:afs:appendix:complexity:future-work}

In this section, we outline future work on alternative feature selection from the complexity-theory perspective, supplementing the Sections~\ref{sec:afs:approach:complexity} and~\ref{sec:afs:conclusion:future-work}.

\paragraph{Scenarios of alternative feature selection}

Our prior complexity analyses focused on special cases of alternative feature selection.
E.g., while we obtained $\mathcal{NP}$-hardness for min-aggregation with feature-set overlap (cf.~Proposition~\ref{prop:afs:complexity-no-partitioning-min-constrained-k}), an analysis of sum-aggregation with overlap is open, even for the sequential-search problem.
Sum-aggregation admits polynomial runtime for $\tau=1$ (cf.~Proposition~\ref{prop:afs:complexity-partitioning-sum}), but this result might not extend to~$\tau < 1$.
In particular, $\tau < 1$ increases the number of solution candidates, which could negatively affect the runtime.

Further, our complexity analyses mostly assumed univariate feature qualities (cf.~Equation~\ref{eq:afs:univariate-filter}).
Other feature-selection methods can reside in different complexity classes.

\paragraph{Complexity classes}

For analyzing other scenarios of alternative feature selection, several questions spring to mind.
First, one could establish a complexity result like $\mathcal{NP}$-hardness or membership in~$\mathcal{P}$.
In the former case, there might be pseudo-polynomial approaches or (F)PTAS.
As a first step in that direction, we showed membership in complexity class~$\mathcal{APX}$ under certain conditions (cf.~Proposition~\ref{prop:afs:approximation-apx}), i.e., there are polynomial-time algorithms yielding constant-factor approximations.
One may attempt to tighten the quality bound we derived.
Further, there might be efficient exact or approximate algorithms for certain types of problem instances, e.g., satisfying additional assumptions regarding the feature-set quality or the parameters~$k$, $a$, and $\tau$.
Finally, while we placed alternative feature selection in the parameterized complexity class~$\mathcal{XP}$ (cf.~Proposition~\ref{prop:afs:complexity-simultaneuos-xp}), one might prove membership or hardness for more specific parameterized complexity classes.

\paragraph{Related problem formulations}

We only focused on the optimization problem of alternative feature selection until now.
Another interesting question is how many alternatives exist for a given $n$, $k$, and $\tau$, regardless of their quality.
Also, given the number of alternatives as well, it would be interesting to have an exact or approximate estimate for the number of valid solutions for alternative feature selection, i.e., sets of feature sets.
While both these estimates are straightforward for $\tau = 1$, allowing arbitrary~$\tau$ poses a larger challenge.
Finally, one could re-formulate alternative feature selection similar to \textsc{Bin Covering} (cf.~Section~\ref{sec:afs:appendix:complexity:related-work}) and analyze this problem in detail.

\subsection{Greedy Depth Search for the Univariate Objective}
\label{sec:afs:appendix:greedy-depth}

In this section, we propose another heuristic search method for univariate feature qualities (cf.~Equation~\ref{eq:afs:univariate-filter} and Section~\ref{sec:afs:appendix:univariate-complete-optimization-problem}), complementing the methods discussed in Section~\ref{sec:afs:approach:univariate-heuristics}.
In particular, the new method \emph{Greedy Depth Search} is a sequential search method that generalizes \emph{Greedy Replacement} and allows to obtain more than $\frac{n - k}{\lceil \tau \cdot k \rceil}$ alternatives.

\paragraph{Algorithm}

\begin{algorithm}[tp]
	\DontPrintSemicolon
	\KwIn{Univariate feature qualities~$q \in \mathbb{R}^n$, \newline
		Feature-set size~$k \in \mathbb{N}$, \newline
		Number of alternatives~$a \in \mathbb{N}_0$, \newline
		Dissimilarity threshold~$\tau \in [0, 1]$}
	\KwOut{List of feature-selection decision vectors~$s^{(\cdot)}$}
	\BlankLine
	$\mathit{indices} \leftarrow$ sort\_indices($q$, order=descending) \tcp*{Order by qualities} \label{al:afs:greedy-depth:line:sorting}
	$\mathit{feature\_positions} \leftarrow \{0\}^k$ \tcp*{Indices of indices of features} \label{al:afs:greedy-depth:line:position-init-start}
	\For(\tcp*[f]{Start with top $k$ features}){$p \leftarrow 1$ \KwTo $k$}{
		$\mathit{feature\_positions}[p] \leftarrow p$ \tcp*{Ordered by qualities as well} \label{al:afs:greedy-depth:line:position-init-end}
	}
	$l \leftarrow 0$\ \tcp*{Number of current alternative}
	$has\_next\_solution \leftarrow$ \textbf{true} \;
	\While{$l \leq a$ \textbf{and} $has\_next\_solution$}{ \label{al:afs:greedy-depth:line:main-loop-start} \label{al:afs:greedy-depth:line:stop}
		$s^{(l)} \leftarrow \{0\}^n$ \; \label{al:afs:greedy-depth:line:selection-start}
		\For(\tcp*[f]{Select $k$ features, indexed by quality}){$p \leftarrow 1$ \KwTo $k$}{
			$j \leftarrow \mathit{indices}[\mathit{feature\_positions}[p]]$ \;
			$s^{(l)}_j \leftarrow 1$\;
		} \label{al:afs:greedy-depth:line:selection-end}
		\If{is\_valid\_alternative($s^{(l)}$, $\{s^{(0)}, \dots, s^{(l-1)}\}$)}{ \label{al:afs:greedy-depth:line:check-valid}
			$l \leftarrow l + 1$ \tcp*{Else, $s^{(l)}$ overwritten in next iteration}
		}
		$p \leftarrow k$ \tcp*{Update feature positions, starting with last} \label{al:afs:greedy-depth:line:position-update-start}
		\While{$p \geq 1$}{
			$position \leftarrow \mathit{feature\_positions}[p]$ \;
			\If(\tcp*[f]{Position can be increased}){$position < n + p - k$}{ \label{al:afs:greedy-depth:line:check-position-update}
				\For(\tcp*[f]{Also update later positions}){$\Delta_p \leftarrow 0$ \KwTo $k - p$}{ \label{al:afs:greedy-depth:line:quality-decrease-start}
					$\mathit{feature\_positions}[p + \Delta_p] \leftarrow position + \Delta_p + 1$ \; \label{al:afs:greedy-depth:line:quality-decrease-end}
				}
				$p \leftarrow -1$ \tcp*{Position update finished}
			}
			\Else(\tcp*[f]{Position cannot be increased}){
				$p \leftarrow p - 1$ \tcp*[f]{Also update at least one prior position} \label{al:afs:greedy-depth:line:continue-prior-position}
			}
		}
		\If(\tcp*[f]{Updating positions further would violate $n$}){$p = 0$}{ \label{al:afs:greedy-depth:line:position-update-impossible}
			$has\_next\_solution \leftarrow$ \textbf{false} \; \label{al:afs:greedy-depth:line:main-loop-end} \label{al:afs:greedy-depth:line:position-update-end}
		}
	}
	\Return{$s^{(0)}, \dots, s^{(l)}$}
	\caption{\emph{Greedy Depth Search} for alternative feature selection.}
	\label{al:afs:greedy-depth}
\end{algorithm}

Algorithm~\ref{al:afs:greedy-depth} outlines \emph{Greedy Depth Search}.
As in the other two heuristics, we start by sorting the features decreasingly according to their qualities~$q_j$ (Line~\ref{al:afs:greedy-depth:line:sorting}).
However, instead of keeping the same $\lfloor (1 - \tau) \cdot k \rfloor$~features in each alternative and only replacing the remaining ones, we now allow all features to be replaced.
In particular, we may exhaustively iterate over all feature sets, depending on the number of alternatives~$a$.
Thus, we maintain not only one feature position as before but a length-$k$ array of the feature positions for the current feature set (Lines~\ref{al:afs:greedy-depth:line:position-init-start}--\ref{al:afs:greedy-depth:line:position-init-end}).
This array represents feature indices regarding the sorted qualities and is sorted increasingly, which prevents evaluating the same feature set, only with different feature order, multiple times.

In the main loop of the algorithm, we find alternatives sequentially
(Lines~\ref{al:afs:greedy-depth:line:main-loop-start}--\ref{al:afs:greedy-depth:line:main-loop-end}).
For each potential alternative, we select the features based on the position array (Lines~\ref{al:afs:greedy-depth:line:selection-start}--\ref{al:afs:greedy-depth:line:selection-end}).
We check the resulting feature set against the constraints for alternatives (Line~\ref{al:afs:greedy-depth:line:check-valid}) and only store it if it is valid.
This check was unnecessary in the other two heuristics, which only formed valid alternatives by design.

Next, we update the feature positions for the next potential alternative (Lines~\ref{al:afs:greedy-depth:line:position-update-start}--\ref{al:afs:greedy-depth:line:position-update-end}).
First, we try to replace the lowest-quality feature in the current feature set by advancing one position in the sorted qualities.
This step may not be possible, as the feature set may already contain the feature with the overall lowest quality, i.e., position~$n$ in the array of sorted qualities (Line~\ref{al:afs:greedy-depth:line:check-position-update}).
In this case, we try to replace the second-lowest-quality feature in the current feature set by advancing its position.
If this action is impossible as well, we iterate further over positions in the current feature set by increasing quality (Line~\ref{al:afs:greedy-depth:line:continue-prior-position}).
Once we find a feature position that we can increase, we also advance all subsequent, i.e., lower-quality, positions accordingly.
Hence, the feature positions remain sorted by decreasing quality (Lines~\ref{al:afs:greedy-depth:line:quality-decrease-start}--\ref{al:afs:greedy-depth:line:quality-decrease-end}).

We repeat the main loop until we reach the desired number of alternatives~$a$ or until we cannot update any feature position without exceeding the number of features~$n$, i.e., we cannot form another alternative (Lines~\ref{al:afs:greedy-depth:line:stop} and~\ref{al:afs:greedy-depth:line:position-update-impossible}).
\begin{example}[Algorithm of \emph{Greedy Depth Search}]
	Consider $n=6$~features with univariate feature qualities $q = (9,8,7,3,2,1)$, feature-set size~$k=4$, number of alternatives~$a=1$, and dissimilarity threshold~$\tau = 0.5$, which permits an overlap of two features between sets here.
	Note that the features are already ordered by quality here, i.e., $\mathit{indices} = (1,2,3,4,5,6)$ (Line~\ref{al:afs:greedy-depth:line:sorting}).
	Next, the algorithm initializes $\mathit{feature\_positions} = (1,2,3,4)$ (Line~\ref{al:afs:greedy-depth:line:position-init-start}--\ref{al:afs:greedy-depth:line:position-init-end}).
	$s^{(0)}$ contains these $k$~features, i.e., $s^{(0)} = (1,1,1,1,0,0)$.
	Given that there are no other alternatives yet, this feature set is valid (Line~\ref{al:afs:greedy-depth:line:check-valid})) and the algorithm moves on to $l=1$.
	
	For forming $s^{(1)}$, the position-update step (Lines~\ref{al:afs:greedy-depth:line:position-update-start}--\ref{al:afs:greedy-depth:line:position-update-end}) first tries to only replace the lowest-quality feature in the alternative, i.e., $\mathit{feature\_positions} = (1,2,3,5)$ and $\mathit{feature\_positions} = (1,2,3,6)$.
	However, neither of these feature sets constitutes a valid alternative regarding~$s^{(0)}$.
	Thus, the algorithm attempts to replace the feature with the second-lowest quality as well, evaluating $\mathit{feature\_positions} = (1,2,4,5)$ and $\mathit{feature\_positions} = (1,2,4,6)$.
	However, the overlap with $s^{(0)}$ is still too large.
	The next value is $\mathit{feature\_positions} = (1,2,5,6)$, which yields the valid alternative $s^{(1)} = (1,1,0,0,1,1)$.
	
	\emph{Greedy Replacement} would terminate now since the options for replacing the $\lceil \tau \cdot k \rceil = 2$ lowest-quality features are exhausted.
	In contrast, \emph{Greedy Depth Search} attempts to replace the third-lowest-quality feature, starting with $\mathit{feature\_positions} = (1,3,4,5)$.
	This feature set is not a valid alternative, and neither are the subsequent feature sets with $\mathit{feature\_positions} = (1,3,4,6)$, $\mathit{feature\_positions} = (1,3,5,6)$, etc.
	After more iterations, the algorithm also replaces the highest-quality feature, starting with $\mathit{feature\_positions} = (2,3,4,5)$.
	Eventually, the algorithm reaches $\mathit{feature\_positions} = (3,4,5,6)$, which yields the valid alternative $s^{(2)} = (0,0,1,1,1,1)$.
	After obtaining $s^{(2)}$, there is no valid update of the feature positions left (Line~\ref{al:afs:greedy-depth:line:position-update-impossible}).
	Thus, the algorithm terminates.
	\label{ex:afs:greedy-depth:algorithm}
\end{example}

\paragraph{Time complexity}

The runtime behavior differs from the other two heuristics.
In particular, \emph{Greedy Replacement} has the same runtime cost between subsequent alternatives since it directly creates valid alternatives by design.
In contrast, \emph{Greedy Depth Search} iterates over all possible feature sets, and the runtime between valid alternatives may vary.
For each values of $\mathit{feature\_positions}$, the algorithm creates a feature selection in $O(k \cdot n)$ (Lines~\ref{al:afs:greedy-depth:line:selection-start}--\ref{al:afs:greedy-depth:line:selection-end}), checks constraints in $O(a \cdot n)$ (Line~\ref{al:afs:greedy-depth:line:check-valid}), and updates the position array in~$O(k^2)$ (Lines~\ref{al:afs:greedy-depth:line:position-update-start}--\ref{al:afs:greedy-depth:line:position-update-end}).
However, there are $O(n^k)$ potential $\mathit{feature\_positions}$, and \emph{Greedy Depth Search} may exhaustively iterate over them.
This cost is comparable to exhaustive conventional feature selection (cf.~Proposition~\ref{prop:afs:complexity-exhaustive-conventional}) and exhaustive sequential search (cf.~Proposition~\ref{prop:afs:complexity-exhaustive-sequential}).
Unlike the latter, the search does not restart for each alternative, i.e., it only considers each feature set once instead of $a+1$~times.

On the positive side, \emph{Greedy Depth Search} can yield more alternatives than \emph{Greedy Replacement} with its $O(n^2)$ cost or \emph{Greedy Balancing} with its $O(a \cdot n^2)$ cost.
Nevertheless, in scenarios where the latter two are applicable, i.e., $k + a \cdot \lceil \tau \cdot k \rceil \leq n$, they have a lower cost than \emph{Greedy Depth Search}.
In particular, \emph{Greedy Depth Search} needs $O(n^{\lceil \tau \cdot k \rceil})$~iterations to cover the options for replacing the worst $\lceil \tau \cdot k \rceil$ features in size-$k$ feature sets, which is the search space of the other two heuristics.
In particular, the cost disadvantage relative to the other two heuristics grows with the dissimilarity threshold~$\tau$.
As a remedy, one may use \emph{Greedy Replacement} for as many alternatives as possible and then continue with \emph{Greedy Depth Search}, initializing the $\mathit{feature\_positions}$ (Line~\ref{al:afs:greedy-depth:line:position-init-start}--\ref{al:afs:greedy-depth:line:position-init-end}) based on the results of the former heuristic.

\paragraph{Quality}

\emph{Greedy Depth Search} initially yields the same solutions as \emph{Greedy Replacement}.
Thus, \emph{Greedy Depth Search} also yields a constant-factor approximation of the optimal solution in case $k + a \cdot \lceil \tau \cdot k \rceil \leq n$ (cf.~Proposition~\ref{prop:afs:approximation-greedy-replacement}).
The quality analysis becomes more involved for further alternatives since these do not contain all top $\lfloor (1 - \tau) \cdot k \rfloor$~features anymore, on which our proof of Proposition~\ref{prop:afs:approximation-greedy-replacement} builds.
Thus, we leave this analysis open for future work.
The quality of alternatives may not even be monotonically decreasing anymore, as the following example shows:
\begin{example}[Non-monotonic quality of \emph{Greedy Depth Search}]
	Consider $n=4$~features with univariate feature qualities $q = (9,8,7,1)$, feature-set size~$k=2$, number of alternatives~$a=3$, and dissimilarity threshold~$\tau = 0.5$, which permits an overlap of one feature between sets here.
	\emph{Greedy Depth Search} yields the the selection $s^{(0)} = (1,1,0,0)$, $s^{(1)} = (1,0,1,0)$, $s^{(2)} = (1,0,0,1)$, and $s^{(3)} = (0,1,1,0)$, with the corresponding feature-set qualities~17, 16, 10, and 15.
	\label{ex:afs:greedy-depth:non-monotonic}
\end{example}

\paragraph{Limitations}

Like \emph{Greedy Balancing} and \emph{Greedy Replacement}, \emph{Greedy Depth Search} assumes univariate feature qualities and may be worse than exact search.
As a sequential procedure, it does not balance the alternatives' qualities.
It may yield more alternatives than the former two heuristics but has a higher and more variable runtime.

\renewcommand*{\bibfont}{\small} 
\printbibliography

\end{document}